\documentclass[11pt,letterpaper]{article}
\usepackage[margin=1in]{geometry}
\usepackage{float}
\usepackage{multirow}
\usepackage{bm}
\usepackage{amsmath,amsthm,amsfonts,amssymb,mathrsfs}
\usepackage{enumerate,color,xcolor}
\usepackage{graphicx}
\usepackage{subcaption}
\usepackage{url}
\usepackage[ruled,vlined]{algorithm2e} % Use algorithm2e package with ruled and vlined options
\usepackage{enumitem}
\usepackage{natbib}
\usepackage[colorlinks,linkcolor=orange,citecolor=blue]{hyperref}
\usepackage{booktabs}
\usepackage{array}
\usepackage{tikz}
\usetikzlibrary{arrows.meta,calc,positioning,intersections}

\numberwithin{equation}{section}
\theoremstyle{plain}
\newtheorem{theorem}{Theorem}
\newtheorem{lemma}{Lemma}
\newtheorem{proposition}{Proposition}

\newtheorem{definition}{Definition}
\newtheorem{assumption}{Assumption}
\newtheorem{remark}{Remark}

\newenvironment{keywords}{\par\noindent\textbf{Keywords: }\ignorespaces}{\par}

\newcommand{\E}{\mathbb{E}}
\newcommand{\R}{\mathbb{R}}
\newcommand{\Prob}{\mathbb{P}}
\newcommand{\calL}{\mathcal{L}}

\newcommand{\calP}{\mathcal{P}}

\newcommand{\inner}[2]{\left\langle #1, #2 \right\rangle}
\newcommand{\CK}{C(K)}
\newcommand{\norm}[1]{\|#1\|}
\newcommand{\muinv}{\mu} % Invariant measure
\newcommand{\DomLJ}{D(\calL_J)} % Domain of L_J
\newcommand{\CbK}{C_b(K)} % Space of bounded continuous functions on K
\newcommand{\DomLplus}{D^+(\calL_J)} % D^+

\begin{document}

\title{Accelerating Constrained Sampling: A Large Deviations Approach}

\author{
  Yingli Wang\thanks{School of Mathematics, Shanghai University of Finance and Economics, Shanghai, People's Republic of China; FinTech Thrust, Hong Kong University of Science and Technology (Guangzhou), Guangzhou, Guangdong Province, People's Republic of China; naturesky1994@gmail.com}
  \and Changwei Tu\thanks{FinTech Thrust, Hong Kong University of Science and Technology (Guangzhou), Guangzhou, Guangdong Province, People's Republic of China; ctu570@connect.hkust-gz.edu.cn}
  \and Xiaoyu Wang\thanks{Corresponding author. FinTech Thrust, Hong Kong University of Science and Technology (Guangzhou), Guangzhou, Guangdong Province, People's Republic of China; xiaoyuwang@hkust-gz.edu.cn}
  \and Lingjiong Zhu\thanks{Department of Mathematics, Florida State University, Tallahassee, Florida, United States of America; zhu@math.fsu.edu}
}

\date{}

\maketitle

\begin{abstract}
The problem of sampling a target probability distribution on a constrained domain arises in many applications including machine learning. For constrained sampling, various Langevin algorithms such as projected Langevin Monte Carlo (PLMC), based on the discretization of reflected Langevin dynamics (RLD) and more generally skew-reflected non-reversible Langevin Monte Carlo (SRNLMC), based on the discretization of skew-reflected non-reversible Langevin dynamics (SRNLD), have been proposed and studied in the literature. This work focuses on the long-time behavior of SRNLD, where a skew-symmetric matrix is added to RLD. Although acceleration for SRNLD has been studied, it is not clear how one should design the skew-symmetric matrix in the dynamics to achieve good performance in practice. We establish a large deviation principle (LDP) for the empirical measure of SRNLD when the skew-symmetric matrix is chosen such that its product with the outward unit normal vector field on the boundary is zero. By explicitly characterizing the rate functions, we show that this choice of the skew-symmetric matrix accelerates the convergence to the target distribution compared to RLD and reduces the asymptotic variance. Numerical experiments for SRNLMC based on the proposed skew-symmetric matrix show superior performance, which validate the theoretical findings from the large deviations theory.
\end{abstract}


\section{Introduction}

We consider the problem of sampling a distribution $\mu$
on a constrained domain $K\subset\mathbb{R}^{d}$
with probability density function 
\begin{equation} 
\mu(x)\propto\exp(-f(x)),\ x\in K,
    \label{eq-target}
\end{equation}
for a function $f:\mathbb{R}^d \to \mathbb{R}$. The sampling problem 
for both constrained domain $K\subset\mathbb{R}^{d}$ and unconstrained domain $K=\mathbb{R}^{d}$ is fundamental in many applications, such as Bayesian statistical inference, Bayesian formulations of inverse problems, and Bayesian classification and regression tasks in machine learning \citep{gelman1995bayesian,stuart2010inverse,andrieu2003introduction,teh2016consistency,DistMCMC19,GHZ2022,GIWZ2024}.

In the literature, \cite{bubeck2015finite, bubeck2018sampling} proposed and studied the
\textit{projected Langevin Monte Carlo} (PLMC) algorithm for constrained sampling, which
projects each unconstrained Langevin step back to the constraint set $K$:
\begin{equation}\label{projected:Langevin}
x_{k+1}=\mathcal{P}_{K}\!\left(x_{k}-\eta\nabla f(x_{k})+\sqrt{2\eta}\,\xi_{k+1}\right),
\end{equation}
where $\eta>0$ is the stepsize and $\{\xi_k\}_{k\ge1}$ are i.i.d.\ Gaussian random vectors with
$\xi_k\sim\mathcal N(0,I_{d})$, and the mapping $\mathcal P_K$ denotes the (Euclidean) metric projection onto $K$,
defined for any $x\in\mathbb R^d$ by
\begin{equation}\label{Euclidean:projection}
\mathcal P_K(x)\;:=\;\arg\min_{y\in K}\|x-y\|_2.
\end{equation}
In particular, when $K$ is nonempty, closed, and convex, $\mathcal P_K(x)$ is well-defined and single-valued.

The dynamics \eqref{projected:Langevin} is motivated by the \textit{reflected Langevin dynamics} (RLD),
i.e.\ the continuous-time overdamped Langevin SDE with reflection at the boundary:
\begin{equation}\label{reflected:SDE}
dX_{t}=-\nabla f(X_{t})\,dt+\sqrt{2}\,dW_{t}-\mathbf n(X_{t})\,Z(dt),
\end{equation}
where $W_{t}$ is a standard $d$-dimensional Brownian motion, $\mathbf n(x)$ is the outward unit normal vector for $x\in\partial K$ (and can be chosen arbitrarily off $\partial K$),
and $\int_{0}^{t}\mathbf n(X_s)\,Z(ds)$ is a bounded-variation reflection term that enforces the constraint
$X_t\in K$ for all $t\ge0$ whenever $X_0\in K$. The measure $Z(dt)$ is nonnegative with $Z([0,t])<\infty$ for each $t$,
and it charges only boundary times, i.e.
$\mathrm{supp}(Z)\subseteq \{t\ge0:\,X_t\in\partial K\}$.

The reflected Langevin dynamics \eqref{reflected:SDE} can be viewed as a Skorokhod problem on $K$:
one first follows the unconstrained Langevin increment, and whenever the trajectory hits the boundary,
a bounded-variation correction is added so that $X_t \in K$ for all $t\ge 0$.
Geometrically, this correction acts only at boundary times (i.e., the measure $Z(dt)$ is supported on
$\{t: X_t\in \partial K\}$) and pushes the trajectory inward along $-\mathbf n$.
In smooth domains, the reflection term is the \emph{minimal} correction that prevents leaving $K$,
in contrast to discrete-time schemes that reproject the iterate after each time step.
%\textit{a) What is the geometry of the reflected Langevin dynamics, e.g., (1.3) and (1.7)? How is this different from and related to projection? More description like a figure would be helpful.}

Since the seminal works  \citep{bubeck2015finite, bubeck2018sampling}
where the density $\mu$ is assumed to be log-concave, there have been many studies proposing various Langevin algorithms for constrained sampling. 
%%%%%%%%%%%%%%%%%%%%%%%%%%%%%%%%%%%%%%%%
\cite{Lamperski2021} considers the \textit{projected stochastic gradient Langevin dynamics} in the setting of non-convex smooth Lipschitz $f$ on a convex body where the gradient noise is assumed to have finite variance with a uniform sub-Gaussian structure; see also \cite{zheng2022constrained}.
%%%%%%%%%%%%%%%%%%%%%%%%%%%%%%%%%%%%%%%%
\textit{Proximal Langevin Monte Carlo}
is proposed in \cite{Brosse} for constrained sampling; see also \cite{SR2020} for a study on the proximal stochastic gradient Langevin algorithm
from a primal-dual perspective. 
%%%%%%%%%%%%%%%%%%%%%%%%%%%%%%%%%%%%%%%%%%%%%%%%
\textit{Mirror descent-based Langevin algorithms} (see e.g. \cite{hsieh2018mirrored,Chewi2020,Zhang2020,TaoMirror2021,Ahn2021})
can also be used for constrained sampling,  that was first proposed in \cite{hsieh2018mirrored}, inspired by the classical mirror descent in optimization. 
%%%%%%%%%%%%%%%%%%%%%%%%%%%%%
More recently, inspired by the penalty method in the optimization literature, 
\textit{penalized Langevin Monte Carlo algorithms} are proposed and studied in \cite{GHZ2022},
where the objective $f$ can be non-convex in general. In addition, constrained non-convex exploration combined with replica-exchange Langevin dynamics was proposed and studied in \cite{constraint_replica}. 
%%%%%%%%%%%%%%%%%%%%%%

In the literature, it is known
that in the setting of unconstrained sampling, 
by \textit{breaking reversibility}, 
a \textit{non-reversible} Langevin dynamics, 
where an \textit{skew-symmetric} matrix is added,
can converge to the target distribution faster on the Euclidean space $\mathbb{R}^{d}$. 
This phenomenon is well understood through non-asymptotic
convergence analysis; see e.g. \cite{HHS93,HHS05,reyLDP,HWGGZ20,GGZ2}.
%%%%%%%%%%%%%%%%%%%%%%%%%%%%%%%%%%%%%%%%%%%
Motivated by this, in the setting of constrained sampling, non-reversible Langevin algorithms with skew-reflected boundary was
proposed and studied in \cite{DFTWZ2025}, 
that can accelerate the convergence 
of the (reversible) projected Langevin Monte Carlo, 
based on the discretization of the \textit{skew-reflected non-reversible Langevin dynamics} (SRNLD):
\begin{equation}\label{eqn:anti:reflected}
dX_{t}=-(I+J(X_{t}))\nabla f(X_{t})dt+\sqrt{2}dW_{t}-\mathbf n^J(X_{t})L(dt),
\end{equation}
where $W_{t}$ is a standard $d$-dimensional Brownian motion; for every $x$, $J(x)$ is a skew-symmetric matrix, i.e. $J(x)=-(J(x))^{\top}$, and moreover,
$\Vert J\Vert_{\infty}:=\sup_{x\in K}\Vert J(x)\Vert<\infty$ and $\nabla \cdot J=0$.
The term $\mathbf n^J(X_{t})L(dt)$ ensures that $X_{t}\in K$ 
for every $t$ given that $X_{0}\in K$. 
In particular, $\int_{0}^{t}\mathbf n^J(X_{s})L(ds)$ is a bounded variation \emph{skew-reflection} process
and the measure $L(dt)$ is such that $L([0,t])$ is finite, 
$L(dt)$ is supported on $\{t:\,X_{t}\in\partial K\}$, where the skew-reflection is defined through the following \emph{skew unit normal vector}: 
\begin{equation}\label{skew unit norm vector}
    \mathbf n^J(X_{t}) := \frac{(I+J(X_{t}))\mathbf n(X_{t})}{\sqrt{\|\mathbf n(X_{t})\|_2^2+\|J(X_{t})\mathbf n(X_{t})\|_2^2}},
\end{equation}
where $\mathbf n(x)$ is the \emph{outward} unit normal vector for $x\in\partial K$
(and can be chosen arbitrarily off $\partial K$), and
$\mathcal N_K(x)$ denotes the (Euclidean) normal cone of $K$ at $x$, defined by
\[
\mathcal N_K(x):=\left\{\,v\in\mathbb R^d:\ \langle v, y-x\rangle \le 0,\ \text{for any $y\in K$}\,\right\},
\qquad x\in K.
\]
(For $x\in\mathrm{int}(K)$, $\mathcal N_K(x):=\{0\}$.)

We assume the \text{skew-matrix} $J$ satisfies the condition that $\langle \mathbf n^J, \mathbf n\rangle \ge \delta_0> 0$ for some positive constant $\delta_0$.
To implement SRNLD \eqref{eqn:anti:reflected}, \emph{skew-reflected non-reversible Langevin Monte Carlo} (SRNLMC) algorithm was proposed and studied in \cite{DFTWZ2025}
that is based on the discretization of \eqref{eqn:anti:reflected}:
\begin{equation}\label{eqn:algorithm}
x_{k+1}=\mathcal{P}^J_{K}\left(x_{k}-\eta(I+J(x_{k}))\nabla f(x_{k})+\sqrt{2\eta}\xi_{k+1}\right),
\end{equation}
where $\xi_{k}$ are i.i.d. Gaussian random vectors $\mathcal{N}(0,I_{d})$, and $\mathcal P^J_{K}$ is the \emph{skew-projection}:
\begin{equation}\label{defn: PCJ}
    \mathcal P^J_{K}(x):=\arg\max_{y\in\bar{K}} \left\langle y-x, \mathbf n^J(\mathcal P_{K}(x))\right\rangle,
\end{equation}

where $\mathcal{P}_{K}$ is the Euclidean projection defined in \eqref{Euclidean:projection}.
The mapping $\mathcal P^J_K$ in \eqref{defn: PCJ} is an \emph{oblique} (or skew) projection associated with the
skew-reflection direction $\mathbf{n}^J$. Unlike the Euclidean metric projection $\mathcal P_K$ in \eqref{Euclidean:projection}, which returns the
closest point in $K$ to a given $x$, the skew-projection corrects infeasible points by moving them
back to $K$ \emph{along the prescribed direction} $\mathbf n^J$ (computed at the Euclidean projection point
$\mathcal P_K(x)$). In two dimensions, one may visualize $\mathcal P^J_K(x)$ as the boundary point obtained by
transporting $x$ in the direction $-\mathbf n^J(\mathcal P_K(x))$ until it re-enters the feasible set.
When $J(x)\mathbf n(x)=0$ for all $x\in\partial K$, we have $\mathbf n^J(x)=\mathbf n(x)$ on $\partial K$.
Consequently, the skew-projection correction reduces to the standard normal one (and $\mathcal P^J_K$ coincides
with the usual projection/reflection correction).

%\textit{a) What is the geometry of the reflected Langevin dynamics, e.g., (1.3) and (1.7)? How is this different from and related to projection? More description like a figure would be helpful.}

\begin{figure}[htbp]
\centering
\begin{tikzpicture}[>=Latex, font=\small, scale=0.97]
    
\tikzset{
  ghost/.style={thick, dashed, gray!70, -{Latex[length=2.2mm]}},
  push/.style={very thick, -{Latex[length=2.4mm]}},
  pushJ/.style={very thick, dashed, -{Latex[length=2.4mm]}},
  vec/.style={very thick, -{Latex[length=2.4mm]}},
  vecJ/.style={very thick, dashed, -{Latex[length=2.4mm]}},
}

% layout: left-right panels
\def\panelshift{8.3} % horizontal distance between panels

% =======================================================
% Panel (a): Reflection vs skew-reflection (LEFT)
% =======================================================
\begin{scope}[shift={(0,0)}]
  % Title (two lines to save width)
  \node[font=\bfseries, align=center] at (0, 3.35)
    {(a) Reflection vs\\ skew-reflection};

  \coordinate (Oa) at (0,-1.5);
  \def\Ra{3.3}

  \coordinate (startA) at ($(Oa)+(150:\Ra)$);
  \coordinate (endA)   at ($(Oa)+(30:\Ra)$);

  \fill[gray!10] (startA) arc (150:30:\Ra) -- (endA |- Oa) -- (startA |- Oa) -- cycle;
  \draw[thick] (startA) arc (150:30:\Ra);
  \node at (0,-0.55) {\large $K$};

  % Hit point x_tau
  \coordinate (xtau) at ($(Oa)+(60:\Ra)$);
  \fill (xtau) circle (1.6pt);
  \node[fill=white, inner sep=1pt, anchor=west] at ($(xtau)+(-0.24,-0.24)$) {$x_\tau$};

  % Normal push n (RLD)
  \draw[vec] (xtau) -- ($(xtau)+(60:1.55)$) coordinate (nA);
  \node[fill=white, inner sep=1pt, anchor=south east] at (nA) {$\mathbf n$};

  % Skew push n^J (SRNLD)
  \draw[vecJ] (xtau) -- ($(xtau)+(25:1.70)$) coordinate (nJ);
  \node[fill=white, inner sep=1pt, anchor=west] at (nJ) {$\mathbf n^J$};

  % Tangential component / J-drift (optional but helpful)
  \draw[->, red, very thick] (xtau) -- ($(xtau)-(150:0.95)$) coordinate (jend);
  \node[red, font=\footnotesize, anchor=west] at ($(jend)+(0.05,0.05)$) {$J$-drift};

\end{scope}

% =======================================================
% Panel (b): Euclidean projection vs skew-projection (RIGHT)
% =======================================================
\begin{scope}[shift={(\panelshift,0)}]
  \node[font=\bfseries, align=center] at (0, 3.35)
    {(b) Euclidean projection\\ vs skew-projection};

  \coordinate (Ob) at (0,-1.5);
  \def\Rb{3.3}

  \coordinate (startB) at ($(Ob)+(150:\Rb)$);
  \coordinate (endB)   at ($(Ob)+(30:\Rb)$);

  \fill[gray!10] (startB) arc (150:30:\Rb) -- (endB |- Ob) -- (startB |- Ob) -- cycle;
  \draw[thick] (startB) arc (150:30:\Rb);
  \node at (0,-0.55) {\large $K$};

  % base point x_k
  \coordinate (xk) at (0.95,0.00);
  \fill (xk) circle (1.4pt);
  \node[fill=white, inner sep=1pt, anchor=north east]
    at ($(xk)+(-0.03,-0.03)$) {$x_k$};

  % Euclidean projection point on the boundary
  \coordinate (pk) at ($(Ob)+(55:\Rb)$);
  \fill (pk) circle (1.5pt);

  % normal and skew-normal
  \draw[vec] (pk) -- ($(pk)+(55:1.18)$) coordinate (nb);
  \node[fill=white, inner sep=1pt, anchor=south east] at (nb) {$\mathbf n$};

  \draw[vecJ] (pk) -- ($(pk)+(40:1.22)$) coordinate (nJb);
  \node[fill=white, inner sep=1pt, anchor=west] at ($(nJb)+(0.04,0.02)$) {$\mathbf n^J$};

  % proposed point outside K
  \coordinate (xtilde) at ($(pk)+(55:2.10)$);
  \fill (xtilde) circle (1.5pt);
  \node[fill=white, inner sep=1pt, anchor=west]
    at ($(xtilde)+(0.14,0.05)$) {$\tilde{x}_{k+1}$};

  % ghost move from x_k to \tilde x_{k+1}
  \draw[ghost] (xk) -- (xtilde);

  % skew-projection point: manually placed on the boundary so that the dashed correction
% follows the oblique correction direction and reaches the boundary point
  \coordinate (pj) at ($(Ob)+(65:\Rb)$);
  \fill (pj) circle (1.5pt);

  % labels for projection points
  \node[fill=white, inner sep=1pt, anchor=north west]
    at ($(pk)+(0.12,-0.12)$)
    {$\mathcal P_K(\tilde{x}_{k+1})$};

  \node[fill=white, inner sep=1pt, anchor=south east]
    at ($(pj)+(-0.10,-0.18)$)
    {$\mathcal P^J_K(\tilde{x}_{k+1})$};

  % Euclidean projection and skew-projection arrows
  \draw[push]  (xtilde) -- (pk);
  \draw[pushJ] (xtilde) -- (pj);

  % legend
  \node[draw, rounded corners, inner sep=3pt, anchor=north west, fill=white]
    at (-3.15, 2.85) {
    \scriptsize
    \begin{tabular}{@{}l@{}}
      \tikz{\draw[->, thick] (0,0)--(0.6,0);} Euclidean projection\\
      \tikz{\draw[->, dashed, thick] (0,0)--(0.6,0);} Skew-projection
    \end{tabular}
  };
\end{scope}

\end{tikzpicture}
\caption{Geometric interpretation of the dynamics.
\textbf{(a)} At a boundary point $x_\tau\in\partial K$, standard reflection pushes along the normal direction $\mathbf n$,
whereas skew-reflection pushes along $\mathbf n^J$; when $J(x)\mathbf n(x)=0$ on $\partial K$, $\mathbf n^J=\mathbf n$ on $\partial K$.
Otherwise, $\mathbf n^J$ generally introduces an additional tangential component (red) induced by the skew field.
\textbf{(b)} Starting from $x_k$, the algorithm proposes an unconstrained step to $\tilde{x}_{k+1}$ and then maps it back to $K$
via either the Euclidean projection $\mathcal P_K$ (solid) or an oblique correction in the direction
$-\mathbf n^J$ leading to the skew-projection $\mathcal P^J_K$ (dashed).}
\label{fig:geometry_reflection_skew_projection}
\end{figure}

Although non-asymptotic convergence analysis for SRNLD and SRNLMC
are obtained in \cite{DFTWZ2025} and it is shown that
by adding the skew-symmetric matrix $J$, acceleration 
can be achieved in the context of constrained sampling,
the analysis in \cite{DFTWZ2025} does not
shed any light on \textit{how to design and construct} the skew-symmetric matrix $J$ to achieve good performance in practice.
Indeed, even in the unconstrained setting, how to choose
the skew-symmetric matrix is a challenging problem 
with very limited understanding. In the unconstrained setting, when the objective $f$ is quadratic, $X_t$ becomes a 
Gaussian process. Using the rate of convergence of the covariance of $X_t$
as the criterion, \cite{HHS93} showed that $J\equiv 0$ is the worst choice. \cite{lelievre2013optimal} proved the existence of the optimal skew-symmetric matrix $J$
to accelerate the convergence to equilibrium via maximizing spectral gaps.
\cite{WHC2014} proposed two approaches to designing $J$ for optimal convergence of Gaussian diffusion
and they also compared their algorithms with the one in \cite{lelievre2013optimal}. See also 
\cite{guillin2016} for related results. For non-quadratic objectives, the literature is sparse but reveals consistent acceleration benefits. \cite{GGZ2} demonstrated reduced recurrence times in non-convex optimization via non-reversible Langevin dynamics, by investigating the metastability behavior. \cite{le2020sharp} derived sharp spectral asymptotics for non-reversible metastable diffusions, and confirmed faster transitions via Eyring-Kramers formulas. Furthermore, \cite{lee2022non} proved accelerated Eyring-Kramers formulas for non-reversible Gibbs-sampled diffusions using novel capacity estimation. Collectively, these works validate that breaking reversibility enhances convergence, yet the optimal $J$ for general non-quadratic objectives remains an open challenge.
To the best of our knowledge, 
we are the first to study how to choose
and design the skew-symmetric matrix $J$
in the context of constrained sampling.
We adopt a different approach than \cite{DFTWZ2025} to study SRNLD through the lens 
of \textit{large deviations theory} that can help us
better understand how to design and construct
the skew-symmetric matrix $J$ effectively. 

In this paper, we focus on the study of large deviations of the \textit{empirical measure} for SRNLD $(X_{t})_{t\geq 0}$ in \eqref{eqn:anti:reflected}:
\begin{equation}
L_{t}:=\frac{1}{t}\int_0^t \delta_{X_s} ds, 
\end{equation}
such that $L_{t}\in\calP(K)$, where $\calP(K)$ denotes the space
of probability measures on $K$. 
Since \eqref{eqn:anti:reflected}
admits $\mu$ as a unique invariant distribution, by ergodic theorem, $L_{t}\rightarrow\mu$ as $t\rightarrow\infty$ a.s.
Informally speaking, the large deviations theory
concerns the small probability of the rare events
$\mathbb{P}(L_{t}\simeq\nu)$ where $\nu\neq\mu$. 
Typically, this probability is exponentially small, 
i.e. $\mathbb{P}(L_{t}\simeq\nu)=e^{-tI(\nu)+o(t)}$
as $t\rightarrow\infty$, where $I(\nu)\geq 0$
is known as the \textit{rate function} with $I(\nu)=0$
if and only if $\nu=\mu$. 
More formally speaking, the empirical measures $(L_t)_{t \ge 0}$ satisfy a \textit{large deviation principle} (LDP) in $\calP(K)$ (equipped with the weak topology) with speed $t$ and rate function $I: \calP(K) \to [0, \infty]$ 
if $I$ is non-negative, lower semicontinuous and for any measurable set $A$, we have
\begin{equation}
- \inf_{\nu\in A^\circ} I(\nu) \leq \liminf_{t\to \infty} \frac{1}{t}\log \mathbb{P}(L_{t}\in A) \leq
\limsup_{t\to \infty} \frac{1}{t} \log \mathbb{P}(L_{t}\in A) \leq - \inf_{\nu\in \bar A} I(\nu) \,,
\end{equation}
where $A^\circ$ denotes the interior of $A$ and $\bar A$ its closure, and the rate function $I(\nu)$ is said to be good if it has compact level sets.
We refer to \cite{Varadhan1,DZ1998} for a more general definition of large deviation principle, as well as its theory and applications. 

Large deviations theory can help us better understand
the convergence to the target distribution. When we have variants
of Langevin dynamics that target the same
Gibbs distribution $\mu$, they might lead to different
rate functions, and if we can show the rate function
for SRNLD $X_{t}$ given in \eqref{eqn:anti:reflected}
is greater than that for the RLD \eqref{reflected:SDE}, 
then as $t\rightarrow\infty$, SRNLD is more concentrated
around the target distribution $\mu$, which indicates
acceleration. 

The Donsker-Varadhan theory \citep{DV1,DV2,DV3,DV4} establishes 
a large deviation principle with a rate function \( I(\nu) \) for the empirical measure $L_{t}$ when the underlying process $X_{t}$ is a Markov process under suitable conditions.
However, the Donsker-Varadhan theory
does not apply to Langevin SDEs directly
because overdamped Langevin SDEs and its variants on $\mathbb{R}^{d}$ live on 
a non-compact set that requires redesign
and rethinking \citep{LDP-GG,LDP-acceleration}. 
Since reflected Langevin dynamics is reversible and on a compact domain, its large deviation principle follows
from the literature. This leads to the natural
question whether one can establish large deviations for skew-reflected non-reversible
Langevin dynamics. It turns out that it is a much harder problem
for the following technical reasons:

\begin{enumerate}[label=(\roman*)]
    \item \textbf{General Boundary Conditions:} For an arbitrary skew-symmetric matrix field $J(x)$, the domain $D(\mathcal{L}_J)$ of the infinitesimal generator $\mathcal{L}_J u(x) = \Delta u(x) - \langle (I + J(x))\nabla f(x), \nabla u(x) \rangle$ consists of functions $u \in C^2(K)$ satisfying the boundary condition $\langle \nabla u(x), (I+J(x))\mathbf n(x) \rangle = 0$ on $\partial K$. This is a generalized oblique derivative condition, which can be significantly more complex to analyze than standard boundary conditions such as Neumann or Dirichlet conditions. In contrast, the classical framework of \cite{DV1} applies only to the case $J\equiv 0$, where the Markov process is reversible, whereas our setting with a non-zero $J(x)$ introduces a non-reversible reflected process that falls outside their framework.

    \item \textbf{Operator Decomposition and Adjoints:} The standard approach to decompose $\mathcal{L}_J$ into a symmetric part $\mathcal{L}_S = \frac{1}{2}(\mathcal{L}_J + \mathcal{L}_J^*)$ and a skew-symmetric part $\mathcal{L}_A = \frac{1}{2}(\mathcal{L}_J - \mathcal{L}_J^*)$ (where $\mathcal{L}_J^*$ is the adjoint in $L^2(K,d\mu)$ with respect to the invariant measure $\mu$) requires a well-defined adjoint $\mathcal{L}_J^*$. The calculation of $\mathcal{L}_J^*$ via integration by parts involves boundary terms. The nature of these boundary terms, and consequently the precise form and domain of $\mathcal{L}_J^*$, are dictated by the original oblique boundary condition.

    \item \textbf{Properties of Decomposed Operators:} For the LDP rate function decomposition methods (e.g., as in \cite{LDP-GG}) to be applicable, it is highly desirable that $\mathcal{L}_S$ is self-adjoint and $\mathcal{L}_A$ is anti-self-adjoint, ideally on the same domain $D(\mathcal{L}_J)$ or at least on a common core. If the boundary conditions for $D(\mathcal{L}_J)$ and $D(\mathcal{L}_J^*)$ (which arises from the oblique condition) do not align simply, establishing these self-adjointness and anti-self-adjointness properties for $\mathcal{L}_S$ and $\mathcal{L}_A$ becomes a non-trivial task. The resulting $\mathcal{L}_S$ and $\mathcal{L}_A$ might not have domains that are easy to work with for subsequent spectral analysis or variational arguments.
\end{enumerate}

In this paper, we will overcome these technical challenges. The key insight
is the condition $J(x)\mathbf n(x)=0$ on $\partial K$ we impose (Assumption~\ref{assump:boundary} in this work) 
which is crucial because it simplifies the oblique boundary condition $\langle \nabla u, (I+J)\mathbf n \rangle = 0$ to the standard Neumann condition $\langle \nabla u, \mathbf n \rangle = 0$. For operators with Neumann boundary conditions, the theory of adjoints, self-adjoint extensions, and spectral properties is much better developed and more tractable. This simplification ensures that $\mathcal{L}_J^*$ can be cleanly characterized, leading to well-behaved $\mathcal{L}_S$ and $\mathcal{L}_A$ that are indeed self-adjoint and anti-self-adjoint respectively on $D(\mathcal{L}_J)$. This, in turn, allows for a more direct application and adaptation of the LDP framework and rate function decomposition techniques found in the literature such as \cite{LDP-GG}. Without this simplification, one might need to delve into more advanced partial differential equation (PDE) theory for oblique boundary value problems or find alternative approaches for the LDP analysis. As a by-product, this simplification also provides us a natural choice
for the skew-symmetric matrix, which will be discussed for the examples including the ball constraint, 
and more generally the constrained domains characterized by sublevel sets that include special cases such as the smoothed $\ell_p$ ball constraint in Section~\ref{sec:main} which turns out to work well empirically (Section~\ref{sec:numerical_experiments}).

With this insight, in Section~\ref{sec:main}, we will establish a large deviation principle for the empirical measures of SRNLD for a class of skew-symmetric matrices $J$ by adopting the G\"{a}rtner-Ellis framework (Theorem~\ref{thm:ldp_empirical_measures}).
Moreover, we show that the acceleration can be achieved by comparing the rate function in the presence of the skew-symmetric matrix $J$ and that of RLD without the skew-symmetric matrix (Proposition~\ref{prop:rate_decomposition}). In addition, we rigorously prove that these non-reversible dynamics lead to a reduction in the asymptotic variance for time-average estimators (Theorem~\ref{thm:variance_reduction}).
The class of skew-symmetric matrices $J$ can be constructed depending on the compact domain $K$ (Assumptions~\ref{assump:boundary} and \ref{assump:nablaJ}). As a by-product, it provides us a guidance how to design and construct the skew-symmetric matrix $J$ to achieve better performance in the context of constrained sampling, which bridges a gap between theory and practice (Section~\ref{sec:numerical_experiments}). 
In particular, we conduct numerical experiments including a toy example on truncated multivariate normal distribution, constrained Bayesian linear regression and constrained Bayesian logistic regression 
for both the ball constraint and the 
smoothed $\ell_{p}$ constraint
using both synthetic data and real data.
Our numerical experiments demonstrate the efficiency of the algorithmic design, and show superior performance.
Finally, we conclude in Section~\ref{sec:conclude}.
We provide a summary of notations in Appendix~\ref{sec:notations}.
The proofs of technical lemmas will be provided in Appendix~\ref{app:proofs_of_technical_lemmas} and additional technical derivations will be presented in Appendix~\ref{app:operator_decomposition_simplified}.

%%%%%%%%%%%%%%%%%%%%%%%%%%%%%%%%%%%%%%%%%%%%%%
\section{Main Results}\label{sec:main}

We consider a \textit{skew-reflected non-reversible Langevin dynamics} (SRNLD) on a compact, connected domain $K \subset \R^d$ with 
boundary $\partial K\in C^{2,\alpha}$ for some $\alpha \in (0,1)$:
\begin{equation} \label{eq:sde_local}
  dX_t = -(I+J(X_t))\nabla f(X_t)dt+\sqrt 2 dW_t, \quad X_t \in K,
\end{equation}
subject to reflecting boundary conditions. The infinitesimal generator is
\begin{equation} \label{eq:generator_local}
  \calL_J u(x) = \Delta u(x) - \inner{(I + J(x))\nabla f(x)}{\nabla u(x)},
\end{equation}
with the domain

\[
D(\calL_J)
=\left\{\,u\in C^{2,\alpha}(K):\ \inner{\nabla u(x)}{(I+J(x))\mathbf n(x)}=0\ \text{on }\partial K\,\right\}.
\]

The empirical measure is $L_t = \frac{1}{t}\int_0^t \delta_{X_s} ds \in \calP(K)$, the space of probability measures on $K$. We seek an LDP for $(L_t)_{t\ge 0}$ in the weak topology on $\calP(K)$. We introduce
the following assumptions that will be used throughout the rest of the paper.

\begin{assumption}[Regularity]\label{assump:regularity_local}
  The domain $K\subset\R^d$ is compact and connected, and its boundary satisfies $\partial K\in C^{2,\alpha}$ for some $\alpha \in (0,1)$.
  The potential satisfies $f\in C^{2,\alpha}(K)$.
  The matrix field satisfies $J\in C^{1,\alpha}(K)\cap L^\infty(K)$ and $J^\top=-J$ on $K$.
\end{assumption}

\begin{assumption}[Simplified Boundary Interaction]\label{assump:boundary}
On the boundary $\partial K$, $J(x)\mathbf n(x) = 0$, where $\mathbf n(x)$ is the outward unit normal vector.
\end{assumption}

\paragraph{Degeneracy under Assumption~\ref{assump:boundary}}
Under Assumption~\ref{assump:boundary}, we have $J(x)\mathbf n(x)=0$ for all $x\in\partial K$.
Consequently, on $\partial K$ the modified reflection direction $\mathbf n^J(x)$ coincides with the outward normal $\mathbf n(x)$,
so the skew-reflection reduces to the standard (normal) reflection. We keep the notation $\mathbf n^J$ for consistency with the general SRNLD/SRNLMC formulation.
In particular, $\langle \mathbf n^J(x),\mathbf n(x)\rangle = 1$ on $\partial K$.

\begin{assumption}\label{assump:nablaJ}
In the interior $K^\circ$, the matrix field is divergence-free in the sense that $\nabla\!\cdot J=0$ on $K^\circ$.
\end{assumption}

The assumptions introduced above serve distinct roles in establishing the theoretical framework:
\begin{itemize}
    \item \textbf{Invariance of $\mu$ (Lemma \ref{lem:invariant_measure_existence}):} The result that $\mu(dx) \propto e^{-f(x)}dx$ is the invariant measure relies primarily on the skew-symmetry of $J$ (Assumption \ref{assump:regularity_local}) and the divergence-free condition $\nabla \cdot J = 0$ (Assumption \ref{assump:nablaJ}). These conditions ensure that the additional drift term vanishes in the interior. 
    
    \item \textbf{Existence of LDP (Theorem \ref{thm:ldp_empirical_measures}):} The establishment of the LDP via the G\"artner-Ellis theorem relies heavily on the compactness of $K$ and the regularity of the coefficients (Assumption \ref{assump:regularity_local}). Furthermore, Assumption \ref{assump:boundary} simplifies the domain of the generator to standard Neumann boundary conditions, which facilitates the use of standard spectral theory for elliptic operators on compact domains to prove the existence and differentiability of the scaled cumulant generating function (SCGF) $\lambda(g)$.
    
    \item \textbf{Rate Function Decomposition and Variance Reduction (Proposition \ref{prop:rate_decomposition} \& Theorem \ref{thm:variance_reduction}):} The algebraic decomposition of the rate function and the variance reduction result depend on the specific structure of the adjoint operator $\mathcal{L}_J^*$. This calculation requires the full bundle of assumptions: Assumption \ref{assump:nablaJ} ($\nabla \cdot J = 0$) eliminates interior terms in the integration by parts, while Assumption \ref{assump:boundary} ($J\mathbf n=0$) eliminates the boundary terms, ensuring that $\mathcal{L}_S$ is self-adjoint and $\mathcal{L}_A$ is anti-self-adjoint on the domain $D(\mathcal{L}_J)$.
\end{itemize}

Next, we discuss how to construct $J(x)$ to satisfy
these assumptions and their implications.

%%%%%%%%%%%%%%%%%%%%%%%%%%%%%%%%%%%%%%%%%%%%%
\subsection{Construction of \texorpdfstring{$J(x)$}{J(x)} and Implications of Assumptions}\label{sec:construction_J_implications_new}

The skew-symmetric matrix field $J(x)$ is chosen to satisfy three key assumptions:
\begin{enumerate}[label=(\roman*)]
    \item \textbf{Regularity (Assumption~\ref{assump:regularity_local}):} $J\in C^{1,\alpha}(K)\cap L^\infty(K)$ and $J(x)=-\left(J(x)\right)^\top$.
    \item \textbf{Simplified Boundary Interaction (Assumption~\ref{assump:boundary}):} $J(x)\mathbf n(x) = 0$ on $\partial K$.
    \item \textbf{Divergence-Free Condition (Assumption~\ref{assump:nablaJ}):} $\nabla \cdot J(x) = 0$ for $x \in K^\circ$.
\end{enumerate}

Constructing $J(x)$ to satisfy all three conditions simultaneously is non-trivial.

\subsubsection{In 3D Space} 
A common way to ensure $J(x)$ is skew-symmetric (fulfilling part of Assumption~\ref{assump:regularity_local}) is to define its action via an axial vector field $\mathbf{k}(x)$: $J(x)\mathbf{w} = \mathbf{k}(x) \times \mathbf{w}$.

To satisfy Assumption~\ref{assump:nablaJ} ($\nabla \cdot J = 0$), we require the axial vector field $\mathbf{k}(x)$ to be irrotational (curl-free) in $K^\circ$, i.e., $\nabla \times \mathbf{k}(x) = \mathbf{0}$. This implies $\mathbf{k}(x)$ can be expressed as the gradient of a scalar potential $\phi(x)$, i.e., $\mathbf{k}(x) = \nabla \phi(x)$.
To satisfy Assumption~\ref{assump:boundary} ($J(x)\mathbf n(x) = 0$) on $\partial K$, we require $\mathbf{k}(x) \times \mathbf n(x) = \mathbf{0}$ on $\partial K$, which means $\mathbf{k}(x)$ must be parallel to $\mathbf n(x)$ on the boundary.
Therefore, in three dimensions, it suffices to find a scalar potential $\phi$ such that
\[
    \mathbf{k}(x):=\nabla \phi(x)
\]
has the required H\"{o}lder regularity (so that the resulting matrix field $J(x)$ satisfies
Assumption~\ref{assump:regularity_local}). Moreover, on the boundary $\partial K$ we require
$\nabla \phi(x)$ to be parallel to the outward normal $\mathbf n(x)$, i.e.
\[
    \nabla\phi(x)\times \mathbf n(x)=0,\qquad x\in\partial K .
\]
For instance, choosing $\mathbf{k}(x)=s x$ (with $s\in\mathbb R$) yields
$\nabla\times \mathbf{k}(x)=\mathbf 0$. If $K$ is a ball centered at the origin, then
$\mathbf n(x)$ is parallel to $x$ on $\partial K$, hence $\mathbf{k}(x)\parallel \mathbf n(x)$ holds.
Consequently, the corresponding $J(x)$ (analogous to $J_s(x)$ in
Section~\ref{sec:numerical_experiments}) satisfies the required assumptions.

\subsubsection{In a Higher-Dimensional Ball\texorpdfstring{ ($d > 3$)} {(d>3)}}

For our high-dimensional numerical experiments in Section~\ref{sec:numerical_experiments}, we additionally employ a simple block-diagonal construction on a ball, which yields a more structured nonreversible drift and is easy to implement.

Let $K$ be a ball in $\mathbb{R}^d$ centered at the origin, which corresponds to the $\ell_2$-norm constraints, that are frequently encountered in machine learning applications, such as Bayesian ridge regression \citep{Brosse,GHZ2022}.
The outward unit normal vector on the boundary $\partial K$ is $\mathbf{n}(x) = x/\|x\|_2$. The boundary condition in Assumption~\ref{assump:boundary}, $J(x)\mathbf{n}(x) = 0$, thus simplifies to $J(x)x = 0$ for $x \in \partial K$. We can construct a $J(x)$ that meets all requirements using a block-diagonal structure.
    
Specifically, when the dimension $d$ is a multiple of 3, say $d=3m$, we construct $J(x)$ in the following matrix form:
\[
    J(x) = 
    \begin{pmatrix}
        J^{(1)}(x_{1:3}) & \mathbf{0} & \cdots & \mathbf{0} \\
        \mathbf{0} & J^{(2)}(x_{4:6}) & \cdots & \mathbf{0} \\
        \vdots & \vdots & \ddots & \vdots \\
        \mathbf{0} & \mathbf{0} & \cdots & J^{(m)}(x_{3m-2:3m})
    \end{pmatrix},
\]
where $x_{3i-2:3i} = (x_{3i-2}, x_{3i-1}, x_{3i})^\top$ and $\mathbf{0}$ represents a $3 \times 3$ zero matrix. Each 3D block $J^{(i)}$ is constructed analogously to the 3D case. Specifically, the local axial vector $\mathbf{k}^{(i)}$ plays the role of the vector field $\mathbf{k}$ restricted to the coordinates of the $i$-th block. For example, we may choose $\mathbf{k}^{(i)}(x) = s_i \cdot x_{3i-2:3i}$,
where $s_i$ is a scalar constant. The action of such a block is $J^{(i)}\mathbf{w} = (s_i \cdot x_{3i-2:3i}) \times \mathbf{w}$.

This construction rigorously satisfies all three assumptions for a spherical domain:
\begin{enumerate}
    \item \textbf{Skew-symmetry:} Since each block $J^{(i)}$ is skew-symmetric, the full block-diagonal matrix $J(x)$ is also skew-symmetric.
    \item \textbf{Boundary condition ($J(x)x = 0$):} For any vector $x = \left(x_{1:3}^\top, \dots, x_{3m-2:3m}^\top\right)^\top$, the action $J(x)x$ results in a vector where each 3-dimensional component is $J^{(i)}x_{3i-2:3i} = (s_i x_{3i-2:3i}) \times x_{3i-2:3i} = \mathbf{0}$. Thus, $J(x)x = \mathbf{0}$ for all $x$, which satisfies the boundary condition.
    \item \textbf{Divergence-free ($\nabla \cdot J = 0$):} The divergence of each 3D block with respect to its own coordinates is zero. Due to the block-diagonal structure with no cross-dependencies between blocks, the divergence of the full matrix $J(x)$ is also zero.
\end{enumerate}
    
This block-diagonal construction is what we will employ in our high-dimensional numerical experiments (see Section~\ref{subsec:Bayesian:log:example} with $d=9$). Finally, if the dimension $d$ is not a multiple of 3, one can combine 3D blocks with zero blocks for the remaining dimensions, which still satisfies all the assumptions.

\subsubsection{Constrained Domain Characterized by Sublevel Sets in \texorpdfstring{$\mathbb{R}^d$}{Rd}}

We now present a constructive procedure to design the skew-symmetric matrix field $J(x)$ for a general class of constrained domains defined by sublevel sets. Constraints characterized by sublevel sets (or equivalently, functional inequalities) are ubiquitous in the sampling and optimization literature, encompassing standard convex bodies, polyhedra, and $\ell_p$-balls \citep{schmidt2005least,kook2022sampling,GHZ2022}. This generalizes the explicit construction provided for the $\ell_2$-ball example.

Let $g \in C^{2,\alpha}(\mathbb{R}^d)$ be a confining potential. Fix a level $\lambda > \inf_{x\in\mathbb{R}^d} g(x)$ and let
\[
    K := \{ x \in \mathbb{R}^d : g(x) \le \lambda \}.
\]
We assume that $K$ is bounded and that $\nabla g(x) \neq 0$ for all $x \in \partial K = \{g=\lambda\}$. Under these conditions, $\partial K$ is a $C^{2,\alpha}$ hypersurface, and the outward unit normal vector is given by $\mathbf{n}(x) = \nabla g(x) / \|\nabla g(x)\|$.

The key to satisfying the boundary condition $J(x)\mathbf{n}(x)=0$ is to construct an auxiliary potential $\psi$ whose gradient aligns with the normal vector on the boundary.
Choose an arbitrary function $h \in C^{2,\alpha}(\mathbb{R}^d)$ such that $h(x) \neq 0$ on $\partial K$ (e.g., $h(x) \equiv 1$ or $h(x) = 1 + \|x\|^2$).
Define the potential $\psi: K \to \mathbb{R}$ by:
\begin{equation}\label{eq:general:psi}
    \psi(x) := (\lambda - g(x)) \, h(x).
\end{equation}
On the boundary $\partial K$, where $g(x) = \lambda$, the gradient of $\psi$ satisfies:
\[
    \nabla \psi(x) = -\nabla g(x) \, h(x) + (\lambda - g(x)) \nabla h(x) = 
    -h(x) \nabla g(x).
\]
Crucially, this implies that $\nabla \psi(x)$ is parallel to the normal vector $\mathbf{n}(x)$ everywhere on $\partial K$.

We define $J(x)$ by decomposing the $d$-dimensional space into independent $3$-dimensional subspaces (with potentially one or two dimensions left over if $d$ is not a multiple of 3).
Let $m = \lfloor d/3 \rfloor$. For each $\ell = 1, \dots, m$, consider the coordinate triplet $(a_\ell, b_\ell, c_\ell) = (3\ell-2, 3\ell-1, 3\ell)$.
Define the block-specific axial vector $\mathbf{k}^{(\ell)}(x) \in \mathbb{R}^3$ by the partial derivatives of $\psi$:
\[
    \mathbf{k}^{(\ell)}(x) :=
    \begin{pmatrix}
    \partial_{a_\ell} \psi(x) \\
    \partial_{b_\ell} \psi(x) \\
    \partial_{c_\ell} \psi(x)
    \end{pmatrix}.
\]
Construct the $d \times d$ matrix $J^{(\ell)}(x)$ such that its only non-zero entries form a $3 \times 3$ skew-symmetric block on the indices $(a_\ell, b_\ell, c_\ell)$. By denoting the components of the local axial vector as $\mathbf{k}^{(\ell)}(x) = \left(k^{(\ell)}_1, k^{(\ell)}_2, k^{(\ell)}_3\right)^\top$, this block is given by:
\[
    \left(J^{(\ell)}\right)_{\{a_\ell,b_\ell,c_\ell\}\times\{a_\ell,b_\ell,c_\ell\}} =
    \begin{pmatrix}
    0 & -k^{(\ell)}_3 & k^{(\ell)}_2 \\
    k^{(\ell)}_3 & 0 & -k^{(\ell)}_1 \\
    -k^{(\ell)}_2 & k^{(\ell)}_1 & 0
    \end{pmatrix}.
\]
Finally, set the full matrix field as the sum:
\begin{equation}
    J(x) := \sum_{\ell=1}^m J^{(\ell)}(x).
\end{equation}

This construction satisfies all required assumptions:
\begin{enumerate}
    \item \textbf{Regularity:} Since $g, h \in C^{2,\alpha}$, $\psi \in C^{2,\alpha}$ and $J \in C^{1,\alpha}$. $J$ is skew-symmetric by construction.
    \item \textbf{Boundary Condition:} For any $x \in \partial K$, we showed $\nabla \psi(x) \parallel \mathbf{n}(x)$. This implies that for each block $\ell$, the local vector $\mathbf{k}^{(\ell)}$ is parallel to the projection of $\mathbf{n}$ onto those coordinates. The cross-product structure ensures that $J^{(\ell)}(x) \mathbf{n}(x) = 0$. Summing these gives $J(x)\mathbf{n}(x) = 0$.
        \item \textbf{Divergence-Free:} The divergence of a curl field is identically zero. Since each block is constructed as the curl of a gradient field ($\nabla \psi$), we have $\nabla \cdot J = 0$ in $K^\circ$.
\end{enumerate}

%%%%%%%%%%%%%%%%%%%%%%%%%%%%%%%%%%%%%
\subsubsection{Smoothed \texorpdfstring{$\ell_{p}$}{ellp} Ball Constraint}

Next, we apply the discussions in the previous section to the example of a smoothed $\ell_p$ ball sublevel-set constraint in $\mathbb{R}^d$. Constraints of this type are widely encountered in high-dimensional statistics and machine learning for regularization, such as in Bayesian Lasso ($p=1$) and ridge ($p=2$) regression \citep{ schmidt2005least,GHZ2022}.
Fix $d\ge 3$, $p \ge 1$ and $\varepsilon>0$, 
and define the smooth convex potential
\begin{equation}
\label{eq:smoothlp_Rd:f}
g(x)
:= \sum_{i=1}^{d} \left(x_i^2+\varepsilon^2\right)^{\frac{p}{2}},
\qquad x=(x_1,\dots,x_d)^\top\in\mathbb{R}^d .
\end{equation}
Then $\min_{x\in\mathbb R^d} g(x) = d\,\varepsilon^p$ (attained at $x=0$). For any $\lambda>d\,\varepsilon^p$, we consider the
sublevel-set domain
\begin{equation}
\label{eq:smoothlp_Rd:K}
K := \{x\in\mathbb{R}^d:\ g(x)\le \lambda\}.
\end{equation}
We have $\partial K=\{g=\lambda\}$ and $\nabla g(x)\neq 0$ on $\partial K$ since $\nabla g(x)=0$ if and only if $x=0$,
while $0\in K^\circ$ because $g(0)=d\varepsilon^p<\lambda$.
The outward unit normal vector on $\partial K$ is
\begin{equation}
\label{eq:smoothlp_Rd:normal}
\mathbf n(x)=\frac{\nabla g(x)}{\|\nabla g(x)\|},
\qquad x\in\partial K,
\end{equation}
where, for $i=1,\dots,d$,
\begin{equation}
\label{eq:smoothlp_Rd:gradf}
\partial_i g(x)
=
p\,x_i\left(x_i^2+\varepsilon^2\right)^{\frac{p}{2}-1}.
\end{equation}

Define
\begin{equation}
\label{eq:smoothlp_Rd:psi}
\psi(x):=\lambda-g(x).
\end{equation}
For each triple of indices $(a,b,c)$ with $1\le a<b<c\le d$, define the associated $3$-vector
\begin{equation}
\label{eq:smoothlp_Rd:k_abc}
\mathbf k^{(a,b,c)}(x):=
\begin{pmatrix}
\partial_a\psi(x)\\
\partial_b\psi(x)\\
\partial_c\psi(x)
\end{pmatrix}.
\end{equation}
Equivalently, for any $r\in\{a,b,c\}$,
\begin{equation}
\label{eq:smoothlp_Rd:dpsi}
\partial_r\psi(x)= -\,\partial_r g(x).
\end{equation}

For the choice \eqref{eq:smoothlp_Rd:f}, we have for each $r\in\{1,\dots,d\}$,
\begin{equation}
\label{eq:smoothlp_Rd:dpsi_explicit}
\partial_r\psi(x)
=
-\,\partial_r g(x)
=
-\,p\,x_r\left(x_r^2+\varepsilon^2\right)^{\frac{p}{2}-1}.
\end{equation}
Hence, for any triple $(a,b,c)$, we have:
\begin{equation}
\label{eq:smoothlp_Rd:k_explicit}
\mathbf k^{(a,b,c)}(x)
:=
\begin{pmatrix}
k^{(a,b,c)}_1(x) \\
k^{(a,b,c)}_2(x) \\
k^{(a,b,c)}_3(x)
\end{pmatrix}
=
-p\begin{pmatrix}
x_a\left(x_a^2+\varepsilon^2\right)^{\frac{p}{2}-1}\\[2pt]
x_b\left(x_b^2+\varepsilon^2\right)^{\frac{p}{2}-1}\\[2pt]
x_c\left(x_c^2+\varepsilon^2\right)^{\frac{p}{2}-1}
\end{pmatrix}.
\end{equation}
In particular, the components satisfy $k^{(a,b,c)}_j(x)=\partial_{\iota_j}\psi(x)$, where $(\iota_1,\iota_2,\iota_3)=(a,b,c)$.

We now select a family of disjoint coordinate triples that cover as many coordinates as possible.
Let
\[
m:=\left\lfloor \frac{d}{3}\right\rfloor,
\qquad
(a_\ell,b_\ell,c_\ell):=(3\ell-2,\,3\ell-1,\,3\ell),
\qquad \ell=1,\dots,m.
\]
(If $d\not\equiv 0\pmod 3$, the remaining $1$ or $2$ coordinates will be left without a swirl block.)

For each $\ell=1,\dots,m$, define a $d\times d$ matrix field $J^{(\ell)}(x)$ supported on the coordinates
$(a_\ell,b_\ell,c_\ell)$ by inserting the $3\times 3$ cross-product matrix generated by
$\mathbf k^{(a_\ell,b_\ell,c_\ell)}(x)$ into those rows and columns:
\begin{equation}
\label{eq:smoothlp_Rd:Jell}
\left[J^{(\ell)}(x)\right]_{\{a_\ell,b_\ell,c_\ell\}\times\{a_\ell,b_\ell,c_\ell\}}
=
\begin{pmatrix}
0 & -k^{(a_\ell,b_\ell,c_\ell)}_3(x) & k^{(a_\ell,b_\ell,c_\ell)}_2(x)\\
k^{(a_\ell,b_\ell,c_\ell)}_3(x) & 0 & -k^{(a_\ell,b_\ell,c_\ell)}_1(x)\\
-k^{(a_\ell,b_\ell,c_\ell)}_2(x) & k^{(a_\ell,b_\ell,c_\ell)}_1(x) & 0
\end{pmatrix},
\end{equation}
and set all other entries of $J^{(\ell)}(x)$ equal to $0$.
In particular, by \eqref{eq:smoothlp_Rd:dpsi_explicit}--\eqref{eq:smoothlp_Rd:k_explicit}, for the smoothed $\ell_p$ ball example,
\begin{align*}
k^{(a_\ell,b_\ell,c_\ell)}_1(x)
&=
-\,p\,x_{a_\ell}\left(x_{a_\ell}^2+\varepsilon^2\right)^{\frac{p}{2}-1},\\
k^{(a_\ell,b_\ell,c_\ell)}_2(x)
&=
-\,p\,x_{b_\ell}\left(x_{b_\ell}^2+\varepsilon^2\right)^{\frac{p}{2}-1},\\
k^{(a_\ell,b_\ell,c_\ell)}_3(x)
&=
-\,p\,x_{c_\ell}\left(x_{c_\ell}^2+\varepsilon^2\right)^{\frac{p}{2}-1}.
\end{align*}

We define
\begin{equation}
\label{eq:smoothlp_Rd:Jsum}
J(x):=\sum_{\ell=1}^{m} J^{(\ell)}(x),
\end{equation}
where $J^{(\ell)}(x)$ is defined in \eqref{eq:smoothlp_Rd:Jell}.
By construction, the matrix field $J(x)$ in \eqref{eq:smoothlp_Rd:Jsum} satisfies the required theoretical conditions.
First, for $\varepsilon>0$, the function $g$ is smooth, hence $\psi=\lambda-g$ is smooth as well; therefore each component $\partial_r\psi$ is $C^{1}$ on $K$, and so $J\in C^{1}(K)$ (Assumption~\ref{assump:regularity_local}).

Second, on $\partial K=\{g=\lambda\}=\{\psi=0\}$ we have $\nabla\psi=-\nabla g$, so $\nabla\psi(x)$ is colinear with the outward
unit normal $\mathbf n(x)=\nabla g(x)/\|\nabla g(x)\|$ (indeed anti-parallel). Since each local axial vector
$\mathbf k^{(a_\ell,b_\ell,c_\ell)}(x)$ is defined by selecting components of $\nabla\psi(x)$,
it follows that the restricted vectors
$(n_{a_\ell},n_{b_\ell},n_{c_\ell})^\top$ and $\mathbf k^{(a_\ell,b_\ell,c_\ell)}(x)$ are colinear on $\partial K$.
Consequently, on $\partial K$ the $3\times 3$ cross-product block generated by $\mathbf k^{(a_\ell,b_\ell,c_\ell)}(x)$ annihilates
$\mathbf n(x)$ in those coordinates, and hence
\[
J(x)\mathbf n(x)=0,
\qquad x\in\partial K
\]
(Assumption~\ref{assump:boundary}).

Finally, each block $J^{(\ell)}$ is divergence-free. Writing $(a,b,c)=(a_\ell,b_\ell,c_\ell)$ and
$\mathbf k=(k_1,k_2,k_3)^\top=(\partial_a\psi,\partial_b\psi,\partial_c\psi)^\top$, the nonzero entries are
$J_{ab}=-k_3$, $J_{ac}=k_2$, $J_{bc}=-k_1$ (and skew-symmetry).
Thus, for example, since $J_{aj}=0$ for all $j\notin\{b,c\}$,
\[
\left(\nabla\cdot J^{(\ell)}\right)_a
=\sum_{j=1}^d \partial_j J_{aj}
=\partial_b J_{ab}+\partial_c J_{ac}
=-\partial_b k_3+\partial_c k_2
=-(\partial_b\partial_c\psi)+(\partial_c\partial_b\psi)=0,
\]
and similarly $\left(\nabla\cdot J^{(\ell)}\right)_b=\left(\nabla\cdot J^{(\ell)}\right)_c=0$, while all other components vanish trivially.
Therefore $\nabla\cdot J^{(\ell)}=0$ for each $\ell$, and summing over $\ell$ yields $\nabla\cdot J=0$
(Assumption~\ref{assump:nablaJ}).

\paragraph{Standard $\ell_p$-ball constraint}
It is worth noting the conditions under which one can set the smoothing parameter $\varepsilon=0$ to recover the standard $\ell_p$-ball constraint. In the limit $\varepsilon \to 0$, the components $k_j(x)$ become proportional to $x |x|^{p-2}$ (interpreted as $\mathrm{sgn}(x)|x|^{p-1}$).
\begin{itemize}
    \item If $p \ge 2$, the function $x \mapsto x|x|^{p-2}$ belongs to $C^1(\mathbb{R})$ (and $C^2(\mathbb{R})$ if $p \ge 3$). In this case, the resulting matrix field $J(x)$ satisfies the $C^{1,\alpha}$ regularity required by Assumption~\ref{assump:regularity_local}, and thus we can set $\varepsilon=0$.
    \item If $1 \leq p < 2$, the derivative of $x|x|^{p-2}$ behaves like $|x|^{p-2}$, which becomes singular at $x=0$. Therefore, for $1 \leq p < 2$, a strictly positive smoothing parameter $\varepsilon > 0$ is necessary to ensure the validity of our theoretical framework.
\end{itemize}

\subsection{Preliminaries: Ergodicity and Invariant Measure}

We first introduce the following lemma from \cite{DFTWZ2025} which guarantees that $\mu$ is the unique invariant distribution for $(X_t)_{t\ge0}$ in  \eqref{eq:sde_local}.

\begin{lemma}[Existence and Uniqueness of Invariant Measure]\label{lem:invariant_measure_existence}
Under Assumptions \ref{assump:regularity_local}, \ref{assump:boundary} and \ref{assump:nablaJ}, the process $(X_t)_{t\ge0}$ defined by Eq. \eqref{eq:sde_local} with Neumann boundary conditions has a unique invariant probability measure $\mu \in \calP(K)$. This measure has a density with respect to Lebesgue measure $dx$ on $K$, given by
$d\mu(x) = \frac{1}{Z} e^{-f(x)} dx$, where $Z := \int_K e^{-f(x)} dx$.
The process $(X_t)_{t\ge0}$ is ergodic with respect to $\mu$.
\end{lemma}

%%%%%%%%%%%%%%%%%%%%%%%%%%%%%%%%%%%%%%%%%%%%%%
\subsection{The G{\"a}rtner-Ellis Framework}

We will establish LDP using the G{\"a}rtner-Ellis theorem. 
In the G{\"a}rtner-Ellis framework, the key object is the \textit{scaled cumulant generating function} (SCGF).

\begin{definition}[Scaled Cumulant Generating Function]
For $g \in C(K)$ (the space of continuous functions on $K$), the SCGF is defined as:
\[
    \lambda(g) = \lim_{t\to\infty} \frac{1}{t} \log \E^x \left[ \exp\left( \int_0^t g(X_s) ds \right) \right],
\]
\end{definition}
whenever the limit exists.
We need to show that $\lambda(g)$ exists, is finite, convex, Gateaux-differentiable, and does not depend on the starting point $x\in K$.
These properties are not merely technical: existence and finiteness ensure that $\lambda(g)$ is a well-defined long-time growth rate for the exponential moment; convexity and Gateaux differentiability are needed to identify the associated rate function via the Legendre--Fenchel transform and to characterize the unique ``tilted'' invariant measure through the gradient $\nabla \lambda(g)$; finally, independence of the initial condition guarantees that $\lambda(g)$ describes an intrinsic asymptotic quantity of the dynamics rather than a transient effect.
This naturally leads to the study of the Feynman--Kac semigroup and its spectral properties.

%%%%%%%%%%%%%%%%%%%%%%%%%%%%%%%%%%%%%%
\subsubsection{Feynman-Kac Semigroup and Spectral Properties}
Define the Feynman-Kac semigroup $P_t^g : C(K) \to C(K)$ by
\[ (P_t^g \phi)(x) = \E^x \left[ \phi(X_t) \exp\left( \int_0^t g(X_s) ds \right) \right]. \]
The SCGF $\lambda(g)$ is related to the principal eigenvalue of the operator $\calL_J + g$.
We will prove for any $g \in C(K)$:
\begin{enumerate}
    \item $\lambda(g)$ exists, is finite, and is independent of $x \in K$.
    \item $\lambda(g)$ is the principal eigenvalue of the operator $\calL_J + g$ (acting on functions satisfying the boundary condition). That is, there exists a unique positive eigenfunction $h_g \in C(K)$ (unique up to scaling) such that $(\calL_J + g)h_g = \lambda(g)h_g$.
    \item The function $\lambda: C(K) \to \R$ is convex and continuous.
    \item The function $\lambda:C(K)\to\mathbb R$ is Gateaux-differentiable. More precisely, for
    any $g,\psi\in C(K)$ the limit
    \[
    D\lambda(g)[\psi]=\lim_{\varepsilon\to 0}\frac{\lambda(g+\varepsilon\psi)-\lambda(g)}{\varepsilon}
    \]
    exists and is given by
    \[
    D\lambda(g)[\psi]=\int_K \psi(x)\,\nu_g(dx),
    \]
    where $\nu_g$ is the probability measure built from the principal right/left eigenfunctions.
\end{enumerate}

\paragraph{Eigenfunction representation of the derivative and the tilted measure}
    Let us provide a more specific explanation of the 4th paragraph above. Let $h_g>0$ and $h_g^*>0$ be the principal right and left eigenfunctions of
    $\mathcal L_J+g$ and $(\mathcal L_J+g)^*$, respectively, normalized by
    $\int_K h_g(x)\,h_g^*(x)\,dx=1$.
    Equivalently, in terms of the Feynman--Kac semigroup and its $L^2(K,dx)$-adjoint,
    \[
    P_t^g h_g = e^{\lambda(g)t}h_g,
    \qquad
    P_t^{g,*} h_g^* = e^{\lambda(g)t}h_g^*,
    \]
    and therefore for every $\varphi\in C(K)$,
    \[
    \int_K (P_t^g\varphi)(x)\,h_g^*(x)\,dx
    =
    e^{\lambda(g)t}\int_K \varphi(x)\,h_g^*(x)\,dx.
    \]
    Define
    \[
    \nu_g(dx):=h_g(x)\,h_g^*(x)\,dx \qquad \left(\text{so that }\nu_g(K)=1\right).
    \]
    Then $D\lambda(g)[\psi]=\int_K \psi\,d\nu_g$ as claimed. Moreover, $\nu_g$ is the invariant probability measure of the Doob $h$-transform
    of the Feynman--Kac semigroup:
    \[
    \widetilde P_t^g\varphi(x)
    :=e^{-\lambda(g)t}\frac{1}{h_g(x)}\,P_t^g(h_g\varphi)(x),
    \qquad \varphi\in C(K).
    \]
    Indeed, using $P_t^g h_g=e^{\lambda(g)t}h_g$ and
    $\int_K (P_t^g\phi)\,h_g^*dx=e^{\lambda(g)t}\int_K \phi\,h_g^*dx$,
    one can check that
    \[
    \int_K \widetilde P_t^g\varphi\,d\nu_g=\int_K \varphi\,d\nu_g,
    \]
    so that $\nu_g$ is the unique invariant measure associated with the tilted dynamics.
    %\textit{p.10, last two lines of numbered paragraph 4: I do not see how the property $P_{t}^{g}\mathbf{1}=e^{\lambda(g)t}$ is related to the measure $\nu_{g}$. Can you explain?}

\begin{lemma}[Existence of $\lambda(g)$ and a principal eigenfunction]\label{existence}
For any $g \in C(K)$, the Feynman--Kac semigroup $P_t^g$ defined by
\[
(P_t^g \phi)(x) = \E^x \left[ \phi(X_t)\exp\!\left( \int_0^t g(X_s)\,ds \right)\right]
\]
acts on $C(K)$ and has the following properties:
\begin{enumerate}
    \item $P_t^g$ is a compact and strongly positive operator on $C(K)$ for each $t>0$.

    \item For each fixed $t>0$, the spectral radius $r(P_t^g)>0$ is a simple eigenvalue of $P_t^g$.
    There exists an eigenfunction $h_{t,g}\in C(K)$ such that $h_{t,g}>0$ on $K$ and
    \[
        P_t^g h_{t,g} = r(P_t^g)\,h_{t,g}.
    \]
    Writing $r(P_t^g)=e^{t\lambda_t(g)}$ defines $\lambda_t(g)\in\R$.

    \item There exist $\lambda(g)\in\R$ and a strictly positive function $h_g\in C(K)$ such that
    \[
        P_t^g h_g = e^{t\lambda(g)}h_g, \quad \text{for any $t\ge 0$}.
    \]
    Moreover, $h_g\in D(\mathcal L_J+g)$ and
    \[
    (\mathcal L_J+g)h_g=\lambda(g)h_g.
    \]
    In particular, for every $t>0$ the principal eigenfunction of $P_t^g$ is unique up to normalization,
    so that $h_{t,g}$ coincides with $h_g$ after rescaling and $\lambda_t(g)=\lambda(g)$.

    \item Consequently, the limit in Definition~2 exists and equals $\lambda(g)$:
    \[
        \lambda(g)=\lim_{t\to\infty}\frac1t\log \E^x\left[\exp\!\left(\int_0^t g(X_s)\,ds\right)\right],
    \]
  and it is finite and independent of the starting point $x\in K$.
\end{enumerate}
\end{lemma}

\begin{proof}
    The proof will be provided in Appendix~\ref{app:existence}.
\end{proof}

\begin{lemma}\label{convexity}
The function $\lambda: C(K) \to \R$ is convex.
\end{lemma}

\begin{proof}
 The proof will be provided in Appendix~\ref{app:proof_convexity}.
\end{proof}

\begin{lemma}%[Continuity and Gateaux-Differentiability of $\lambda(g)$]
\label{gateaux}
    The function $\lambda: \CK \to \R$ is continuous and Gateaux-differentiable.
\end{lemma}

\begin{proof}
The proof will be provided in Appendix~\ref{app:proof_gateaux}.
\end{proof}

%%%%%%%%%%%%%%%%%%%%%%%%%%%%%%%%%%
%\vspace{-0.1in}
\subsubsection{Exponential Tightness}

A key technical condition for establishing a large deviation principle (LDP) is \textit{exponential tightness} of the family of probability measures. Exponential tightness ensures that the probability mass does not escape to regions of the space that are not compact, which is essential for the LDP to hold and for the rate function to have desirable properties such as being a good rate function (i.e., having compact level sets). We refer to \cite[Section 1.2]{DZ1998} for a detailed discussion. By definition (see, e.g., \cite[Definition 1.2.18]{DZ1998}), a family of probability measures $(\mathbb{Q}_t)_{t \ge 0}$ on a Polish space $\mathcal{X}$ is exponentially tight if for every $A < \infty$, there exists a compact set $\Gamma_A \subset \mathcal{X}$ such that
    $\limsup_{t\to\infty} \frac{1}{t} \log \mathbb{Q}_t(\Gamma_A^c) \le -A$.
 In our context, $\mathcal{X} = \calP(K)$, and $\mathbb{Q}_t(\cdot) = \Prob(L_t \in \cdot)$.
We establish exponential tightness in the next following lemma.

\begin{lemma}[Exponential Tightness]\label{Exponential-Tightness}
The family of probability measures $(\Prob(L_t \in \cdot))_{t\ge 0}$ is exponentially tight in $\calP(K)$ equipped with the weak topology.
\end{lemma}

\begin{proof}
The proof will be provided in Appendix~\ref{app:proof_exponential_tightness}.
\end{proof}

%%%%%%%%%%%%%%%%%%%%%%%%%%%%%%%%%%%%%%%%%%%%
%\vspace{-0.1in}
\subsection{Large Deviation Principle}

Now, we are ready to state the main result
of the paper.

\begin{theorem}[LDP for Empirical Measures]\label{thm:ldp_empirical_measures}
    Assume $K$ is a compact metric space. Let $(X_t)_{t \ge 0}$ be the skew-reflected non-reversible Langevin dynamics on $K$ in \eqref{eq:sde_local} with unique invariant probability measure $\muinv \in \calP(K)$. Suppose the scaled cumulant generating function $\lambda(g)$ defined for $g \in \CK$ exists, is finite, convex, and Gateaux-differentiable on $\CK$. Further, assume the family of empirical measures $(L_t)_{t\ge 0}$ is exponentially tight in $\calP(K)$ equipped with the weak topology.
    
    Then, the empirical measures $(L_t)_{t \ge 0}$ satisfy a large deviation principle in $\calP(K)$ (equipped with the weak topology) with speed $t$ and good rate function $I: \calP(K) \to [0, \infty]$ given by
    \[ I(\nu) = \sup_{g \in \CK} \left\{ \int_K g(y) d\nu - \lambda(g) \right\}. \]
    The rate function $I(\nu)$ has the following properties:
    \begin{itemize}
        \item[(a)] $I(\nu)$ is convex and lower semi-continuous.
        \item[(b)] $I(\nu) \ge 0$ for all $\nu \in \calP(K)$.
        \item[(c)] $I(\nu) = 0$ if and only if $\nu = \muinv$.
        \item[(d)] If $I(\nu) < \infty$, then $\nu$ is absolutely continuous with respect to $\muinv$ ($\nu \ll \muinv$).
    \end{itemize}
    \end{theorem}
    
    \begin{proof}
    The main statement of the LDP follows directly from the G{\"a}rtner-Ellis Theorem (e.g., \cite[Theorem 2.3.6]{DZ1998}).
    \begin{enumerate}
        \item The existence, finiteness, convexity, and Gateaux-differentiability of $\lambda(g)$ on $\CK$ have been guaranteed in Lemmas \ref{existence}, \ref{convexity}, \ref{gateaux}.
        \item Exponential tightness of $(L_t)_{t\ge 0}$ in $\calP(K)$ is proved in Lemma \ref{Exponential-Tightness}.
        \item The G{\"a}rtner-Ellis theorem then states that an LDP holds for $(L_t)_{t\ge 0}$ with speed $t$ and rate function $I(\nu) = \lambda^*(\nu)$, where $\lambda^*$ is the Legendre-Fenchel transform of $\lambda(g)$, as given in the theorem statement.
        \item \textbf{Property (a):} The rate function $I(\nu)$, being the supremum of affine (and thus convex and lower semi-continuous) functions, is itself convex and lower semi-continuous. Since $K$ is compact, $\calP(K)$ is also compact in the weak topology. For a lower semi-continuous function on a compact space, all its level sets $\{ \nu \in \calP(K) : I(\nu) \le \alpha \}$ are closed, and thus compact. Hence, $I(\nu)$ is a good rate function.
    
        \item \textbf{Property (b): $I(\nu) \ge 0$:}
        By definition, $I(\nu) = \sup_{g \in \CK} \{ \int_K g d\nu - \lambda(g) \}$.
        Choose $g \equiv 0 \in \CK$. Then the SCGF for $g \equiv 0$ is $\lambda(0) =  0$. Therefore, $I(\nu) \ge 0$.
    
        \item \textbf{Property (c): $I(\nu) = 0 \iff \nu = \muinv$:}
        We know $I(\muinv) \ge 0$. For any $g \in \CK$, by Jensen's inequality and the stationarity of $\muinv$:
        \begin{align*} \label{eq:lambda_ge_integral_mu}
        \lambda(g) = \lim_{t\to\infty} \frac{1}{t} \log \E^x \left[ \exp\left( \int_0^t g(X_s) ds \right) \right]\ge \lim_{t\to\infty} \frac{1}{t} \E^x \left[ \int_0^t g(X_s) ds \right]
        = \int_K g d\muinv,
        \end{align*}
        where, by ergodicity, the limit is independent of $x$ and equals expectation w.r.t. $\muinv$.
        This implies $\int_K g d\muinv - \lambda(g) \le 0$ for all $g \in \CK$.
        Therefore, 
        \[
            I(\muinv) = \sup_{g \in \CK} \left\{ \int_K g d\muinv - \lambda(g) \right\} \le 0.
        \]
        Combined with $I(\muinv) \ge 0$, we have $I(\muinv) = 0$.

        For the reverse implication, we use the variational characterization of the rate function established in \cite[Section 2]{DV2}. Specifically, if $I(\nu)=0$, then by Proposition \ref{prop:dv}, we must have $\int_K \frac{\mathcal{L}_J u}{u} d\nu \ge 0$ for all positive test functions $u$. Considering perturbations of the form $u = 1 + \epsilon \phi$ for small $\epsilon$ and arbitrary smooth $\phi$, this condition implies $\int_K \mathcal{L}_J \phi \, d\nu = 0$, which is precisely the definition of an invariant measure. Since $\muinv$ is the \emph{unique} invariant measure (Lemma \ref{lem:invariant_measure_existence}), it follows that $\nu=\muinv$.
        
        %\textit{One thing that made the paper harder to read was that the authors at times refer to book chapters rather than specific results they are using from the specific book chapter. The authors should refer to the specific result they are using so that the reader does not have to read a whole book chapter to know/understand the comment. For example, see p.13 reference to Donsker-Varadhan; p. 24, reference to Pazy, 1983, Chapter 6; p. 25 reference to Pavliotis, 2014, Chapter 4; the references to Kato, Chapters II and VII on p.43.}
    
        \item \textbf{Property (d): $I(\nu) < \infty \implies \nu \ll \muinv$:}
        This property is established using the spectral characterization of $\lambda(g)$ combined with elliptic regularity estimates. We proceed by contradiction.

        Suppose $\nu$ is not absolutely continuous with respect to $\muinv$. Then there exists a Borel set $A \subset K$ such that $\muinv(A) = 0$ but $\nu(A) > 0$. By the regularity of measures on the compact metric space $K$, there exists a closed set $F \subset A$ such that $\nu(F) =: \delta > 0$. Note that $\muinv(F)=0$.

        Fix an arbitrary constant $M > 0$. We construct a sequence of continuous approximations. Let $d(x, F)$ be the distance from $x$ to $F$. For $k \in \mathbb{N}$, define $g_{k} \in C(K)$ by:
        \[
            g_{k}(x) = M \cdot \max\left(0, 1 - k \cdot d(x, F)\right).
        \]
        Then $0 \le g_k \le M$ on $K$, and $g_k(x)\to M\mathbf 1_F(x)$ pointwise.

        We now control the principal eigenvalue $\lambda(g_k)$. First, by the variational representation and $0 \le g_k \le M$, we have
        $0 \le \lambda(g_k) \le \|g_k\|_\infty \le M$ for all $k$.

        Let $h_k \in C^2(K)$ be the strictly positive principal eigenfunction of $\mathcal{L}_J + g_k$
        subject to the (Neumann) boundary condition, normalized by $\int_K h_k \, d\muinv = 1$.
        Since $\muinv$ is invariant, $\int_K \mathcal{L}_J \phi \, d\muinv = 0$ for any sufficiently regular
        $\phi$ satisfying the boundary condition. Applying this with $\phi=h_k$ in
        $(\mathcal{L}_J + g_k)h_k = \lambda(g_k)h_k$ yields
        \[
            \lambda(g_k) = \int_K g_k\, h_k \, d\muinv .
        \]

        It remains to show $\lambda(g_k)\to 0$. For this we derive a uniform $L^\infty$ bound for $h_k$.
        Rewrite the eigenvalue equation as
        \[
            \mathcal{L}_J h_k + c_k(x) h_k = 0,
            \qquad c_k(x):=g_k(x)-\lambda(g_k),
        \]
        with the Neumann boundary condition. Since $0\le g_k\le M$ and $0\le \lambda(g_k)\le M$,
        we have the uniform bound $\|c_k\|_\infty \le 2M$.

        \smallskip
        \noindent\emph{Local Harnack inequality (interior and boundary).}
        If $B_{2r}(x)\subset K^\circ$, then by the interior Harnack inequality
        \cite[Theorem~8.20]{GilbargTrudinger2001} applied to $h_k\ge 0$ applied to $h_k\ge 0$ solving a uniformly elliptic equation with bounded (in fact $C^{\alpha}$) coefficients and zero-order term bounded by $2M$, there exists $C_{\rm int}>1$
        (independent of $k$) such that
        \[
            \sup_{y\in B_r(x)} h_k(y) \le C_{\rm int}\, \inf_{y\in B_r(x)} h_k(y) .
        \]
        If $x\in \partial K$, then the Neumann condition is a special case of the oblique derivative
        boundary condition (with obliqueness bounded below), and the boundary Harnack inequality for
        nonnegative solutions yields constants $r_{\rm bd}>0$ and $C_{\rm bd}>1$ (depending only on the ellipticity constants (here $a_{ij}=\delta_{ij}$),
        the drift bound $\|b\|_{L^\infty(K)}$ with $b=(I+J)\nabla f$, 
        the zeroth-order bound $\|c_k\|_\infty\le 2M$,
        and the $C^{2,\alpha}$-geometry of $K$, but independent of $k$) such that for all $0<r\le r_{\rm bd}$,
        \[
            \sup_{y\in B_r(x)\cap K} h_k(y)\le C_{\rm bd}\, \inf_{y\in B_r(x)\cap K} h_k(y).
        \]
        (See \cite[Corollary~3.5]{lieberman1987local}.)

        Set $r_0 := \min\{r_{\rm bd},\, 1\}$ and $C_0 := \max\{C_{\rm int},\, C_{\rm bd}\}$.
        Then for every $x\in K$ and $0<r\le r_0$, we have the uniform (in $k$) estimate
        \begin{equation}\label{eq:local-harnack-k}
            \sup_{y\in B_r(x)\cap K} h_k(y) \le C_0\, \inf_{y\in B_r(x)\cap K} h_k(y).
        \end{equation}

        \smallskip
        \noindent\emph{Global Harnack inequality via a finite chain.}
        Fix $x^+,x^-\in K$ such that $h_k(x^+)=\sup_{x\in K} h_k(x)$ and $h_k(x^-)=\inf_{x\in K} h_k(x)$.
        Since $K$ is compact and connected, there exists a finite sequence
        $x_0=x^-,x_1,\dots,x_N=x^+$ such that
        \[
            \left(B_{r_0/2}(x_i)\cap K\right)\cap \left(B_{r_0/2}(x_{i+1})\cap K\right)\neq\varnothing,
            \qquad i=0,\dots,N-1,
        \]
        where $N$ depends only on $K$ and $r_0$. Applying \eqref{eq:local-harnack-k} successively along the
        chain yields $h_k(x_{i+1}) \le C_0\, h_k(x_i)$. Hence,
        \[
            \sup_{x\in K} h_k(x) = h_k(x^+) \le C_0^N\, h_k(x^-) = C_0^N \inf_{x\in K} h_k(x).
        \]
        Denoting $C_H:=C_0^N$, we obtain the global estimate $\sup_{x\in K} h_k(x) \le C_H \inf_{x\in K} h_k(x)$ with $C_H$
        independent of $k$.

        Using $\int_K h_k\, d\muinv = 1$ and $\muinv(K)=1$, we have
        \[
            1=\int_K h_k\, d\muinv \ge \inf_{x\in K} h_k(x),
        \]
        and thus $\|h_k\|_\infty = \sup_{x\in K} h_k(x) \le C_H$ for all $k$.

        Substituting this into $\lambda(g_k)=\int_K g_k h_k\, d\muinv$ gives
        \[
            0\le \lambda(g_k) \le \|h_k\|_\infty \int_K g_k\, d\muinv \le C_H \int_K g_k\, d\muinv.
        \]
        Since $0\le g_k\le M$ and $g_k\to M\mathbf 1_F$ pointwise with $\muinv(F)=0$, dominated
        convergence yields $\int_K g_k\, d\muinv \to 0$, and hence $\lambda(g_k)\to 0$.

        Finally, by definition of the rate function,
        \[
            I(\nu) \ge \int_K g_k\, d\nu - \lambda(g_k).
        \]
        Since $g_k \ge M\mathbf 1_F$ and $\nu(F)=\delta$, we have $\int_K g_k\, d\nu \ge M\delta$,
        and letting $k\to\infty$ gives $I(\nu)\ge M\delta$. As $M>0$ is arbitrary and $\delta>0$,
        we conclude $I(\nu)=\infty$. Therefore $I(\nu)<\infty$ implies $\nu\ll \muinv$.
        
    \end{enumerate}
    This completes the proof.
    \end{proof}
    
%%%%%%%%%%%%%%%%%%%%%%%%%%%%%%%%%%%%%%%
%\vspace{-0.1in}
\subsection{Donsker-Varadhan Variational Formula and Decomposition of the Rate Function}
This part is inspired by \cite{LDP-GG}, Section 3, and we adopt their framework to the setting of a compact domain.

%%%%%%%%%%%%%%%%%%%%%%%%%%%%%%%%%
\subsubsection{Donsker-Varadhan Formula for \texorpdfstring{$I(\nu)$}{I(nu)}}
\begin{definition}[Space $D^+(\calL_J)$]
    Let\footnote{Since $K$ is compact, if $u \in \CK$, $u>0$, and $\calL_J u \in \CK$, then $\frac{\calL_J u}{u} \in \CK$, which implies $\frac{\calL_J u}{u} \in \CbK$.} 
    \[
        D^+(\calL_J) = \left\{ u \in \DomLJ \cap \CK : u(x) > 0 \text{ for all } x \in K, \text{ and } \frac{\calL_J u}{u} \in \CbK \right\}.
    \]
    \end{definition}
    
    \begin{proposition}[Donsker-Varadhan Variational Formula]\label{prop:dv}
    For any $\nu \in \calP(K)$:
    \begin{enumerate}
        \item If $\nu \not\ll \muinv$, then $I(\nu) = \infty$. In this case, 
        \[
            \sup_{u \in D^+(\calL_J)} \left\{ -\int_K \frac{\calL_J u}{u} d\nu \right\} = \infty.
        \]
        \item If $\nu \ll \muinv$, then the rate function $I(\nu)$ can be expressed as:
        \[ I(\nu) = \sup_{u \in D^+(\calL_J)} \left\{ -\int_K \frac{\calL_J u}{u} d\nu \right\}. \]
    \end{enumerate}
    \end{proposition}
    
    \begin{proof}
    The proof strategy for showing the equivalence of the Legendre-Fenchel transform $I(\nu)$ and the variational form 
    $I_V(\nu) := \sup_{u \in D^+(\calL_J)} \left\{ -\int_K \frac{\calL_J u}{u} d\nu \right\}$
    largely follows the approach in \cite[Section 6.2.2]{LDP-GG}, adapted to our compact setting.

    \textbf{Part 1: The case when $\nu \not\ll \muinv$.}

    As established in Property (d) of Theorem~\ref{thm:ldp_empirical_measures} (LDP for Empirical Measures), if $\nu \not\ll \muinv$, then 
    $$I(\nu) = \sup_{g \in \CK} \left\{ \int_K g d\nu - \lambda(g) \right\} = \infty.$$
    In Step 1 of Part 2 below, we show that for any $\nu \in \calP(K)$, it holds that $I_V(\nu) \ge I(\nu)$, where 
    $I_V(\nu) := \sup_{u \in D^+(\calL_J)} \left\{ -\int_K \frac{\calL_J u}{u} d\nu \right\}$.
    Specifically, for any $g \in \CK$, by choosing $u = h_g$ (the principal eigenfunction of $\calL_J+g$), we have $u \in D^+(\calL_J)$ and 
    $$-\int_K \frac{\calL_J u}{u} d\nu = \int_K g d\nu - \lambda(g).$$
    Since this holds for all $g$, taking the supremum over $g$ yields $I_V(\nu) \ge I(\nu)$.

    Therefore, if $\nu \not\ll \muinv$, we have $I_V(\nu) \ge I(\nu) = \infty$.
    This implies $I_V(\nu) = \infty$.
    
    \textbf{Part 2: The case when $\nu \ll \muinv$.}

    Let
    \[
        I_V(\nu) = \sup_{u \in D^+(\calL_J)} \left\{ -\int_K \frac{\calL_J u}{u} \,d\nu \right\}.
    \]
    We aim to show $I(\nu) = I_V(\nu)$.

    \textit{Step 1}: Proof of $I_V(\nu) \ge I(\nu)$.
    Let $g\in C(K)$ and let $h_g>0$ be the principal eigenfunction of $\calL_J+g$, i.e.
    \[
    (\calL_J+g)h_g=\lambda(g)h_g.
    \]
    Dividing by $h_g$ gives
    \[
    -\frac{\calL_J h_g}{h_g}=g-\lambda(g).
    \]
    Since $h_g\in D(\calL_J)\cap C(K)$ and $g-\lambda(g)\in C(K)$ is bounded on $K$, we have
    $h_g\in D^+(\calL_J)$. Therefore,
    \[
    I_V(\nu)\ge -\int_K \frac{\calL_J h_g}{h_g}\,d\nu
    = \int_K (g-\lambda(g))\,d\nu
    = \int_K g\,d\nu-\lambda(g),
    \]
    where we used $\nu(K)=1$. Taking the supremum over $g\in C(K)$ yields
    $I_V(\nu)\ge I(\nu)$.

    \textit{Step 2}: Proof of $I(\nu) \ge I_V(\nu)$.
    Fix $u\in D^+(\calL_J)$ and define
    \[
    f_u:=-\frac{\calL_J u}{u}\in C(K).
    \]
    Choosing $g=f_u$ in the variational formula for $I(\nu)$ gives
    \[
    I(\nu)\ge \int_K f_u\,d\nu-\lambda(f_u)
    = -\int_K \frac{\calL_J u}{u}\,d\nu-\lambda(f_u).
    \]
    It remains to show that $\lambda(f_u)\le 0$ for all $u\in D^+(\calL_J)$.
    Note that
    \[
    (\calL_J+f_u)u=\calL_J u+\left(-\frac{\calL_J u}{u}\right)u=0.
    \]
    For $t\ge0$, the Feynman--Kac formula gives
    \[
    P_t^{f_u}u(x)=\E^x\!\left[u(X_t)\exp\!\left(\int_0^t f_u(X_s)\,ds\right)\right].
    \]
    Applying It\^o's formula to $e^{\int_0^t f_u(X_s)\,ds}u(X_t)$ yields
    \begin{align*}
    d\!\left(e^{\int_0^t f_u(X_s)\,ds}u(X_t)\right)
    &=e^{\int_0^t f_u(X_s)\,ds}\left((f_u u+\calL_J u)(X_t)\,dt
    +\sqrt{2}\,\nabla u(X_t)\cdot dW_t\right).
    \end{align*}
    Since $(f_u u+\calL_J u)=0$, the drift vanishes, hence
    \[
    M_t:=e^{\int_0^t f_u(X_s)\,ds}u(X_t)
    \]
    is a local martingale. Because $u$ is continuous and strictly positive on compact $K$,
    $0<\min_K u\le u\le \max_K u<\infty$, and $f_u$ is bounded; thus $M_t$ is in fact a martingale and
    $\E^x[M_t]=M_0=u(x)$. Consequently,
    \[
    u(x)=\E^x\!\left[e^{\int_0^t f_u(X_s)\,ds}u(X_t)\right]
    \ge \left(\min_{y\in K}u(y)\right)\,\E^x\!\left[e^{\int_0^t f_u(X_s)\,ds}\right].
    \]
    Therefore,
    \[
    \frac1t\log \E^x\!\left[e^{\int_0^t f_u(X_s)\,ds}\right]
    \le \frac1t\log\!\left(\frac{u(x)}{\min_{y\in K}u(y)}\right)\xrightarrow[t\to\infty]{}0,
    \]
    which implies $\lambda(f_u)\le 0$. Hence
    \[
    I(\nu)\ge -\int_K \frac{\calL_J u}{u}\,d\nu
    \quad\text{for all }u\in D^+(\calL_J),
    \]
    and taking the supremum over $u$ yields $I(\nu)\ge I_V(\nu)$.
    
    Therefore, for any fixed $u \in D^+(\calL_J)$ and $g=f_u=-\frac{\calL_J u}{u}$,
    \[
        I(\nu)=\sup_{g\in\CK}\left\{\int_K g\,d\nu-\lambda(g)\right\}
        \ge \int_K f_u\,d\nu-\lambda(f_u)
        \ge -\int_K \frac{\calL_J u}{u}\,d\nu,
    \]
    where the last inequality uses $\lambda(f_u)\le 0$.
    Taking the supremum over $u\in D^+(\calL_J)$ yields $I(\nu)\ge I_V(\nu)$.
    Together with Step~1, we conclude $I(\nu)=I_V(\nu)$.
    \end{proof}
    
    \begin{remark}
    The definition of $D^+(\calL_J)$ requiring $\frac{\calL_J u}{u} \in \CbK$ is crucial for $f_u = -\frac{\calL_J u}{u}$ to be in $\CbK$, so that $\lambda(f_u)$ is well-defined within the established framework for $\lambda(g)$ where $g \in \CK$ (or $\CbK$). On a compact domain $K$, if $u \in C(K)$ and $u>0$, and $\calL_J u \in C(K)$, then $\frac{\calL_J u}{u}$ is automatically continuous, and hence bounded, so that the condition simplifies slightly. The core requirement is that $u$ is sufficiently regular for $\calL_J u$ to be well-defined and continuous, and $u$ itself is continuous and strictly positive.
    \end{remark}

%%%%%%%%%%%%%%%%%%%%%%%%%%%%%%%%%%%%%%%%%%
\subsubsection{Decomposition of \texorpdfstring{$I(\nu)$}{I(nu)} into Symmetric and Skew-Symmetric Parts}\label{subsec:decomposition_rate_function}

Let $\mu(dx) = \frac{1}{Z}e^{-f(x)}dx$ be the unique invariant probability measure for the process $(X_t)_{t \ge 0}$ in \eqref{eq:sde_local} under Assumptions \ref{assump:regularity_local}, \ref{assump:boundary}, and \ref{assump:nablaJ}, as established in Lemma \ref{lem:invariant_measure_existence}. The infinitesimal generator of $(X_t)_{t \ge 0}$ is given by $\mathcal{L}_J u(x) = \Delta u(x) - \inner{\nabla f(x)}{\nabla u(x)} - \inner{J(x)\nabla f(x)}{\nabla u(x)}$,
where the domain $D(\mathcal{L}_J) = \{ u \in C^2(K) : \inner{\nabla u(x)}{(I + J(x))\mathbf n(x)} = 0 \text{ on } \partial K \}$ (where $\mathbf n(x)$ is the outward unit normal vector) simplifies to functions satisfying Neumann boundary conditions, i.e., $\inner{\nabla u(x)}{\mathbf n(x)} = 0$ on $\partial K$, due to Assumption \ref{assump:boundary}.

To analyze the structure of the rate function $I(\nu)$, we decompose the infinitesimal generator $\mathcal{L}_J$ with respect to the invariant measure $\mu$. Let $\mathcal{L}_J^*$ denote the adjoint of $\mathcal{L}_J$ in the Hilbert space $L^2(K,d\mu)$. The symmetric part $\mathcal{L}_S$ and the skew-symmetric part $\mathcal{L}_A$ of $\mathcal{L}_J$ are defined as:
\begin{align}
  \mathcal{L}_S  = \frac{1}{2}(\mathcal{L}_J + \mathcal{L}_J^*), 
  \qquad
  \mathcal{L}_A  = \frac{1}{2}(\mathcal{L}_J - \mathcal{L}_J^*). 
\end{align}
Under Assumptions~\ref{assump:regularity_local} and~\ref{assump:boundary}, and using in addition Assumption~\ref{assump:nablaJ} (the divergence-free condition
$\nabla\!\cdot J=0$; in particular, since $J$ is skew-symmetric and $f\in C^2$, this implies
$\nabla\!\cdot(J\nabla f)=0$), the generator admits the symmetric/antisymmetric decomposition
$\mathcal L_J=\mathcal L_S+\mathcal L_A$ in $L^2(\mu)$: $\mathcal L_S$ is self-adjoint and
$\mathcal L_A$ is anti-self-adjoint. Moreover,
%\textit{p. 16, line above equation (2.6): Aren't you also using the divergence-free condition here? If so, I would emphasize it.}
$\mathcal{L}_S$ is self-adjoint, $\mathcal{L}_A$ is anti-self-adjoint with respect to $\mu$, and they have the explicit expressions (see Appendix \ref{app:operator_decomposition_simplified} for derivation):
\begin{align}
  \mathcal{L}_S  u(x) = \Delta u(x) - \inner{\nabla f(x)}{\nabla u(x)}, \qquad
  \mathcal{L}_A  u(x) = -\inner{J(x)\nabla f(x)}{\nabla u(x)}. 
\end{align}
Note that $\mathcal{L}_S$ is the infinitesimal generator of the reversible reflected Langevin dynamics (RLD) \eqref{reflected:SDE}. The following proposition details the decomposition of $I(\nu)$ based on these operators.

In the following, we will use the notation $\|\cdot\|_{\mathcal{H}^1_S(\nu)}$ for the Sobolev $H^1$-type semi-norm with respect to the measure $\nu$, defined by
$\|\varphi\|_{\mathcal{H}^1_S(\nu)}^2 := \int_K \|\nabla \varphi\|_{2}^2 d\nu$
for a sufficiently regular function $\varphi$. The corresponding dual semi-norm, denoted $\|\cdot\|_{\mathcal{H}_S^{-1}(\nu)}$, for an element $\varphi$ is implicitly defined via a variational problem. Specifically, 
\begin{equation}\label{eq:dual_sobolev_seminorm}
\|\varphi\|_{\mathcal{H}_S^{-1}(\nu)}^2 = \sup_{\psi\in D^+(\mathcal L_J)}\left\{ 2\int_K \psi\varphi d\nu - \|\psi\|_{\mathcal{H}^1_S(\nu)}^2 \right\},
\end{equation}
which is consistent with usage in literature such as \cite{LDP-GG}.
The proposition below details the decomposition of $I(\nu)$ using these concepts.

\begin{proposition}[Decomposition of the Rate Function]\label{prop:rate_decomposition}
Suppose Assumptions \ref{assump:regularity_local}, \ref{assump:boundary}, and \ref{assump:nablaJ} hold.
For any $\nu\in\mathcal P(K)$ with $\nu\ll\mu$ and $v=\log(d\nu/d\mu)$ satisfying
Remark~\ref{rem:admissible_v}, the rate function $I(\nu)$ given by Proposition \ref{prop:dv} can be decomposed as:
$I(\nu) = I_S(\nu) + I_A(\nu)$,
where
\begin{enumerate}[label=(\alph*)]
    \item $I_S(\nu) = \frac14 \int_K \|\nabla v\|_2^2 d\nu =: \frac14 \|v\|_{\mathcal{H}^1_S(\nu)}^2$. This term corresponds to the rate function for the reversible reflected Langevin dynamics (RLD) \eqref{reflected:SDE} governed by $\mathcal{L}_S$.
    \item $I_A(\nu) = \frac14 \|\mathcal{L}_A v\|_{\mathcal{H}_S^{-1}(\nu)}^2$, where $\|\cdot\|_{\mathcal{H}_S^{-1}(\nu)}$ is defined in \eqref{eq:dual_sobolev_seminorm}.
    \item The term $I_A(\nu)$ is non-negative, i.e., $I_A(\nu) \ge 0$.
\end{enumerate}
If $\nu$ is not absolutely continuous with respect to $\mu$, then $I(\nu)=\infty$.
In the absolutely continuous case, the decomposition holds whenever the admissibility conditions in Remark~\ref{rem:admissible_v} hold;
otherwise we set $I(\nu)=\infty$.
\end{proposition}

\begin{remark}[Admissibility and regularity of $v$]\label{rem:admissible_v}
In Proposition~\ref{prop:rate_decomposition}, recall that we write $v=\log(d\nu/d\mu)$. We assume that ``$v$ is sufficiently regular'', that is:

\begin{itemize}
  \item[(i)] $v$ belongs to the Sobolev space associated with the symmetric part, i.e.
  \[
    v\in \mathcal H_S^1(\nu)
    \quad\text{if and only if}\quad
    \int_K \|\nabla v(x)\|_2^2\,\nu(dx)<\infty,
  \]
  so that $I_S(\nu)=\frac14\int_K\|\nabla v\|_2^2\,d\nu$ is well-defined.

  \item[(ii)] $\mathcal L_A v$ belongs to the dual space $\mathcal H_S^{-1}(\nu)$, i.e.
  \[
    \mathcal L_A v \in \mathcal H_S^{-1}(\nu)
    \quad\text{if and only if}\quad
    \|\mathcal L_A v\|_{\mathcal H_S^{-1}(\nu)}<\infty,
  \]
  where $\|\cdot\|_{\mathcal H_S^{-1}(\nu)}$ is defined in \eqref{eq:dual_sobolev_seminorm}.
  In particular, $\mathcal L_A v$ is understood in the weak sense when $v$ is only in $\mathcal H_S^1(\nu)$.
\end{itemize}

If either (i) or (ii) fails (or if $\nu\not\ll\mu$), we set $I(\nu)=\infty$.
\end{remark}

\paragraph{Proof of Proposition \ref{prop:rate_decomposition}}
%\begin{proof}[Proof of Proposition \ref{prop:rate_decomposition}]
The proof is inspired from the framework from \cite[Theorem 3.3]{LDP-GG} for the unconstrained setting. From Proposition \ref{prop:dv}, the rate function is $I(\nu) = \sup_{u \in D^+(\mathcal{L}_J)} \left\{ -\int_K \frac{(\mathcal{L}_J u)}{u} d\nu \right\}.$
Assume $d\nu=e^v d\mu$, where $v=\log(d\nu/d\mu)$. Following \cite{LDP-GG}, we choose the test function $u=e^{\omega}$, where $\omega = v/2+\psi/2$, and $\psi$ is another test function (assumed to be in $D(\mathcal{L}_J)$ and sufficiently regular). Then, using the fact that the diffusion matrix is the identity,
\[
  -\frac{(\mathcal{L}_J u)(x)}{u(x)} = -\frac{(\mathcal{L}_J e^\omega)(x)}{e^\omega(x)} = -\left((\mathcal{L}_J\omega)(x)+\|\nabla \omega(x)\|_{2}^2\right).
\]
Substituting $\mathcal{L}_J = \mathcal{L}_S + \mathcal{L}_A$ and $\omega=v/2+\psi/2$, $I(\nu)$ is the supremum over $\psi$ of:
\begin{multline}\label{eq:sup_terms_expanded_new_corr} 
    -\int_K \left( \mathcal{L}_S\left(\frac{v+\psi}{2}\right) + \mathcal{L}_A\left(\frac{v+\psi}{2}\right) + \left\|\nabla \left(\frac{v+\psi}{2}\right)\right\|_{2}^2 \right) d\nu \\
    = -\frac12\int_K(\mathcal{L}_S v + \mathcal{L}_A v)d\nu -\frac12\int_K(\mathcal{L}_S\psi + \mathcal{L}_A\psi)d\nu \\
    -\frac14\int_K\|\nabla v\|_2^2d\nu -\frac12\int_K\nabla v\cdot\nabla \psi d\nu -\frac14\int_K\|\nabla \psi\|_2^2d\nu.
\end{multline}
We use the divergence theorem (integration by parts): for a vector field $\mathbf{F}$ and scalar function $g$,
$\int_K(\nabla \cdot \mathbf{F})g \,dx = -\int_K \mathbf{F}\cdot \nabla g \,dx + \int_{\partial K}g(\mathbf{F}\cdot \mathbf{n}) \,dS$,
where $\mathbf{n}$ is the outward unit normal vector.
To simplify terms involving $\mathcal{L}_S$: $\mathcal{L}_S u = e^f \nabla \cdot (e^{-f}\nabla u)$.
The term $-\frac12\int_K (\mathcal{L}_S\psi) d\nu$ becomes:
\begin{align*}
    -\frac12\int_K (\mathcal{L}_S\psi) e^v d\mu
    &= -\frac12\int_K \left( e^{f(x)}\nabla \cdot (e^{-f(x)}\nabla \psi(x)) \right) e^{v(x)} \frac{1}{Z}e^{-f(x)}dx \\
    &= -\frac{1}{2Z}\int_K \left( \nabla \cdot(e^{-f(x)}\nabla \psi(x)) \right) e^{v(x)}dx \\
    &= -\frac{1}{2Z} \Bigg( -\int_K \left(e^{-f(x)}\nabla \psi(x)\right)\cdot(e^{v(x)}\nabla v(x))dx 
    \\
    &\qquad\qquad\qquad\qquad\qquad+ \int_{\partial K}e^{v(x)}\left(e^{-f(x)}\nabla \psi(x) \cdot \mathbf{n}(x)\right)dS \Bigg) \\
    &= \frac{1}{2Z}\int_K e^{-f(x)}e^{v(x)}(\nabla \psi(x)\cdot\nabla v(x))dx \quad (\text{since } \nabla \psi \cdot \mathbf{n}=0 \text{ on } \partial K) \\
    &= \frac12\int_K(\nabla \psi \cdot \nabla v)d\nu.
\end{align*}
Similarly, $-\frac12\int_K (\mathcal{L}_S v)d\nu = \frac12\int_K \|\nabla v\|_2^2 d\nu$.

For the term $\int_K(\mathcal{L}_A v)d\nu$:
Since $\mathcal{L}_A 1 = - \inner{J \nabla f}{\nabla 1} = 0$, and $\mathcal{L}_A$ is anti-self-adjoint with respect to $\mu$ ($\mathcal{L}_A^* = -\mathcal{L}_A$), we obtain
$\int_K(\mathcal{L}_A v)d\nu = \int_K(\mathcal{L}_A v)e^v d\mu$.
Use the anti-self-adjointness of $\mathcal{L}_A$, we can show the above term is zero:
\[ \int_K(\mathcal{L}_A v)e^v d\mu = \int_K(\mathcal{L}_A e^v) \cdot 1 \,d\mu = \int_K e^v (\mathcal{L}_A^* 1) d\mu = -\int_K e^v (\mathcal{L}_A 1) d\mu = 0. \]
Substituting these simplifications into \eqref{eq:sup_terms_expanded_new_corr}, the expression becomes:
\begin{align*}
    \frac14\int_K\|\nabla v\|_2^2d\nu - \frac12\int_K\mathcal{L}_A\psi d\nu - \frac14\int_K\|\nabla \psi\|_2^2d\nu.
\end{align*}
Using the anti-self-adjointness of $\mathcal{L}_A$ with respect to $\mu$ and $\mathcal{L}_A(e^v)=e^v(\mathcal{L}_A v)$:
\begin{align*}
    -\frac12\int_K(\mathcal{L}_A\psi)d\nu = -\frac12\int_K(\mathcal{L}_A\psi)e^v d\mu 
    &= \frac12\int_K \psi (\mathcal{L}_A e^v) d\mu 
    \\
    &= \frac12\int_K \psi e^v (\mathcal{L}_A v) d\mu = \frac12\int_K(\mathcal{L}_A v)\psi d\nu.
\end{align*}
Thus, the expression to maximize over $\psi$ is:
\[
    \frac14\int_K\|\nabla v\|_2^2d\nu + \frac12\int_K(\mathcal{L}_A v)\psi d\nu - \frac14\int_K\|\nabla \psi\|_2^2d\nu.
\]
Therefore, the rate function $I(\nu)$ is
\begin{align*}
    I(\nu)
    &= \sup_{\psi\in D^+(\mathcal L_J)}\left\{ \frac14\int_K\|\nabla v\|_2^2d\nu + \frac12\int_K(\mathcal{L}_A v)\psi d\nu - \frac14\int_K\|\nabla \psi\|_2^2d\nu \right\} \\
    &= \frac14\int_K\|\nabla v\|_2^2 d\nu + \frac14\|\mathcal{L}_A v\|_{\mathcal{H}_S^{-1}(\nu)}^2 \\
    &= \frac14\|v\|_{\mathcal{H}^1_S(\nu)}^2 + \frac14\|\mathcal{L}_A v\|_{\mathcal{H}_S^{-1}(\nu)}^2
    =: I_S(\nu) + I_A(\nu).
\end{align*}
Finally, by the variational representation
\[
I_A(\nu)
=\sup_{\psi \in D^+(\mathcal{L}_J)}\left\{\frac12\int_K (\mathcal L_A v)\,\psi\,d\nu-\frac14\int_K\|\nabla\psi\|_2^2\,d\nu\right\},
\]
we immediately have $I_A(\nu)\ge 0$ by choosing $\psi\equiv 0$ (which yields value $0$ inside the supremum).
This completes the proof.
\hfill$\blacksquare$
%\end{proof}

Let us discuss the implications for acceleration and structural dependence on $J$. Proposition~\ref{prop:rate_decomposition} shows that the rate function $I(\nu)$ for the SRNLD is the sum of the rate function $I_S(\nu)$ for the RLD and a non-negative term $I_A(\nu)$ due to the skew-symmetric part of the dynamics. Thus, $I(\nu) \ge I_S(\nu)$. A larger rate function implies that the probability of observing an empirical measure $\nu \neq \mu$ decays faster as $t \to \infty$. This indicates that the SRNLD process concentrates more sharply around the target invariant measure $\mu$ than the RLD process, thereby providing a theoretical justification for the acceleration observed with non-reversible dynamics.

The term $I_A(\nu) = \frac{1}{4} \|\mathcal{L}_A v\|_{\mathcal{H}_S^{-1}(\nu)}^2$, where $\mathcal{L}_A v(x) = -\inner{J(x)\nabla f(x)}{\nabla v(x)}$ and $v = \log(d\nu/d\mu)$, quantifies this improvement. To understand the intricate structural dependence of $I_A(\nu)$ on $J(x)$ beyond its overall magnitude, we first define the operator $(-\Delta_{\nu})^{-1}$. This operator acts on functions $\varphi \in L^2(K, d\nu)$ that have zero mean with respect to the measure $\nu$ (i.e., $\int_K \varphi \, d\nu = 0$). For such a $\varphi$, the function $u = (-\Delta_{\nu})^{-1}\varphi$ is defined as the unique solution in $H^1(K,d\nu)$ that also has zero mean with respect to $\nu$ (i.e., $\int_K u \, d\nu = 0$) and satisfies the weak formulation of the Neumann problem:
\[
\int_K \nabla u(x) \cdot \nabla \chi(x) \, \nu(dx) = \int_K \varphi(x) \chi(x) \, \nu(dx), \quad \text{for all } \chi \in H^1(K,d\nu).
\]
As established in the derivation of Proposition \ref{prop:rate_decomposition} (specifically, that $\int_K (\mathcal{L}_A v)d\nu = 0$), the term $\mathcal{L}_A v$ meets the zero-mean condition required for $(-\Delta_{\nu})^{-1}$ to be well-defined.
The rate function component $I_A(\nu)$ can then be expressed using this operator as:
\[
I_A(\nu) = \frac{1}{4} \int_K (\mathcal{L}_A v)(x) \left( (-\Delta_{\nu})^{-1} (\mathcal{L}_A v) \right)(x) \, \nu(dx).
\]
If $G_\nu(x,y)$ denotes the Green's function (integral kernel) of the operator $(-\Delta_{\nu})^{-1}$, this can be written as a quadratic form:
\[
I_A(\nu) = \frac{1}{4} \iint_{K \times K} \left(-\inner{J(x)\nabla f(x)}{\nabla v(x)}\right) G_\nu(x,y) \left(-\inner{J(y)\nabla f(y)}{\nabla v(y)}\right) \, \nu(dx) \, \nu(dy).
\]

%%%%%%%%%%%%%%%%%%%%%%%%%%%
\subsection{Variance Reduction in Skew-Reflected Non-Reversible Langevin Dynamics}
In this section, we demonstrate that the non-reversible structure can also lead to a reduction in the asymptotic variance. 
In the unconstrained setting, it is well-documented that breaking reversibility by adding a suitable skew-symmetric drift can significantly improve sampling efficiency in terms of reducing the asymptotic variance (see e.g. \cite{Duncan2017,reyLDP}).
In this section, we will show that this phenomenon 
also holds in the constrained setting.

Let $g \in C^{2}(K)$ be an observable, and consider its time-average estimator 
\[
    \hat{g}_t = \frac{1}{t}\int_0^t g(X_s)d s
\]
for the true mean $\bar{g} = \int_K g d\mu$. The long-term statistical properties of this estimator are governed by a Central Limit Theorem (CLT). The validity of the CLT for the SRNLD process \eqref{eq:sde_local} is guaranteed by its strong ergodic properties, which we briefly outline.

The generator $\calL_J$ of the SRNLD process \eqref{eq:sde_local} is a uniformly elliptic operator on a compact, connected domain $K$ with $\partial K\in C^{2,\alpha}$ and reflecting (Neumann) boundary conditions. These conditions imply that the associated Markov semigroup admits a strictly positive and continuous heat kernel (see, e.g., \cite[Theorem.~3.3.5]{DaviesHeatKernels}), which ensures that the semigroup is strongly Feller and irreducible on $K$. Consequently, the process admits a unique invariant probability measure $\mu$ (Lemma~\ref{lem:invariant_measure_existence}) and is ergodic.

Moreover, the reversible part $\calL_S$ (the generator of the reflected Langevin dynamics) has a spectral gap on the mean-zero subspace $L^2_0(\mu)$: there exists $\lambda_P>0$ (the Poincar\'e constant) such that for all $u\in D(\calL_S)\cap L^2_0(\mu)$,
\begin{equation}\label{eq:poincare_gap}
    \left\langle u,(-\calL_S)u\right\rangle_\mu=\int_K \|\nabla u\|_2^2\,d\mu \ \ge\ \lambda_P \|u\|_{L^2(\mu)}^2 .
\end{equation}
Since Assumptions~\ref{assump:boundary} and~\ref{assump:nablaJ} ensure $\calL_A$ is anti-self-adjoint in $L^2(\mu)$, we have for all $u\in D(\calL_J)\cap L^2_0(\mu)$,
\begin{equation}\label{eq:coercive_realpart}
    \mathrm{Re}\left\langle u,(-\calL_J)u\right\rangle_\mu
    =\mathrm{Re}\left\langle u,(-\calL_S-\calL_A)u\right\rangle_\mu
    =\left\langle u,(-\calL_S)u\right\rangle_\mu
    \ge \lambda_P \|u\|_{L^2(\mu)}^2 .
\end{equation}
In particular, $0$ is a simple eigenvalue of $\calL_J$ (constants), and $-\calL_J$ is injective and coercive on $L^2_0(\mu)$; hence the Poisson equation
\begin{equation}\label{eq:poisson_wellposed}
    -\calL_J \phi = g, \qquad \int_K \phi\,d\mu=0,
\end{equation}
has a unique solution $\phi\in D(\calL_J)\cap L^2_0(\mu)$ for each $g\in L^2_0(\mu)$, and the inverse $(-\calL_J)^{-1}$ is a bounded operator on $L^2_0(\mu)$ with
\begin{equation}\label{eq:inverse_bound}
    \|(-\calL_J)^{-1}g\|_{L^2(\mu)} \ \le\ \lambda_P^{-1}\|g\|_{L^2(\mu)} .
\end{equation}
Therefore, as $t \to \infty$, the estimator $\hat g_t$ satisfies a CLT:
\begin{equation}
    \sqrt{t}\left(\hat g_t-\bar g\right)\Rightarrow \mathcal N(0,\sigma^2_{g,J}) .
\end{equation}
The term $\sigma_{g,J}^2$ is known as the \textbf{asymptotic variance}, which quantifies the magnitude of the estimator's statistical fluctuations around the mean. It is formally defined by the time-integrated auto-correlation function:
\begin{equation} \label{eq:var_def}
    \sigma_{g,J}^2 = 2 \int_0^\infty \E_\mu \left[ (g(X_0) - \bar{g})(g(X_t) - \bar{g}) \right] d t.
\end{equation}
%\textit{e) Section 2.6 and Appendix B.5: The arguments of these two sections rely on the existence of a spectral gap for the nonreversible generator $L_{J}$ and on the boundedness of its inverse $(-L_{J})^{-1}$. For instance, in the proof of Lemma 6 in Appendix B.5, the paper states that the existence and boundedness of $(-L_{J})^{-1}$ is ``ensured by the spectral gap of $L_J$ on the compact domain $K$'', and subsequently uses the formula \[ \sigma_{g,J}^{2}=2\langle g,(-L_{J})^{-1}g\rangle_{\mu}. \] While Appendix C provides the algebraic decomposition $L_J = L_S + L_A$, this does not guarantee a spectral gap or bounded inverse. It would be helpful if the authors can provide an argument justifying these details.}

Next lemma is important and its analogous version for the unconstrained setting can be found in \cite{DLP2016}. We conduct our analysis in the Hilbert space of mean-zero functions with respect to the invariant measure $\mu$, defined as:
\[
L^2_0(\mu) := \left\{ g \in L^2(K, \mu) : \int_K g(x) d\mu(x) = 0 \right\},
\]
equipped with the standard $L^2(K,d\mu)$ inner product $\inner{g}{h} = \int_K g(x)h(x)d\mu(x)$.

\begin{lemma}[Characterizations of Asymptotic Variance]
\label{lemma:asymptotic_variance}
Let $(X_t)_{t\ge0}$ be the ergodic process defined by Eq. \eqref{eq:sde_local} with generator $\calL_J$ and invariant measure $\mu$. For any mean-zero observable $g \in L^2_0(\mu)$, 
the asymptotic variance $\sigma_{g,J}^2$ in the Central Limit Theorem can be characterized by the following equivalent expressions:
\begin{enumerate}[label=(\alph*)]
    \item \textbf{The Green-Kubo Formula:}
    \begin{equation}
        \sigma_{g,J}^2 = 2 \int_0^\infty \E_\mu [g(X_0) g(X_t)] dt.
    \end{equation}
    
    \item \textbf{The Operator-Theoretic Formula:}
    \begin{equation} \label{eq:var_poisson}
        \sigma_{g,J}^2 = 2 \inner{g}{(-\calL_J)^{-1}g}_\mu.
    \end{equation}
\end{enumerate}
\end{lemma}

\begin{proof}
The proof will be provided in Appendix~\ref{app:proof_asymptotic_variance}.
\end{proof}

\begin{remark}[PDE Formulation for the Asymptotic Variance]
  The asymptotic variance, $\sigma_{g,J}^2$ can be characterized through the solution of a partial differential equation (PDE). The variance is given by the operator-theoretic formula:
\begin{equation}
    \sigma_{g,J}^2 = 2\langle g, (-\mathcal{L}_J)^{-1}g \rangle_\mu.
\end{equation}
To compute this quantity, one can first find an auxiliary function, $\psi(x)$, by solving for the action of the resolvent operator $(-\mathcal{L}_J)^{-1}$ on the observable $g(x)$:
\begin{equation}
    \psi(x) = (-\mathcal{L}_J)^{-1}g(x).
\end{equation}
This is equivalent to solving the following second-order, linear, elliptic partial differential equation for $\psi(x)$:
\begin{equation}
    -\mathcal{L}_J \psi(x) = g(x), \quad x \in K.
\end{equation}
By substituting the explicit form of the infinitesimal generator $\mathcal{L}_J$ as defined in the manuscript, we obtain the full PDE:
\begin{equation}
    -\left( \Delta \psi(x) - \langle(I+J(x))\nabla f(x), \nabla \psi(x)\rangle \right) = g(x).
\end{equation}
This equation is subject to the boundary conditions imposed by the domain of the operator, $D(\mathcal{L}_J)$. Under Assumption~\ref{assump:boundary} ($J(x)\mathbf{n}(x)=0$ on $\partial K$), this simplifies to the standard Neumann boundary condition:
\begin{equation}
    \langle \nabla \psi(x), \mathbf{n}(x) \rangle = 0, \quad \text{for every $x \in \partial K$},
\end{equation}
where $\mathbf{n}(x)$ is the outward unit normal vector on the boundary of the domain $K$.

Once the solution $\psi(x)$ to this boundary value problem is found, the asymptotic variance is obtained by computing the inner product with respect to the invariant measure $\mu$:
\begin{equation}
    \sigma_{g,J}^2 = 2 \int_K g(x) \psi(x) d\mu(x).
\end{equation}
\end{remark}

\subsubsection{Quantitative Variance Reduction via Spectral Analysis}
\label{subsec:quantitative_variance}

We now establish that the SRNLD strictly reduces the asymptotic variance compared to the RLD and quantify this reduction explicitly. While the general inequality $\sigma_{g,J}^2 \le \sigma_{g,0}^2$ can be inferred from abstract operator inequalities, a spectral analysis provides a precise characterization of the efficiency gain.

We work in the Hilbert space $\mathcal{H}^1(\mu)$, defined as the completion of $D(\mathcal{L}_J)$ with respect to the inner product $\langle u, v \rangle_1 := \langle u, (-\mathcal{L}_S) v \rangle_\mu = \int_K \nabla u \cdot \nabla v \, d\mu$. Let $\mathcal{G} := (-\mathcal{L}_S)^{-1}\mathcal{L}_A$ be the operator characterizing the non-reversible perturbation.

It is important to emphasize that the validity of the following spectral analysis rests fundamentally on Assumption \ref{assump:boundary} ($J(x)\mathbf{n}(x)=0$ on $\partial K$). This condition serves two critical purposes:
\begin{enumerate}
    \item \textbf{Domain Compatibility:} It ensures that the non-reversible generator $\mathcal{L}_J$ shares the same domain (functions satisfying Neumann boundary conditions) as the reversible generator $\mathcal{L}_S$. This allows for the well-defined operator decomposition $\mathcal{L}_J = \mathcal{L}_S + \mathcal{L}_A$.
    \item \textbf{Skew-Adjointness:} It eliminates the boundary terms arising from integration by parts (as detailed in Appendix \ref{app:operator_decomposition_simplified}). Consequently, $\mathcal{L}_A$ becomes anti-self-adjoint on $L^2(\mu)$, which implies that $\mathcal{G}$ is a compact, \emph{skew-adjoint} operator on $\mathcal{H}^1(\mu)$ (i.e., $\mathcal{G}^* = -\mathcal{G}$).
\end{enumerate}

Due to this skew-adjointness, the spectrum of $\mathcal{G}$ consists of purely imaginary eigenvalues $\{\pm i \lambda_n\}_{n \ge 1}$ with $\lambda_n > 0$, and a possible kernel. Let $\{e_n\}_{n \ge 1}$ be the corresponding orthonormal eigenfunctions in $\mathcal{H}^1(\mu)$. Let $\phi_g = (-\mathcal{L}_S)^{-1}g$ be the solution to the Poisson equation for the reversible dynamics, decomposed as $\phi_g = \phi_0 + \sum_{n=1}^\infty c_n e_n$, where $\phi_0 \in \ker(\mathcal{G})$ and $c_n = \langle \phi_g, e_n \rangle_1$.

\begin{theorem}[Quantitative Variance Reduction]\label{thm:variance_reduction}
Under Assumptions \ref{assump:regularity_local}, \ref{assump:boundary}, and \ref{assump:nablaJ}, let $\sigma_{g,0}^2$ and $\sigma_{g,J}^2$ be the asymptotic variances for the reversible RLD and non-reversible SRNLD, respectively. The asymptotic variance of SRNLD is explicitly given by:
\begin{equation} \label{eq:spectral_variance}
    \sigma_{g,J}^2 = 2 \|\phi_0\|_1^2 + 2 \sum_{n=1}^\infty \frac{|c_n|^2}{1 + \lambda_n^2}.
\end{equation}
This implies a strict asymptotic variance reduction relative to the reversible case:
\begin{equation} \label{eq:variance_diff}
    \sigma_{g,0}^2 - \sigma_{g,J}^2 = 2 \sum_{n=1}^\infty |c_n|^2 \frac{\lambda_n^2}{1 + \lambda_n^2} \ge 0.
\end{equation}
The inequality is strict ($\sigma_{g,J}^2 < \sigma_{g,0}^2$) unless the observable $g$ is uncorrelated with the non-reversible part of the dynamics (i.e., $c_n = 0$ for all $n$).
\end{theorem}

\begin{proof}
Recall from Lemma~\ref{lemma:asymptotic_variance} that $\sigma_{g,J}^2 = 2 \langle g, (-\mathcal{L}_J)^{-1} g \rangle_\mu$. Using the decomposition $\mathcal{L}_J = \mathcal{L}_S + \mathcal{L}_A$ and the definition of $\mathcal{G}$, we can factorize the resolvent as $(-\mathcal{L}_J)^{-1} = (I + \mathcal{G})^{-1} (-\mathcal{L}_S)^{-1}$. Substituting $g = (-\mathcal{L}_S)\phi_g$, we obtain:
\[
    \sigma_{g,J}^2 = 2 \left\langle (-\mathcal{L}_S)\phi_g, (I + \mathcal{G})^{-1} \phi_g \right\rangle_\mu = 2 \left\langle \phi_g, (I + \mathcal{G})^{-1} \phi_g \right\rangle_1.
\]
Since the asymptotic variance is a real scalar, it is determined by the symmetric part of the resolvent operator in $\mathcal{H}^1(\mu)$. Crucially, because Assumption \ref{assump:boundary} guarantees $\mathcal{G}^* = -\mathcal{G}$, the symmetric part of $(I+\mathcal{G})^{-1}$ is simply $(I-\mathcal{G}^2)^{-1}$. Applying the spectral decomposition of $\phi_g$ to the operator $(I - \mathcal{G}^2)^{-1}$, where $\mathcal{G} e_n = \mathrm i \lambda_n e_n$, yields:
\[
    \left\langle \phi_g, (I - \mathcal{G}^2)^{-1} \phi_g \right\rangle_1 = \|\phi_0\|_1^2 + \sum_{n=1}^\infty \frac{1}{1 - (\mathrm i\lambda_n)^2} |c_n|^2 = \|\phi_0\|_1^2 + \sum_{n=1}^\infty \frac{|c_n|^2}{1 + \lambda_n^2}.
\]
Multiplying by 2 gives Eq.~\eqref{eq:spectral_variance}. For the reversible case ($J\equiv 0$), we have $\mathcal{G}=0$ (implying $\lambda_n=0$), which yields $\sigma_{g,0}^2 = 2 \|\phi_g\|_1^2 = 2 \|\phi_0\|_1^2 + 2 \sum_{n=1}^{\infty} |c_n|^2$. Subtracting the two expressions directly yields Eq.~\eqref{eq:variance_diff}. Since $\lambda_n^2 > 0$, the difference is non-negative, proving the reduction.
The proof is complete.
\end{proof}

%%%%%%%%%%%%%%%%%%

% =========================================================================
%\vspace{-0.1in}
\subsubsection{Connection to Large Deviations Theory}
% =========================================================================

We have established a Large Deviation Principle for the empirical measure $L_T = \frac{1}{T}\int_0^T \delta_{X_t} dt$ with a good rate function $I(\nu)$ (Theorem~\ref{thm:ldp_empirical_measures}). The decomposition $I(\nu)=I_S(\nu)+I_A(\nu)$ immediately suggests that convergence is accelerated, as $I(\nu) \ge I_S(\nu)$. We now show how this LDP structure directly translates to the quantitative variance reduction results obtained in the previous section.

A fundamental result from LDP theory connects the asymptotic variance to the curvature of the rate function for a specific observable. Let 
\begin{align*}
{I}_{g,J}(l) := \inf_{\nu \in \calP(K) : \int_K g d\nu = l} I(\nu)
\end{align*}
be the rate function for large deviations of the time-average $\hat{g}_t$, obtained via the Contraction Principle (see, e.g., \cite[Theorem 4.2.1]{DZ1998}).

\begin{theorem}[LDP Formulations for Asymptotic Variance]
\label{thm:equiv_formulations}
Let $\sigma_{g,J}^2$ be the asymptotic variance of the time-average estimator for a mean-zero observable $g \in L^2_0(\mu)$ under the SRNLD process. The following two characterizations of the variance are equivalent:
\begin{enumerate}[label=(\alph*)]
    \item The operator-theoretic formula:
    \begin{equation} \label{eq:var_op_form}
        \sigma_{g,J}^2 = 2 \inner{g}{(-\calL_J)^{-1}g}_\mu.
    \end{equation}
    \item The LDP formulation via the curvature of the rate function at the mean ($\bar{g}=0$):
    \begin{equation} \label{eq:var_ldp_form}
        \sigma_{g,J}^2 = \frac{1}{I_{g,J}''(0)}.
    \end{equation}
\end{enumerate}
\end{theorem}

\begin{proof}
The proof proceeds by demonstrating that both formulations are equivalent to a third quantity: the second derivative of the Scaled Cumulant Generating Function (SCGF) at the origin.

The connection between the rate function $I_{g,J}(l)$ and the SCGF is established by the G\"artner-Ellis theorem, which provides a variational representation for the rate function. The theorem states that under suitable conditions (finiteness and differentiability of the SCGF, and exponential tightness, which are satisfied on our compact domain $K$), the rate function is the Legendre-Fenchel transform of the SCGF (see, e.g., \cite[Theorem 2.3.6]{DZ1998}):
\begin{equation}
    I_{g,J}(l) = \sup_{\beta \in \R} \{ \beta l - \lambda_J(\beta) \}.
\end{equation}
The SCGF $\lambda_J(\beta)$ itself is defined via the long-time limit of moment generating functions:
\begin{equation}
    \lambda_J(\beta) = \lim_{T \to \infty} \frac{1}{T} \log \E \left[ \exp\left( \beta \int_0^T g(X_t) dt \right) \right].
\end{equation}
By the Feynman-Kac formula, this limit corresponds to the principal eigenvalue of the perturbed operator $\calL_\beta := \calL_J + \beta g$. The existence and properties of this eigenvalue for elliptic operators on compact domains are well-established results in spectral theory, see for example \cite[Section~6.5, Theorem~3]{evans2022partial}.

A fundamental property of the Legendre transform is that the second derivatives of a function and its dual are reciprocals at corresponding points. We are interested in the point $l = \bar{g} = 0$. The corresponding dual point is $\beta=0$, since the derivative of the SCGF at zero is the expectation of the observable:
\begin{equation*}
    \lambda_J'(0) = \lim_{T\to\infty} \E\left[ \frac{1}{T}\int_0^T g(X_t)dt \right] = \E_\mu[g] = \bar{g} = 0.
\end{equation*}
A key property of the Legendre-Fenchel transform is the duality of derivatives. Let $\beta(l)$ be the value of $\beta$ that achieves the supremum in the definition of $I_{g,J}(l)$. The first-order optimality conditions imply the dual relations: $I_{g,J}'(l) = \beta(l) \quad \text{and} \quad l = \lambda_J'(\beta(l)).$
Differentiating the second identity, $l = \lambda_J'(\beta(l))$, with respect to $l$ using the chain rule gives $1 = \lambda_J''(\beta(l)) \cdot \frac{d\beta}{dl}$. Since $\frac{d\beta}{dl} = I_{g,J}''(l)$ from the first identity, we establish the reciprocal relationship between the second derivatives, $I_{g,J}''(l) \cdot \lambda_J''(\beta(l)) = 1.$
We are interested in the curvature at the mean value, $l = \bar{g} = 0$. The corresponding dual point is $\beta(0)=0$, since $\lambda_J'(0) = \E_\mu[g] = 0$. Evaluating at this point yields the desired relation: $I_{g,J}''(0) = \frac{1}{\lambda_J''(0)}.$
This shows that the LDP formulation of variance, Eq. \eqref{eq:var_ldp_form}, is equivalent to the statement $\sigma_{g,J}^2 = \lambda_J''(0)$.
Now, we show that the operator-theoretic formulation, Eq. \eqref{eq:var_op_form}, is also equivalent to $\sigma_{g,J}^2 = \lambda_J''(0)$.
The quantity $\lambda_J''(0)$ is the second derivative of the principal eigenvalue of the perturbed operator $\calL_\beta = \calL_J + \beta g$ at $\beta=0$. This can be computed using perturbation theory for linear operators. Let $\lambda(\beta)$, $u_\beta$, and $v_\beta$ be the principal eigenvalue, right eigenfunction, and left eigenfunction of $\calL_\beta$, respectively, with normalizations $\inner{u_\beta}{v_\beta}_\mu=1$ and $\inner{1}{v_\beta}_\mu=1$.

The first derivative is given by $\lambda'(0) = \inner{g u_0}{v_0}_\mu = \inner{g}{1}_\mu = \bar{g} = 0$, where $u_0=v_0=1$.
The second derivative formula, derived from the perturbation series (see, e.g., \cite[Chapter II, Equation~2.33]{kato2013perturbation}), is given by:
\begin{equation} \label{eq:perturbation_result}
    \lambda_J''(0) = \inner{g u_0'}{v_0}_\mu + \inner{g u_0}{v_0'}_\mu,
\end{equation}
where $u_0'$ and $v_0'$ are the first-order corrections to the right and left eigenfunctions, respectively. 
Solving the first-order equations yields $u_0' = (-\calL_J)^{-1}g$ and $v_0' = (-\calL_J^*)^{-1}g$. 
Due to the duality in $L^2(\mu)$, the two terms in \eqref{eq:perturbation_result} are equal:
\[
    \inner{g u_0}{v_0'}_\mu = \inner{g}{(-\calL_J^*)^{-1}g}_\mu = \inner{(-\calL_J)^{-1}g}{g}_\mu = \inner{u_0'}{g u_0}_\mu.
\]
Thus, the expression simplifies to the dominant term:
\begin{equation*}
    \lambda_J''(0) = 2 \inner{g u_0'}{v_0}_\mu = 2 \inner{g (-\calL_J)^{-1} g}{1}_\mu = 2\inner{g}{(-\calL_J)^{-1}g}_\mu.
\end{equation*}

This formula thus arises fundamentally from solving for the first-order perturbation of the eigenfunction, which necessitates inverting the generator $\calL_J$ on the space of mean-zero functions. For a full treatment of perturbation theory, \cite[Chapter II]{kato2013perturbation} is the standard reference.

By comparing \eqref{eq:perturbation_result} with the operator-theoretic variance formula \eqref{eq:var_op_form}, we immediately see that, $\sigma_{g,J}^2 = \lambda_J''(0).$
Since we have $\sigma_{g,J}^2 = 1 / I_{g,J}''(0) \iff \sigma_{g,J}^2 = \lambda_J''(0)$ and $\sigma_{g,J}^2 = 2\inner{g}{(-\calL_J)^{-1}g}_\mu \iff \sigma_{g,J}^2 = \lambda_J''(0)$, this establishes the theorem.
\end{proof}

\paragraph{Derivation of Quantitative Variance Reduction via LDP.}
We can now utilize the equivalence $\sigma_{g,J}^2 = 1/I_{g,J}''(0)$ to recover the quantitative spectral formula \eqref{eq:spectral_variance} directly from the local geometry of the rate function $I(\nu)$. This confirms that the acceleration phenomenon is encoded in the ``steepness'' of the LDP rate function.

Consider the tangent space of $\mathcal{P}(K)$ at $\mu$, which can be identified with the space of mean-zero functions $L^2_0(\mu)$. For any direction $v \in L^2_0(\mu)$, let $(\nu_\epsilon)_{\epsilon \ge 0}$ be a smooth path of probability measures starting at $\mu$ with initial velocity $v$ (e.g., defined by the density $d\nu_\epsilon \propto e^{\epsilon v} d\mu$). We define the second variation of the rate function along this direction as:
\[
    \mathcal{Q}(v) := \lim_{\epsilon \to 0} \frac{1}{\epsilon^2} I(\nu_\epsilon).
\]
For the symmetric part, a direct calculation using the definition of $I_S$ yields the Dirichlet energy:
\[
    \mathcal{Q}_S(v) = \frac{1}{4} \int_K \|\nabla v\|_2^2 d\mu = \frac{1}{4} \|v\|_1^2.
\]
For the skew-symmetric part, using the operator $\mathcal{G} = (-\mathcal{L}_S)^{-1}\mathcal{L}_A$ and the identity $\mathcal{L}_A = -\mathcal{L}_S \mathcal{G}$, the second variation corresponds to the squared dual norm in $\mathcal{H}^1(\mu)$:
\[
    \mathcal{Q}_A(v) = \frac{1}{4} \|\mathcal{L}_A v\|_{\mathcal{H}_S^{-1}(\mu)}^2 = \frac{1}{4} \langle \mathcal{L}_A v, (-\mathcal{L}_S)^{-1} \mathcal{L}_A v \rangle_\mu = \frac{1}{4} \|\mathcal{G} v\|_1^2.
\]
Thus, the total second variation is strictly equal to the following quadratic form:
\[
    \mathcal{Q}(v) = \frac{1}{4} \left( \|v\|_1^2 + \|\mathcal{G} v\|_1^2 \right) = \frac{1}{4} \left\langle v, (I + \mathcal{G}^*\mathcal{G}) v \right\rangle_1 = \frac{1}{4} \left\langle v, (I - \mathcal{G}^2) v \right\rangle_1,
\]
where we used the skew-adjointness $\mathcal{G}^* = -\mathcal{G}$ in $\mathcal{H}^1(\mu)$.

To find the variance, we compute the curvature of the observable's rate function. This amounts to minimizing the quadratic form $\mathcal{Q}(v)$ over all directions $v$ compatible with a small deviation in the observable mean, i.e., subject to the linearized constraint $\langle \phi_g, v \rangle_1 = l$:
\[
    \frac{1}{2} I_{g,J}''(0) l^2 = \inf_{v \in L^2_0(\mu)} \left\{ \frac{1}{4} \langle v, \left(I - \mathcal{G}^2\right) v \rangle_1 \;\middle|\; \langle \phi_g, v \rangle_1 = l \right\}.
\]
The exact solution to this quadratic programming problem is $v^* = \lambda (I - \mathcal{G}^2)^{-1} \phi_g$ for a Lagrange multiplier $\lambda$. Substituting this back yields the minimum value:
\[
    \frac{1}{2} I_{g,J}''(0) l^2 = \frac{l^2}{4 \langle \phi_g, (I - \mathcal{G}^2)^{-1} \phi_g \rangle_1}.
\]
Using the relation $\sigma_{g,J}^2 = 1/I_{g,J}''(0)$, we immediately recover the operator-theoretic variance formula:
\[
    \sigma_{g,J}^2 = 2 \langle \phi_g, \left(I - \mathcal{G}^2\right)^{-1} \phi_g \rangle_1.
\]
Applying the spectral decomposition of the skew-adjoint operator $\mathcal{G}$ (with eigenvalues $\mathrm i\lambda_n$) to the term $(I - \mathcal{G}^2)^{-1}$ yields the quantitative series expansion derived in Theorem \ref{thm:variance_reduction}:
\[
    \sigma_{g,J}^2 = 2 \|\phi_0\|_1^2 + 2 \sum_{n=1}^\infty \frac{|c_n|^2}{1 + \lambda_n^2}.
\]
This derivation demonstrates that the variance reduction ($\sigma_{g,J}^2 < \sigma_{g,0}^2$) is a direct consequence of the additional positive term $\|\mathcal{G}v\|_1^2$ in the LDP rate function, which increases the curvature $I_{g,J}''(0)$ and thus decreases the variance.

\begin{remark}[Connection to the Cram\'{e}r-Rao Lower Bound]
Note that the identity $\sigma_{g,J}^2 = 1/I_{g,J}''(0)$ in Eq.~\eqref{eq:var_ldp_form} identifies the asymptotic variance with the inverse curvature of the rate function. In the theory of large deviations for estimators, \cite{bahadur1967rates} established that this quantity corresponds to the ``asymptotic effective variance,'' which coincides with the Cram\'{e}r-Rao lower bound (the inverse Fisher information) for efficient estimators. Thus, $I_{g,J}''(0)$ can be interpreted as the dynamical Fisher information. This suggests that an optimal design of $J$ should aim to maximize the curvature $I_{g,J}''(0)$ to minimize the variance bound.
\end{remark}

%%%%%%%%%%%%%%%%%%%%%%%%%%%%%%%%%%%%%%%%%%%%
%\vspace{-0.1in}
\section{Numerical Experiments}\label{sec:numerical_experiments}

We present numerical results to validate our theoretical findings, specifically to demonstrate the impact of choosing different skew-symmetric matrices $J(x)$ on the sampling efficiency of \textit{skew-reflected non-reversible Langevin Monte Carlo} (SRNLMC). We compare the performance of standard \textit{projected Langevin Monte Carlo} (PLMC), corresponding to $J(x) \equiv 0$, against SRNLMC employing various forms of $J(x)$. These forms include a constant skew-symmetric matrix and state-dependent skew-symmetric matrices designed such that $J(x)\mathbf n(x) = 0$ on the boundary of the unit ball (where $\mathbf n(x)$ is the normal vector). Experiments are conducted on a toy problem of sampling a truncated Gaussian, a constrained Bayesian linear regression task, and constrained Bayesian logistic regressions using both synthetic and real data sets. In this work, to make fair comparison, we adopt similar experimental setup as in \cite{DFTWZ2025}. Our goal is to demonstrate that by using an appropriately chosen state-dependent skew-symmetric matrix $J_s(x)$, SRNLMC can achieve better acceleration performance than simply using a constant skew-symmetric matrix $J_a$.  In particular, consider the constraint set to be a $3$-dimensional centered $\ell^2$ unit ball such that $K = \{x \in \mathbb{R}^3: \Vert x \Vert_2^2 \leq r\}$, and smoothed $\ell_p$ ball constraint such that the sublevel domain is $K  = \{x \in \mathbb{R}^3: g(x) \leq \lambda\}$, 
with 
$g(x)
:= \sum_{i=1}^{d} \left(x_i^2+\varepsilon^2\right)^{\frac{p}{2}}$
as given in \eqref{eq:smoothlp_Rd:f}, 
we will conduct a toy example of truncated multivariate normal distribution, constrained Bayesian linear regression and constrained Bayesian logistic regression using synthetic data. And then we will provide constrained Bayesian logistic regression using real data as an example in high-dimension. In the $3$-dimensional examples, we choose the following constant skew-symmetric matrix $J_a$ as a nonstate-dependent skew-symmetric matrix. 
\begin{equation}
\label{eqn:skew:matrix}
J_a = \begin{pmatrix}
0 & a & 0 \\
-a & 0 & a \\
0 & -a & 0
\end{pmatrix},
\end{equation}
where $a \neq 0$ is a self-chosen constant parameter. In general, we will consider two state-dependent skew-symmetric matrices. For the $\ell^2$ unit ball constraint set, the state-dependent skew-symmetric matrix $J_s(x)$ is defined as the following,
\begin{equation}
\label{eqn:sym:matrix}
    J_s(x) = \begin{pmatrix}
    0 & -sx_3 & sx_2 \\
    sx_3 & 0 & -sx_1 \\
    -sx_2 & sx_1 & 0
\end{pmatrix},
\end{equation}
 where $s \neq 0$ is a self-chosen constant parameter. For the smoothed $\ell_p$ ball constraint with the sublevel-set $ \{x \in \mathbb{R}^d: g(x)\leq \lambda\}$, with 
$g(x)
:= \sum_{i=1}^{d} \left(x_i^2+\varepsilon^2\right)^{\frac{p}{2}}$
as given in \eqref{eq:smoothlp_Rd:f}, the skew-symmetric matrix $J_g(x)$ is defined as:
\begin{equation}
\label{eq:smooth:matrix}
J_g(x)=\left(\begin{array}{ccc}
0 & -k_3^{\left(a_{\ell}, b_{\ell}, c_{\ell}\right)}(x) & k_2^{\left(a_{\ell}, b_{\ell}, c_{\ell}\right)}(x) \\
k_3^{\left(a_{\ell}, b_{\ell}, c_{\ell}\right)}(x) & 0 & -k_1^{\left(a_{\ell}, b_{\ell}, c_{\ell}\right)}(x) \\
-k_2^{\left(a_{\ell}, b_{\ell}, c_{\ell}\right)}(x) & k_1^{\left(a_{\ell}, b_{\ell}, c_{\ell}\right)}(x) & 0
\end{array}\right),
\end{equation}
with the coefficients given in~\eqref{eq:smoothlp_Rd:k_explicit}.

 In Section~\ref{subsec:toy:example}, we will first compare the \emph{iteration complexity} of SRNLMC ($J_a$) using constant skew-symmetric matrix $J_a$ to the one of SRNLMC ($J_s(x)$),  the one of SRNLMC ($J_g(x)$) and the one of \textit{projected Langevin Monte Carlo} (PLMC) in $\mathcal{W}_1$ distance by using synthetic data in a toy example of the truncated standard multivariate normal distribution. By introducing \emph{stochastic gradients}, we will propose \emph{skew-reflected non-reversible stochastic gradient Langevin dynamics} (SRNSGLD) and \emph{projected stochastic gradient Langevin dynamics} (PSGLD). 
 In Section~\ref{subsec:Bayesian:linear:example}, we will compare SRNSGLD to PSGLD in terms of \emph{mean squared error} by using synthetic data in the example of constrained Bayesian linear regression. Furthermore, we will consider the example of constrained Bayesian logistic regression and we compare the \emph{accuracy}, which is defined as the ratio of the correctly predicted labels over the whole dataset in deep learning experiments, of the algorithms by using either synthetic or real data in Section~\ref{subsec:Bayesian:log:example}.

%%%%%%%%%%%%%%%%%%%%%%%%

\subsection{Toy Example: Truncated Multivariate Normal Distribution}
\label{subsec:toy:example}
\begin{equation}
\mu(x) = \frac{1}{Z}\exp\left\{-\frac{x^{\top }\Sigma^{-1}x}{2}\right\} \mathbf{1}_{K},\quad \Sigma = \begin{pmatrix}
0.25 & 0 & 0 \\
0 & 1 & 0 \\
0 & 0 & 4
\end{pmatrix},
\end{equation}
where $\mathbf{1}_{K}$ denotes the indicator function on $K$ and $Z:=\int_{K}e^{-\frac{x^{\top }\Sigma^{-1}x}{2}}dx$
is the normalizing constant. 

In our first set of experiments, we take the convex constraint sets to be ball $K$ in $\mathbb{R}^3$: 
\begin{equation}
 K = \{x \in \mathbb{R}^3:\|x\|^2_2 \leq 1\},
\end{equation}
we choose the skew-symmetric matrix $J_a$ with $a = 1$ in~\eqref{eqn:skew:matrix} and the state-dependent skew-symmetric matrix $J_s$ with $s = 5$ in~\eqref{eqn:sym:matrix}. 

In our second set of experiments, we consider the sublevel-set constraint: 
\begin{equation}
\label{eq:K:smoothed}
 K = \{x \in \mathbb{R}^3: g(x) \leq \lambda\}, \qquad g(x):= \sum_{i=1}^3\left(x_i^2 + \varepsilon^2\right)^{p/2}.
\end{equation}
that follows from the definition in \eqref{eq:smoothlp_Rd:f}. 
 To allow fine-tuning, we incorporate a scaling parameter to the skew-symmetric matrix in \eqref{eq:smooth:matrix} to obtain:
\begin{equation}
\label{eq:smooth:matrix:scale}
J_g(x)=\left(\begin{array}{ccc}
0 & -sk_3^{\left(a_{\ell}, b_{\ell}, c_{\ell}\right)}(x) & sk_2^{\left(a_{\ell}, b_{\ell}, c_{\ell}\right)}(x) \\
sk_3^{\left(a_{\ell}, b_{\ell}, c_{\ell}\right)}(x) & 0 & -sk_1^{\left(a_{\ell}, b_{\ell}, c_{\ell}\right)}(x) \\
-sk_2^{\left(a_{\ell}, b_{\ell}, c_{\ell}\right)}(x) & sk_1^{\left(a_{\ell}, b_{\ell}, c_{\ell}\right)}(x) & 0
\end{array}\right), \qquad s \neq 0.
\end{equation}
In this toy example, we will take $p = 4$, $\varepsilon = 0.2$ and $\lambda = 1.0$ in~\eqref{eq:K:smoothed} and take the scaling parameter $s = 8$ in~\eqref{eq:smooth:matrix:scale}.

In the proceeding, we simulate $5000$ samples and take $2000$ iterates starting from the initial point $x_0 = [0.2, 0.3, 0.5]^\top$ with step size $\eta=5\times 10^{-4}$, and we first derive a skew-projection formula explicitly by computing the skew unit normal vector direction on the convex set $K$, and then we approximate the Gibbs distribution by doing rejecting and accepting sampling, i.e. we take \textit{i.i.d.} samples from the 3-dimensional standard normal distribution and discard the samples that are out of the constraint set until obtaining the required number of sample points. Finally, we compare $\mathcal{W}_1$ distance between the target distribution and the sampling stationary distributions driven by SRNLMC ($J_a$), SRNLMC ($J_s(x)$), SRNLMC ($J_g(x)$) and PLMC in each dimension. 

In Figure~\ref{visualized_ball}, we show the first $2$ dimensions of the target distribution and the ones of the sampling stationary distributions by SRNLMC ($J_a$), SRNLMC ($J_s(x)$) and PLMC sequentially.  In the same way, we visulize the target distribution and sampling distribution in first $2$ dimensions by SRNLMC ($J_a$), SRNLMC ($J_g(x)$) and PLMC sequentially in Figure~\ref{visualized_sub}.  These results show that SRNLMC ($J_a$), SRNLMC ($J_s(x)$), SRNLMC ($J_g(x)$) and PLMC can successfully converge to the truncated standard normal distribution under our setting. 

\begin{figure}[htbp]
    \centering
    \begin{subfigure}[c]{0.3\textwidth}
        \centering
        \includegraphics[width=\linewidth]{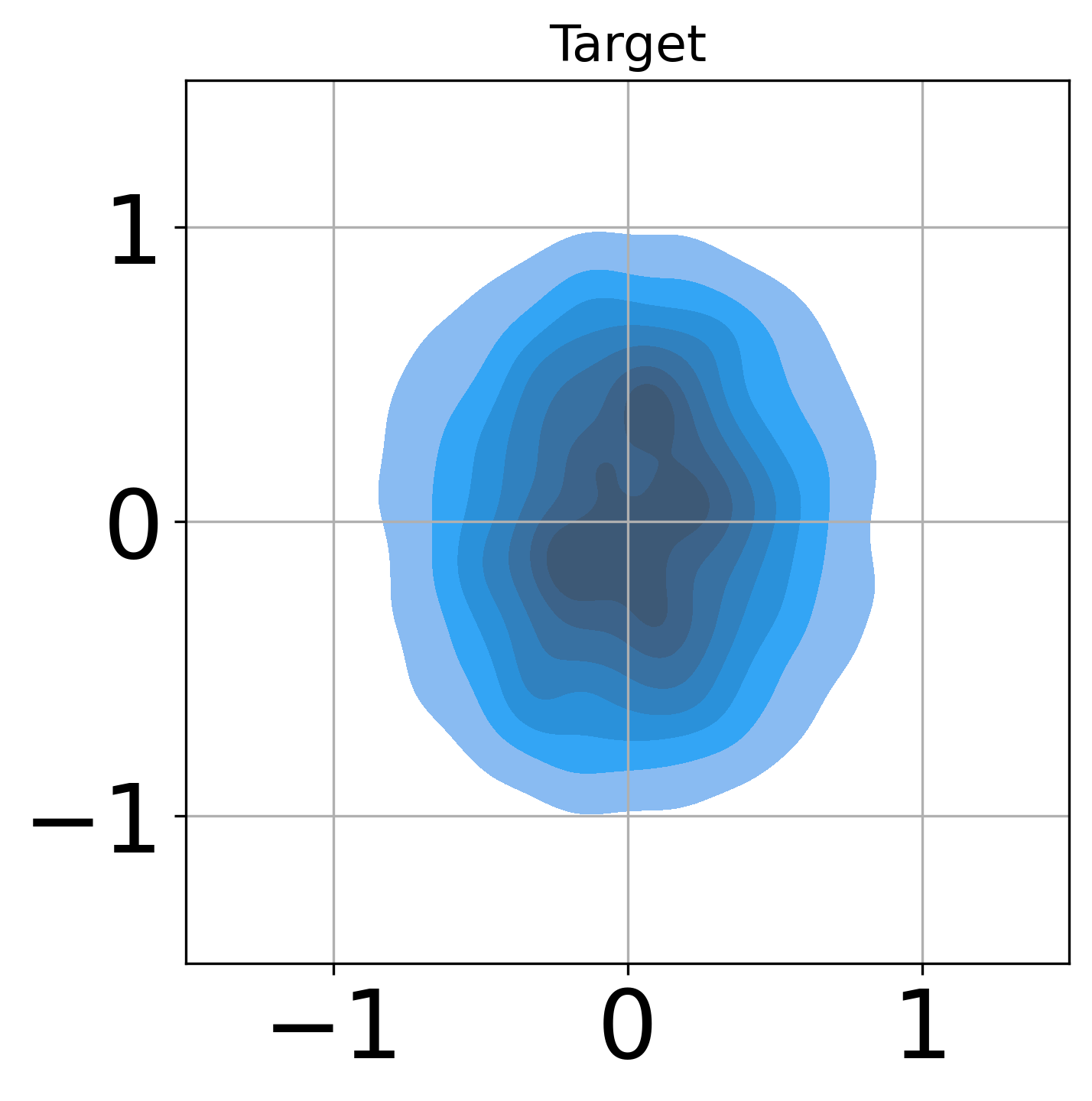}
        \caption{Target distribution}
    \end{subfigure}%
    \hspace{0.1\textwidth}
    \begin{subfigure}[c]{0.3\textwidth}
        \centering
        \includegraphics[width=\linewidth]{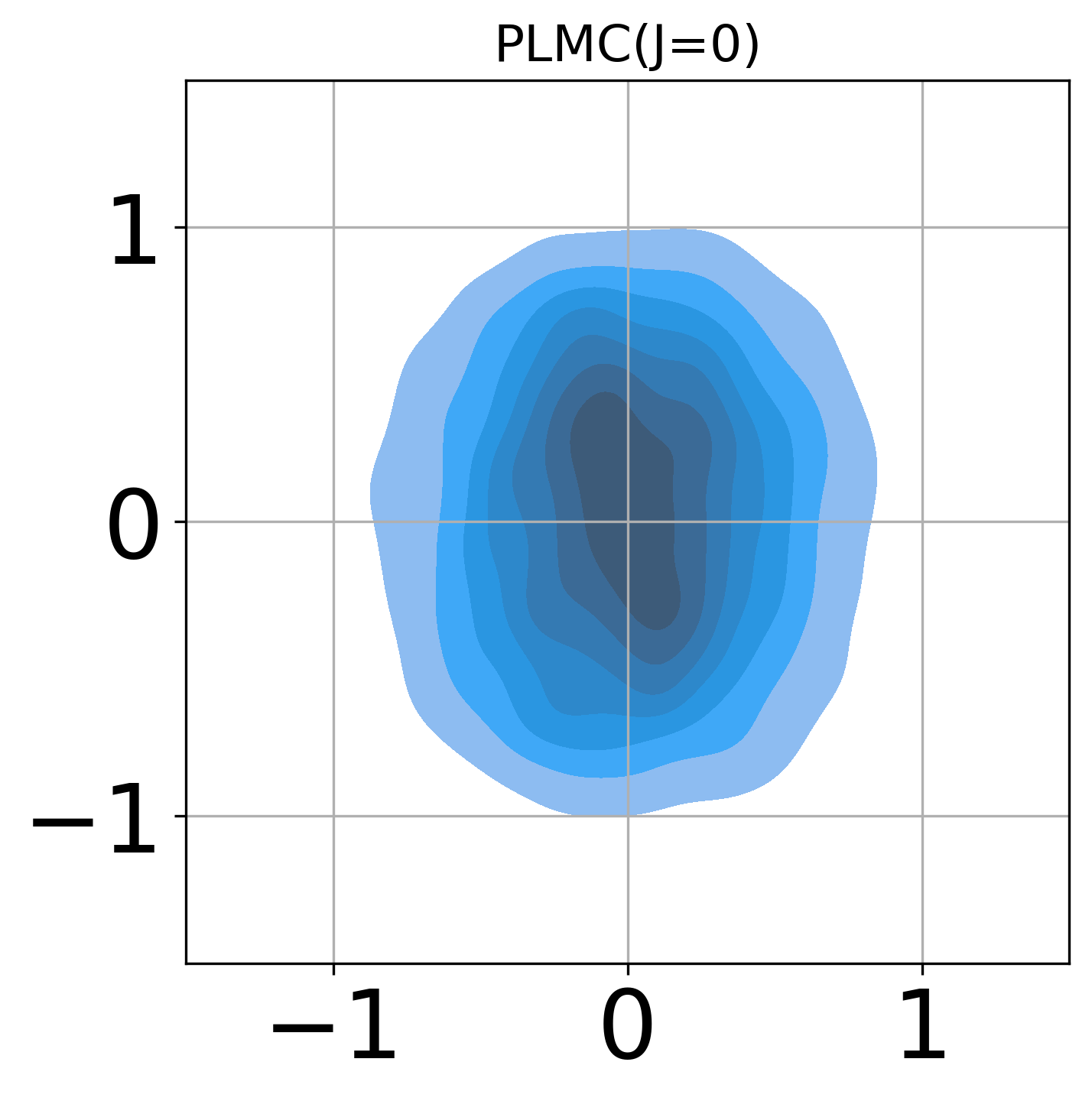}
        \caption{PLMC ($J = 0$)}
    \end{subfigure}
    
    \vspace{0.8cm}
    
    \begin{subfigure}[c]{0.3\textwidth}
        \centering
        \includegraphics[width=\linewidth]{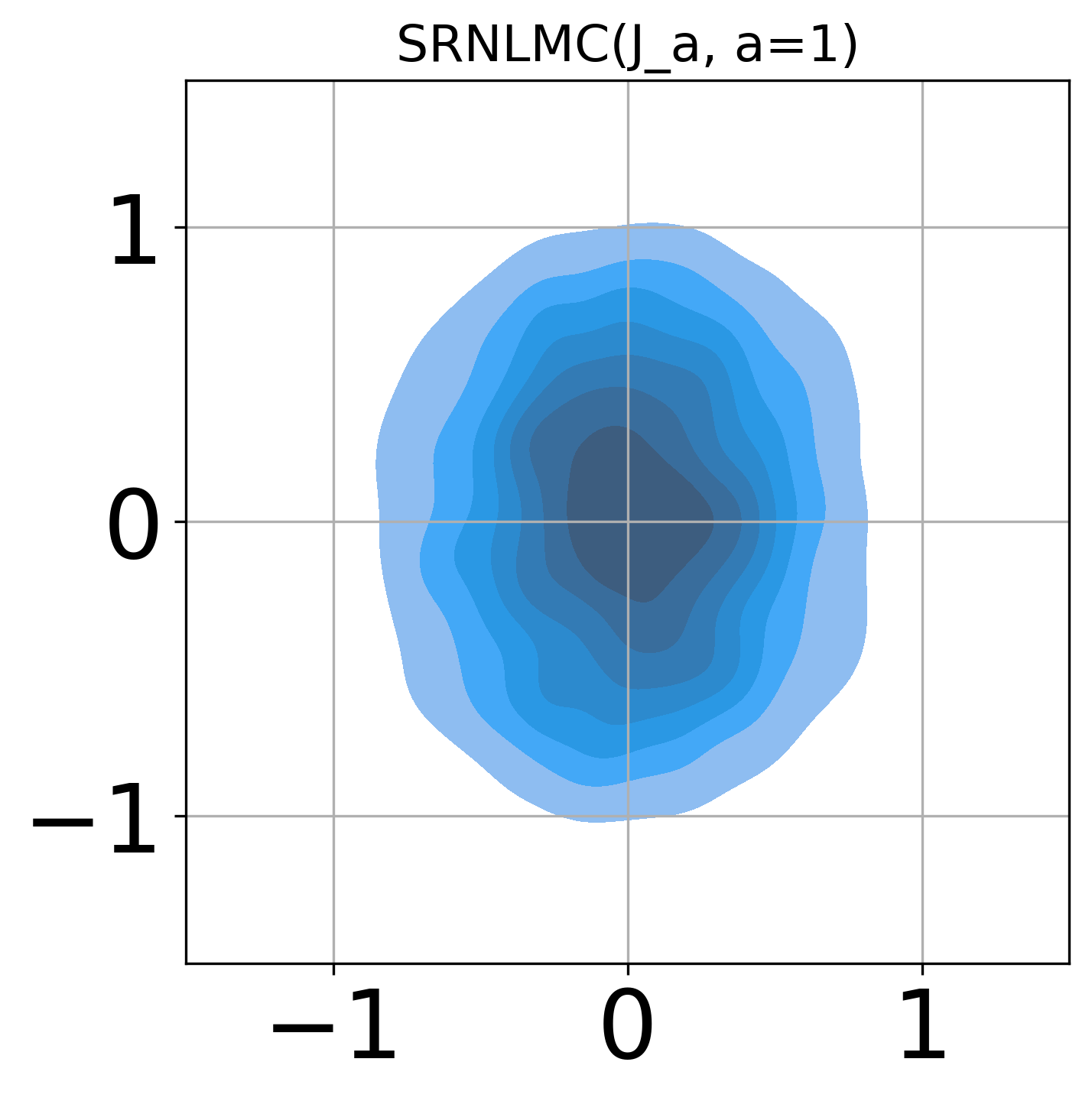}
\caption{SRNLMC ($J_{a = 1}$)}
    \end{subfigure}%
    \hspace{0.1\textwidth}    
    \begin{subfigure}[c]{0.3\textwidth}
        \centering
        \includegraphics[width=\linewidth]{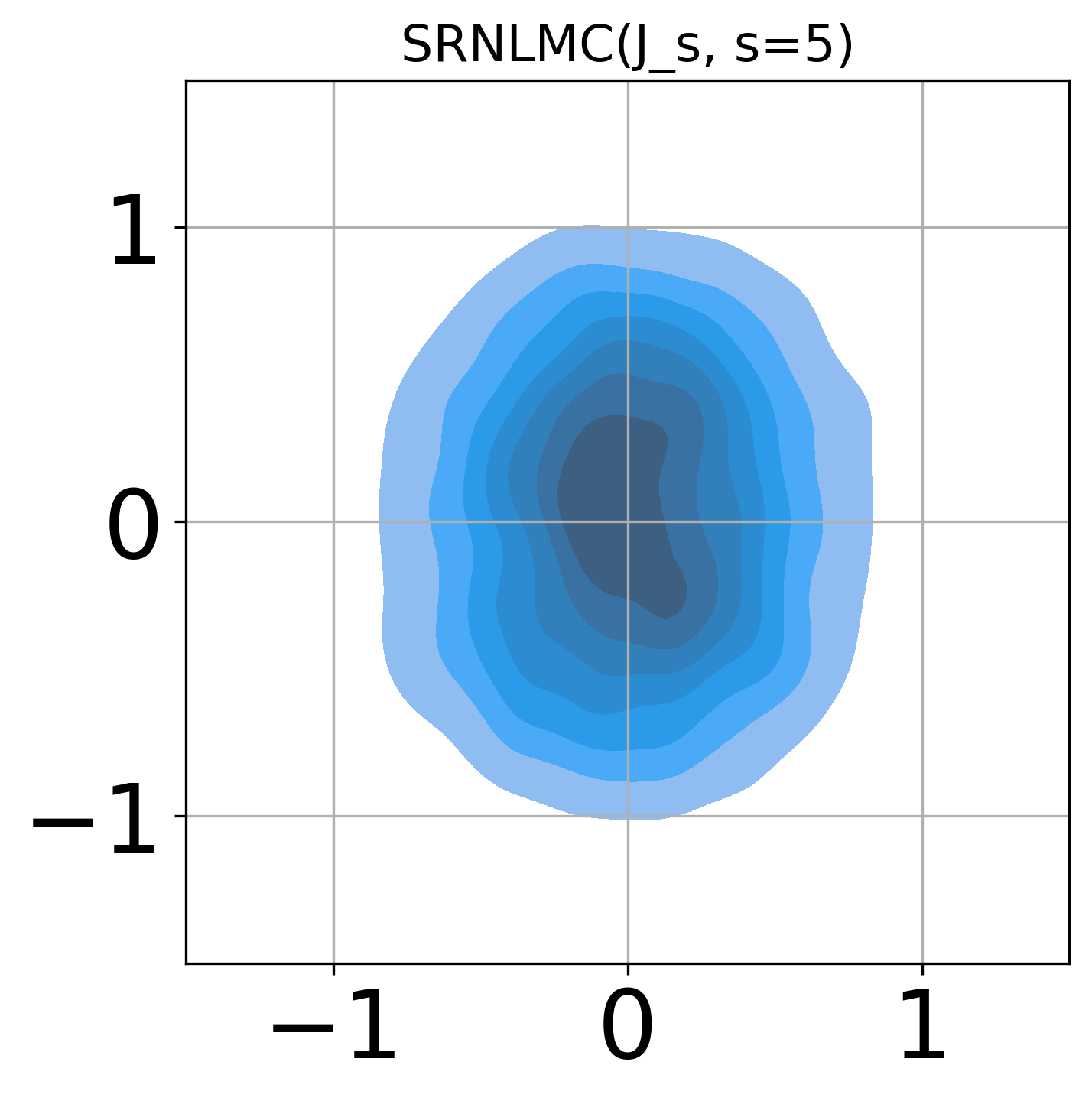}
        \caption{SRNLMC ($J_{s = 5}$)}
    \end{subfigure}
    
    \caption{Visualized density plots for the first 2 dimensions with a centered ball constraint.}
    \label{visualized_ball}
\end{figure}

Then we compare the convergence speed of SRNLMC ($J_s(x)$) to either SRNLMC ($J_a$) or PLMC by computing the $\mathcal{W}_1$ distance between the target distribution and the sampled stationary distributions in each dimension,  and then we do the similar comparison for SRNLMC ($J_g(x)$). We obtain the plots in Figure~\ref{1wasserball} and Figure~\ref{1wassersub} where the blue line represents SRNLMC ($J_a$), the orange line represents SRNLMC ($J_s(x)$)  or SRNLMC ($J_g(x)$) and the green line represents PLMC in each subfigure.
\begin{figure}[htbp]
    \centering
    \begin{subfigure}{0.3\textwidth}
        \centering
        \includegraphics[width=\linewidth]{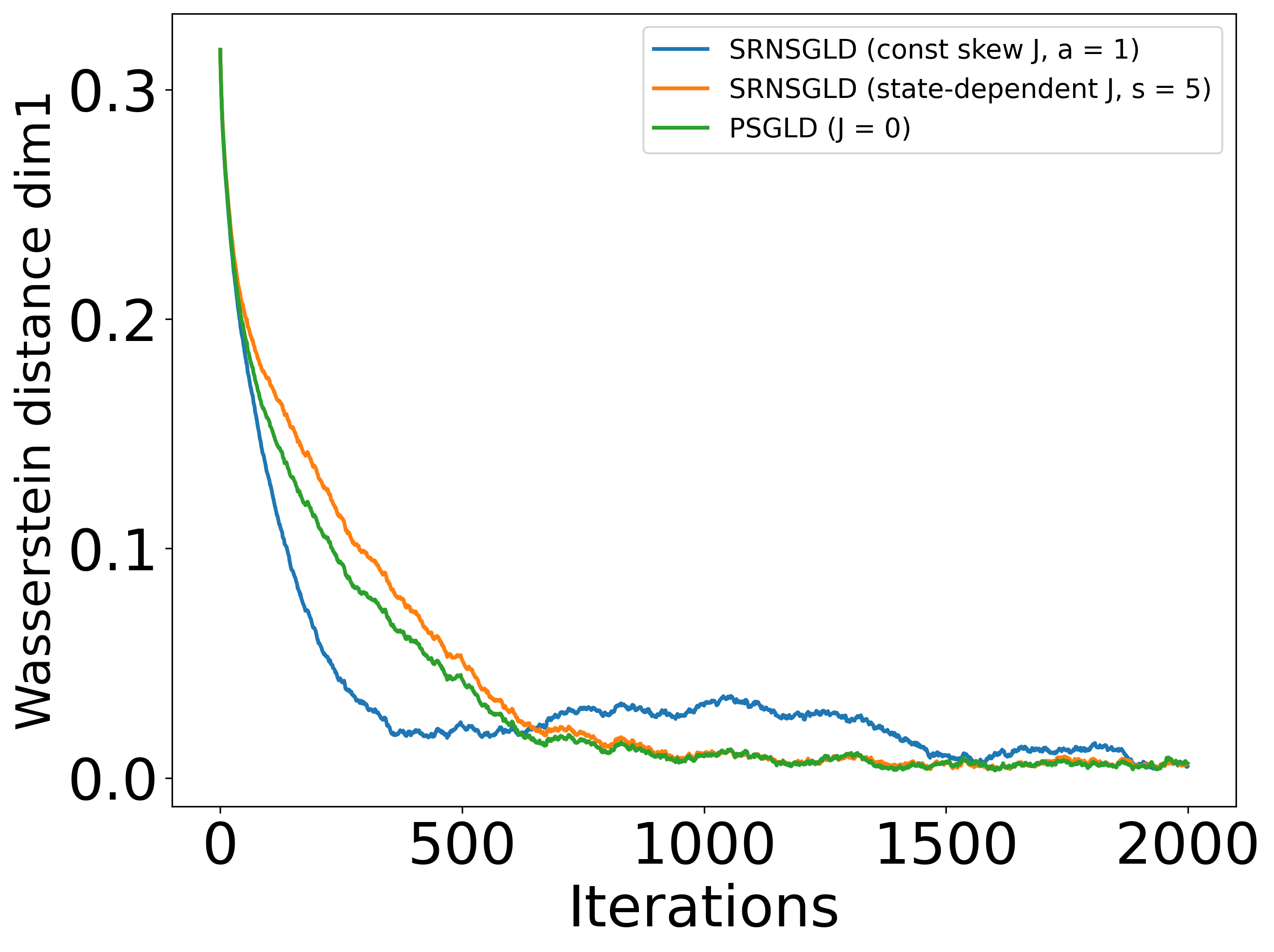}
        \caption{Dimension 1}
    \end{subfigure}%
    \hspace{0.5cm}
    \begin{subfigure}{0.3\textwidth}
        \centering
        \includegraphics[width=\linewidth]{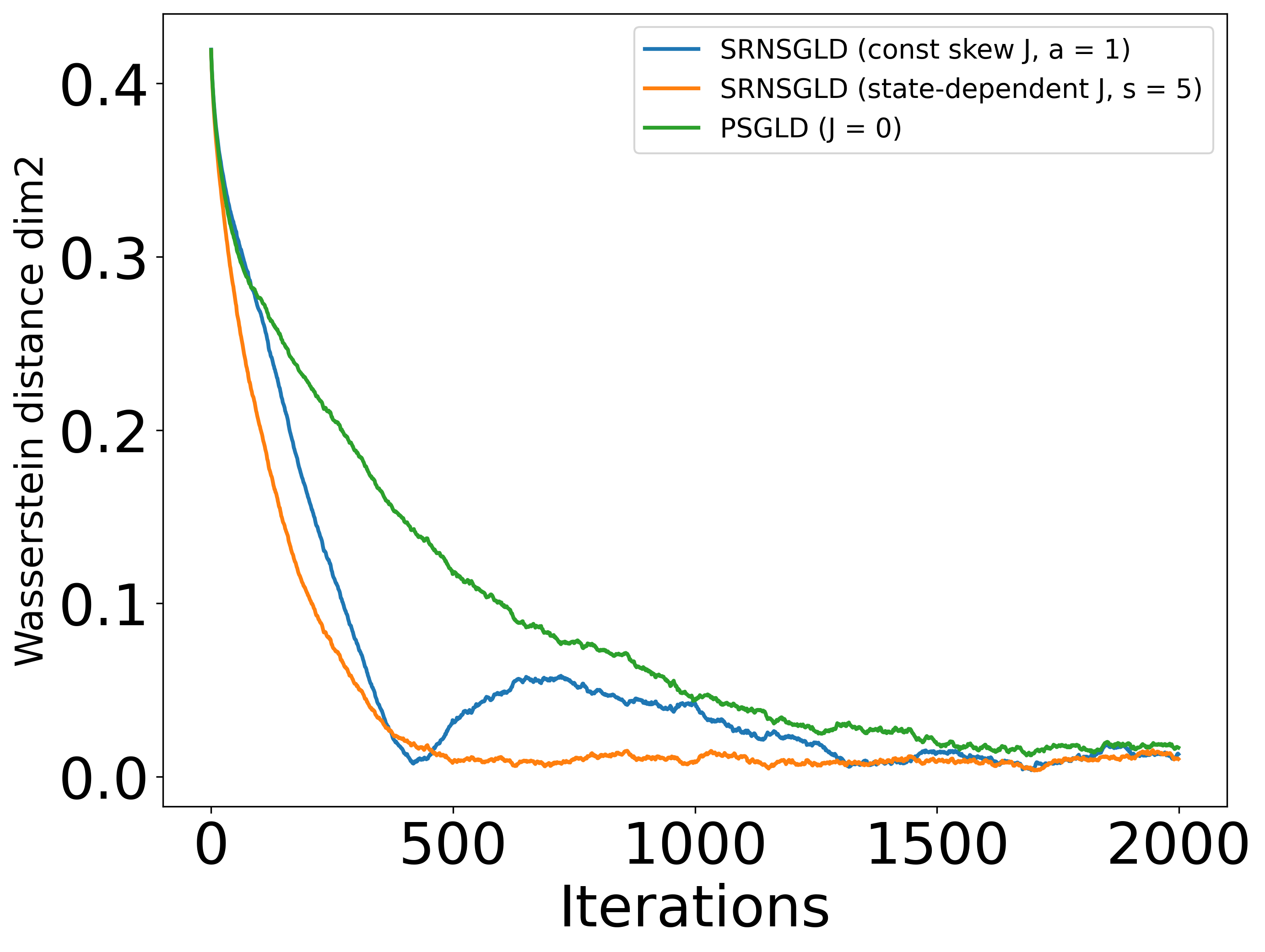}
        \caption{Dimension 2}
    \end{subfigure}%
    \hspace{0.5cm}
    \begin{subfigure}{0.3\textwidth}
        \centering
        \includegraphics[width=\linewidth]{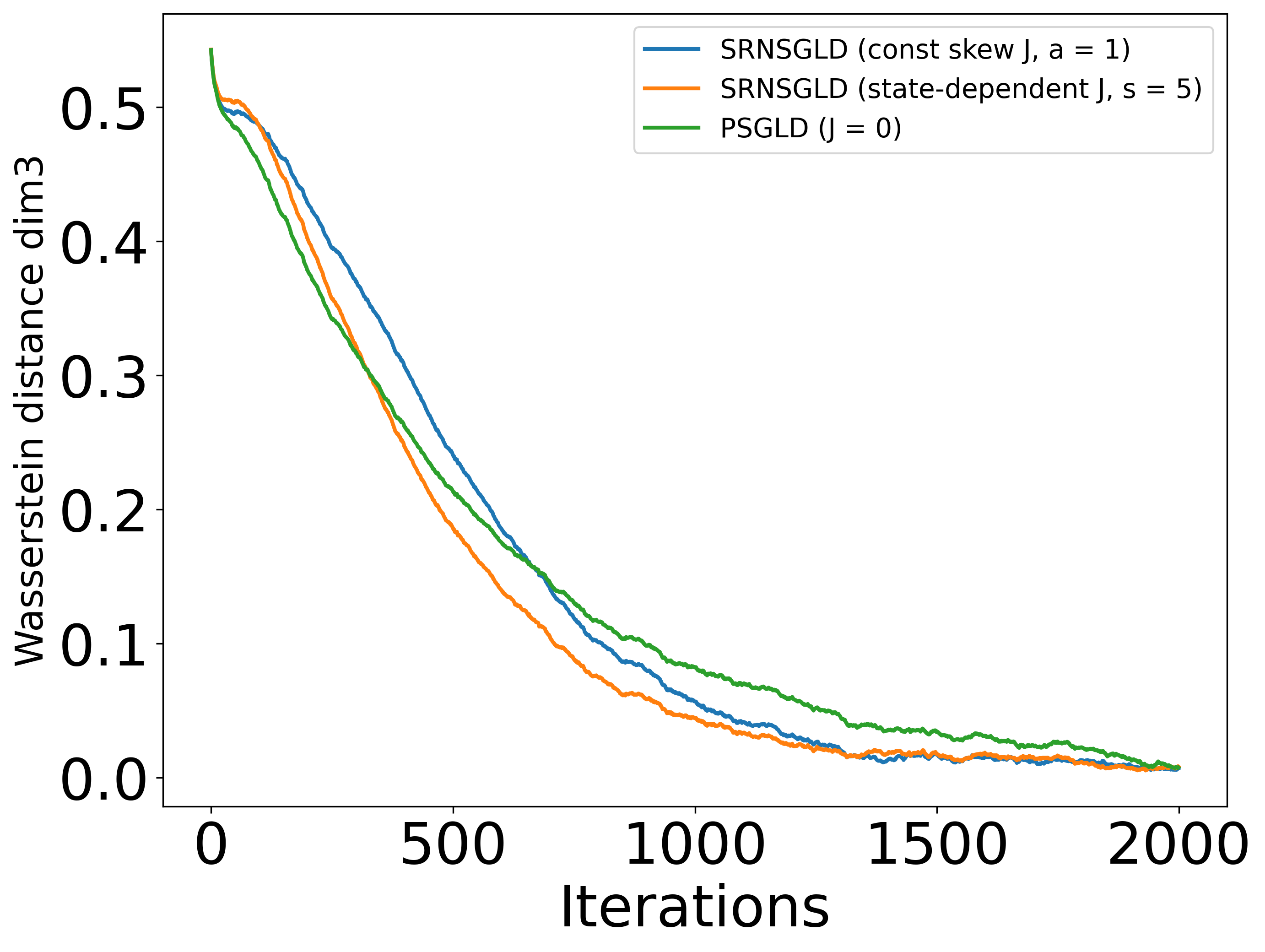}
        \caption{Dimension 3}
    \end{subfigure}
    \caption{$\mathcal{W}_1$ distance in each dimension of PLMC and SRNLMC with a centered ball constraint.}
    \label{1wasserball}
\end{figure}

Figure~\ref{1wasserball} and  Figure~\ref{1wassersub} demonstrates that both SRNLMC variants achieve superior iteration complexity compared to PLMC under fixed step sizes, corroborating the theoretical results in \cite{DFTWZ2025}. Notably, when employing the skew-symmetric matrix $J_s(x)$ (defined in~\eqref{eqn:sym:matrix}) and $J_g(x)$ (defined in~\eqref{eq:smooth:matrix}), both SRNLMC ($J_s(x)$) and   SRNLMC ($J_g(x)$) can attain a lower iteration complexity than SRNLMC ($J_a$) for equivalent  $\mathcal{W}_1$ distance targets.

\begin{figure}[htbp]
    \centering
    \begin{subfigure}[c]{0.3\textwidth}
        \centering
        \includegraphics[width=\linewidth]{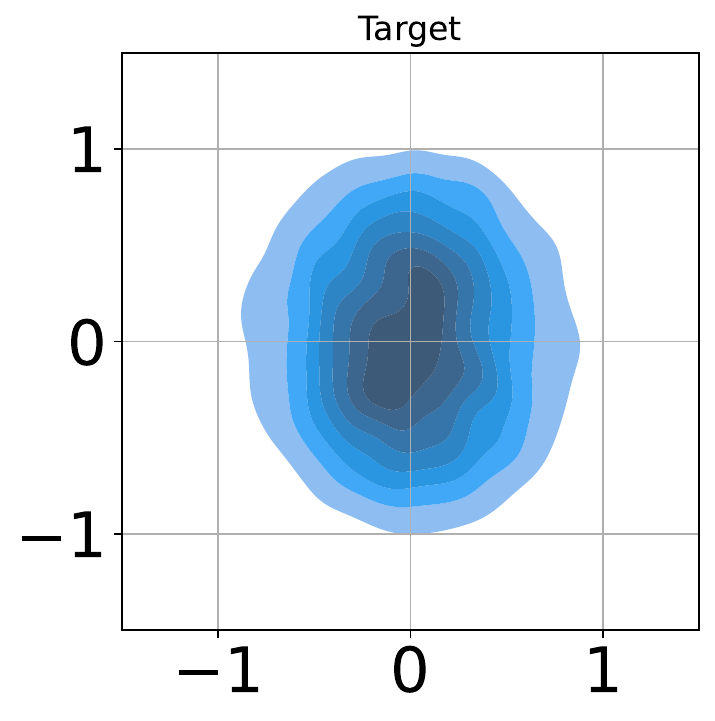}
        \caption{Target distribution}
        % \label{fig:target}
    \end{subfigure}%
    \hspace{0.1\textwidth}
    \begin{subfigure}[c]{0.3\textwidth}
        \centering
        \includegraphics[width=\linewidth]{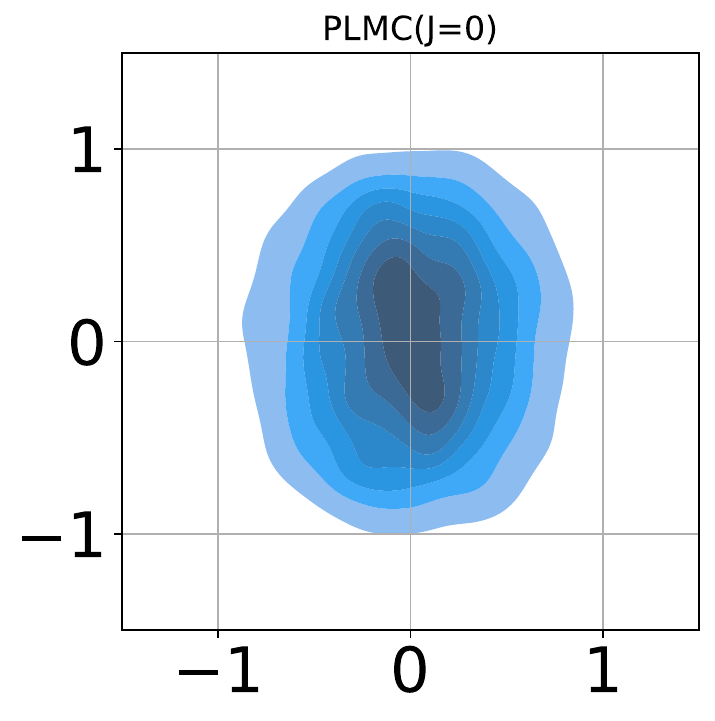}
        \caption{PLMC ($J = 0$)}
        % \label{fig:srnlmc_const}
    \end{subfigure}
    
    \vspace{0.8cm}
    
    \begin{subfigure}[c]{0.3\textwidth}
        \centering
        \includegraphics[width=\linewidth]{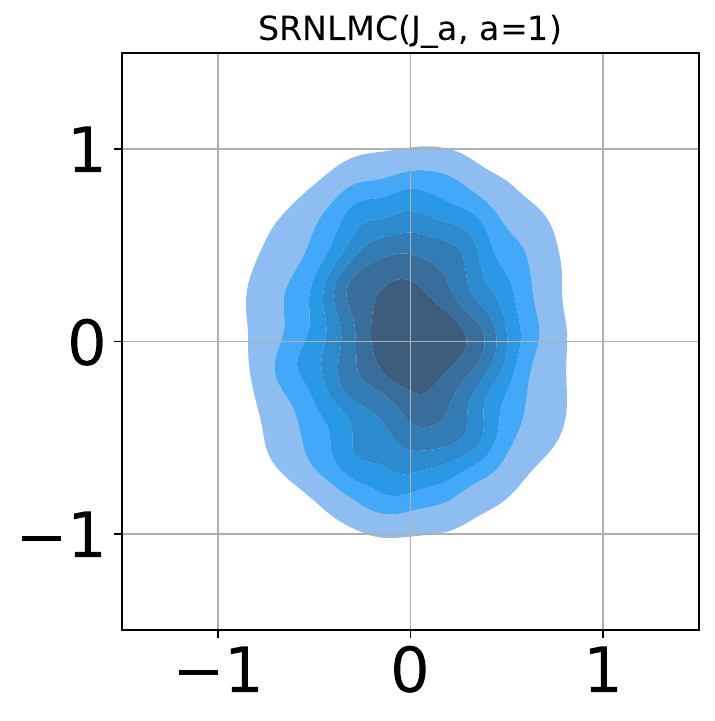}
\caption{SRNLMC ($J_{a = 1}$)}
        % \label{fig:srnlmc_sym}
    \end{subfigure}%
    \hspace{0.1\textwidth}    
    \begin{subfigure}[c]{0.3\textwidth}
        \centering
        \includegraphics[width=\linewidth]{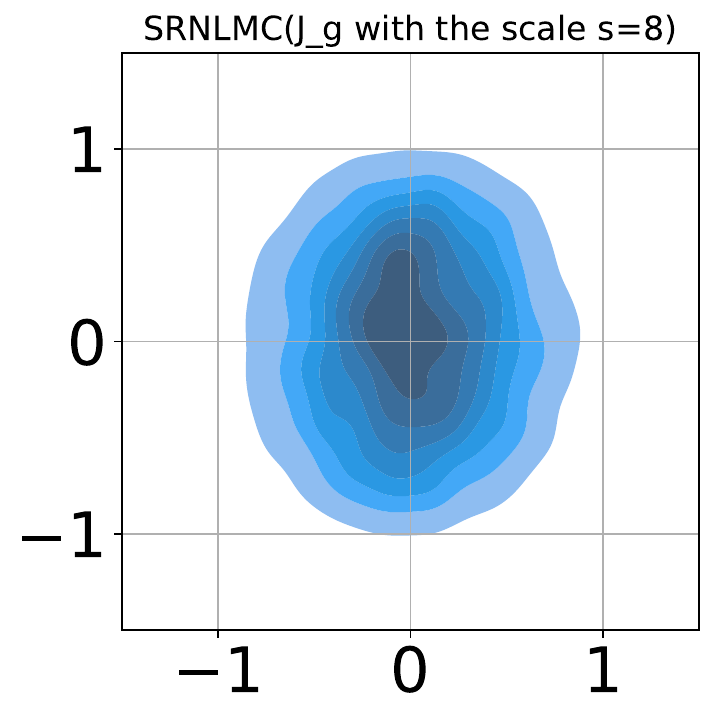}
        \caption{SRNLMC ($J_{s = 5}$)}
        % \label{fig:plmc}
    \end{subfigure}
    
    \caption{Visualized density plots for the first 2 dimensions with a smoothed $\ell_p$ ball constraint with a sublevel-set.}
    \label{visualized_sub}
\end{figure}

\begin{figure}[htbp]
    \centering
    \begin{subfigure}{0.3\textwidth}
        \centering
        \includegraphics[width=\linewidth]{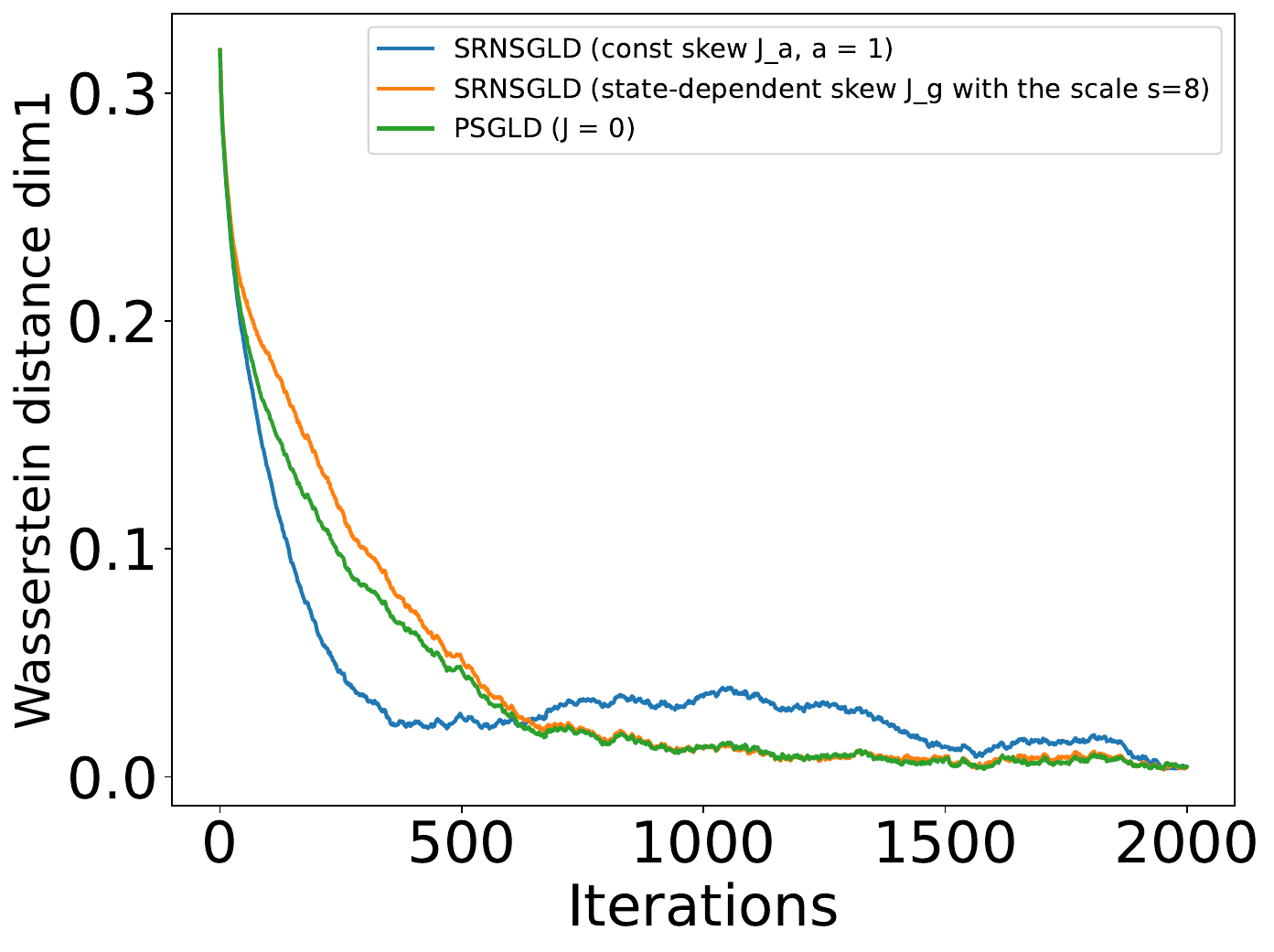}
        \caption{Dimension 1}
    \end{subfigure}%
    \hspace{0.5cm}
    \begin{subfigure}{0.3\textwidth}
        \centering
        \includegraphics[width=\linewidth]{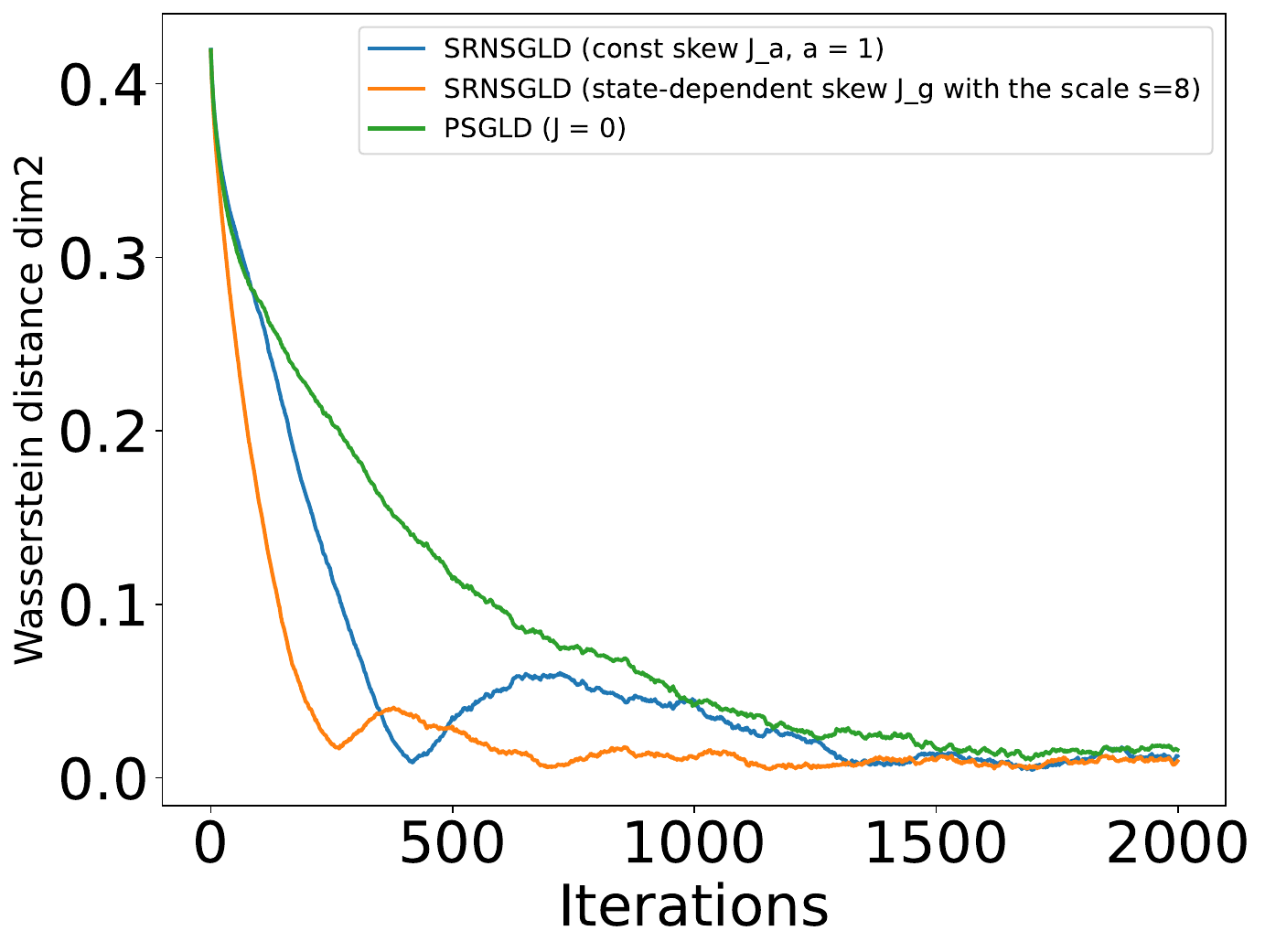}
        \caption{Dimension 2}
    \end{subfigure}%
    \hspace{0.5cm}
    \begin{subfigure}{0.3\textwidth}
        \centering
        \includegraphics[width=\linewidth]{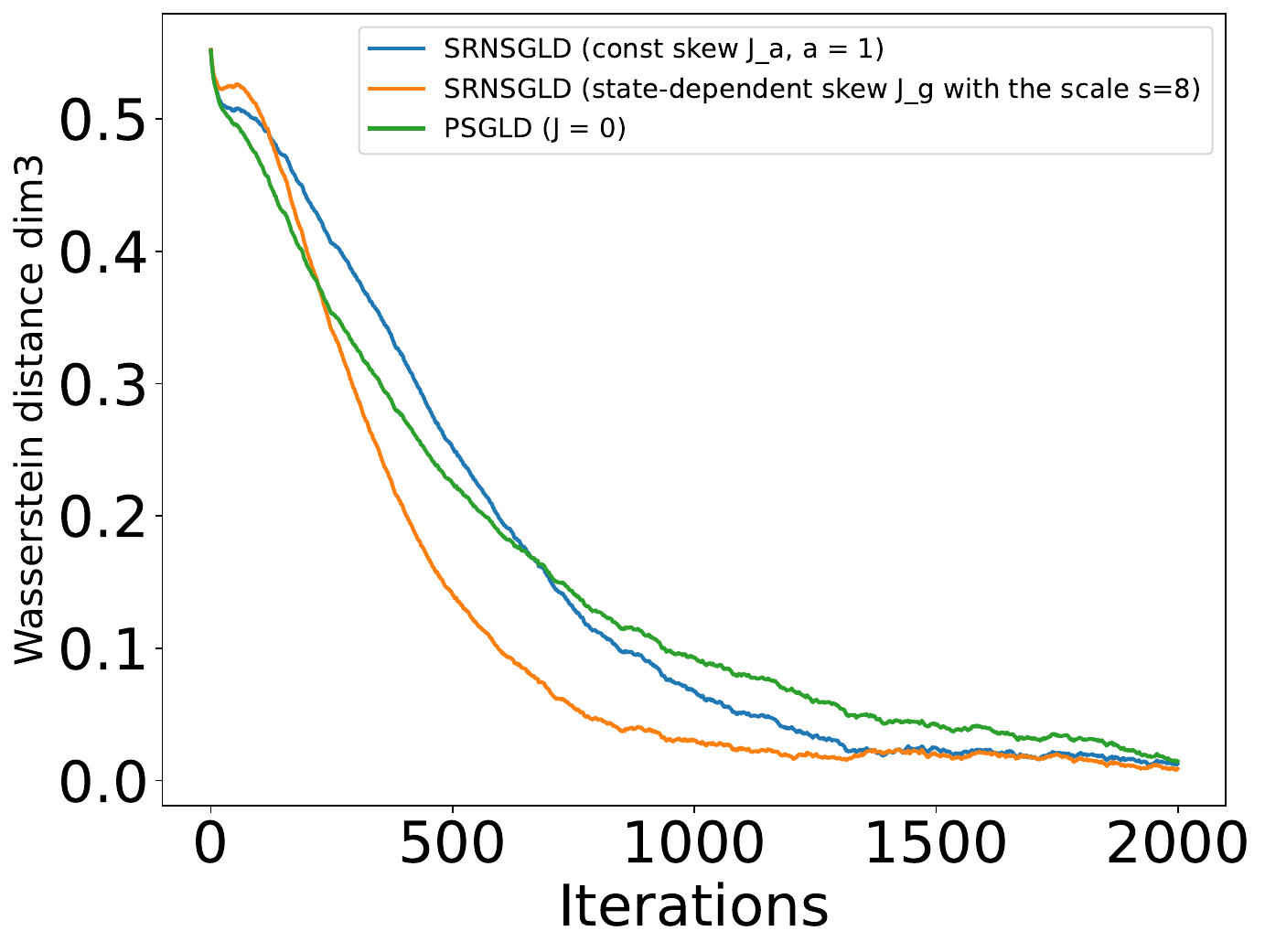}
        \caption{Dimension 3}
    \end{subfigure}
    \caption{$\mathcal{W}_1$ distance in each dimension of PLMC and SRNLMC with a smoothed $\ell_p$ ball constraint with a sublevel-set.}
    \label{1wassersub}
\end{figure}

This empirical finding validates our main result, confirming the theoretical advantage of SRNLMC over PLMC. The enhanced performance stems from the ability of the state-dependent skew-symmetric matrix to provide more effective acceleration in the constrained sampling scenarios compared to its skew-symmetric counterpart $J_a$. The state-dependent nature of SRNLMC ($J_s(x)$) and SRNLMC ($J_g(x)$) demonstrates robustness improvements over SRNLMC ($J_a$), where we can observe the orange line leads to better performance. Although constant matrices enforce fixed projection angles, it may fail for skew-projection of the sample points beyond the constraint boundary, the adaptive projection in SRNLMC ($J_s(x)$) and SRNLMC ($J_g(x)$) dynamically adjusts based on the systematical state. This state awareness prevents excessive projection angles to handle extreme sample points and gives room to apply larger step sizes while maintaining equivalent error tolerance levels.

%%%%%%%%%%%%%%%%%%%%%%%%%%%%%%%%%%%%%%%%%%%%%%%%%%%%%%%

\subsection{Constrained Bayesian Linear Regression}
\label{subsec:Bayesian:linear:example}

In our next set of experiments, we evaluate our algorithms in the context of constrained Bayesian linear regression models. We first consider the example with the ball  constraint set, a $3$-dimensional centered $\ell^2$ unit ball such that $K=\left\{x \in \mathbb{R}^3 :\|x\|_2^2 \leq 1\right\}$,
which corresponds to the ridge Bayesian linear regression. We consider the linear regression model as the following. 
\begin{equation}
\label{linear:setting}
\delta_j \sim \mathcal{N}(0,0.25), \quad a_j \sim \mathcal{N}\left(2 * \mathbf{1}, \frac{1}{2}I_3\right), \quad y_j=x_*^\top a_j+\delta_j, \quad x_*=[1,-0.7, -0.5]^\top, 
\end{equation}
where $\mathbf{1}$ denotes the all-one vector.
The prior distribution is a uniform distribution, in which case the constraints are satisfied. Our goal is to generate the posterior distribution given by
$$\mu(x) \propto \exp\left\{-\frac{1}{2}\sum_{j=1}^n\left(y_j-x^\top a_j\right)^2\right\}\mathbf{1}_{K},$$
where $\mathbf{1}_{K}$ is the indicator function for the constraint set $K$ and $n$ is the total number of data points in the training set. 

To further illustrate the efficiency of our algorithms, we will introduce a stochastic gradient setting, and we will show that our algorithms work well in the presence of stochastic gradients. \cite{DFTWZ2025} proposed \emph{skew-reflected non-reversible stochastic gradient Langevin dynamics} (SRNSGLD):
\begin{equation}
\label{eqn:SRNSGLD}
x_{k+1}=\mathcal{P}^J_{K}\left(x_{k}-\eta(I+J(x_{k}))\nabla f(x_{k},\Omega_{k+1})+\sqrt{2\eta}\xi_{k+1}\right),
\end{equation}
where $\mathcal{P}^J_{K}$ is the skew-projection onto $K$, $\xi_{k}$ are $i.i.d.$ Gaussian random vectors $\mathcal{N}(0,I_d)$ and $\nabla f(x_{k},\Omega_{k+1})$ is a conditionally unbiased estimator of $\nabla f(x_{k})$ over the dataset $\Omega_{k+1}$. For example, a common choice is the mini-batch setting, in which the full gradient is of the form 
$\nabla f(x)=\frac{1}{n}\sum_{i=1}^{n}\nabla f_{i}(x)$, where $n$ is the number of data points and $f_{i}(x)$ associates with the $i$-th data point, and the stochastic gradient is given by 
$\nabla f(x_{k},\Omega_{k+1}):=\frac{1}{m}\sum\nolimits_{i\in\Omega_{k+1}}\nabla f_{i}(x_{k})$
with $m < n$ being the batch-size and $\Omega_{k+1}^{(j)}$ are uniformly sampled random subsets of $\{1,2,\ldots,n\}$ with $|\Omega_{k+1}|=m$ and i.i.d. over $k$. In particular, if $J \equiv 0$, we obtain \emph{projected stochastic gradient Langevin dynamics}~(PSGLD). 

Under the stochastic gradient setting, we choose the batch size $m = 50$ and generate $n=10^5$ data points $\left(a_j, y_j\right)$ by SRNSGLD, where we take the skew-symmetric matrix $J_a$ with $a=1$ defined by~\eqref{eqn:skew:matrix} and skew-symmetric matrix $J_s(x)$ with $s=5$ defined by~\eqref{eqn:sym:matrix}. By  setting the red starred point as the first $2$ dimensions of $x_*$ defined in~\eqref{linear:setting} and taking $300$ iterates with the step size $\eta = 10^{-4}$, we get Figure~\ref{linear} and Figure~\ref{mse} as the following.
 
\begin{figure}[htbp]
    \centering
    \begin{subfigure}[c]{0.3\textwidth}
        \centering
        \includegraphics[width=\linewidth]{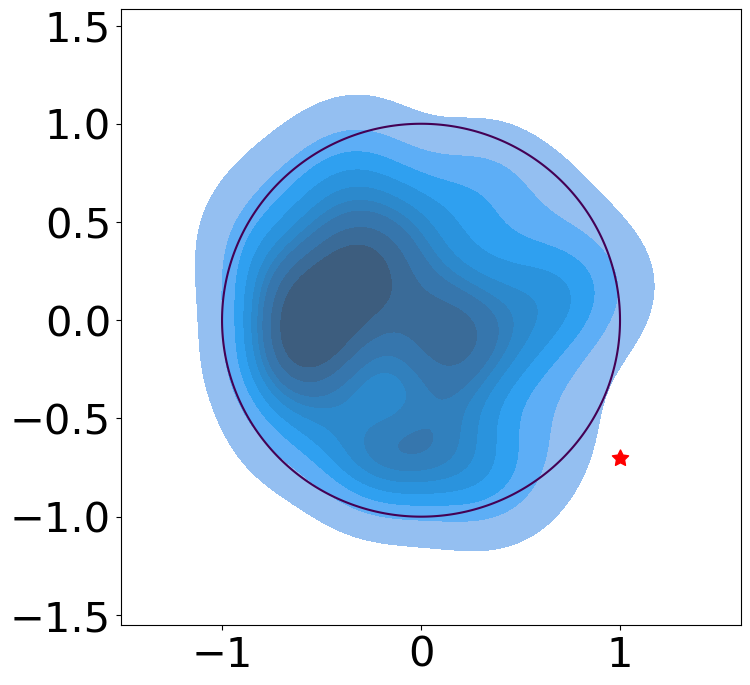}
        \caption{Prior}
    \end{subfigure}%
    \hspace{0.1\textwidth}
    \begin{subfigure}[c]{0.3\textwidth}
        \centering
        \includegraphics[width=\linewidth]{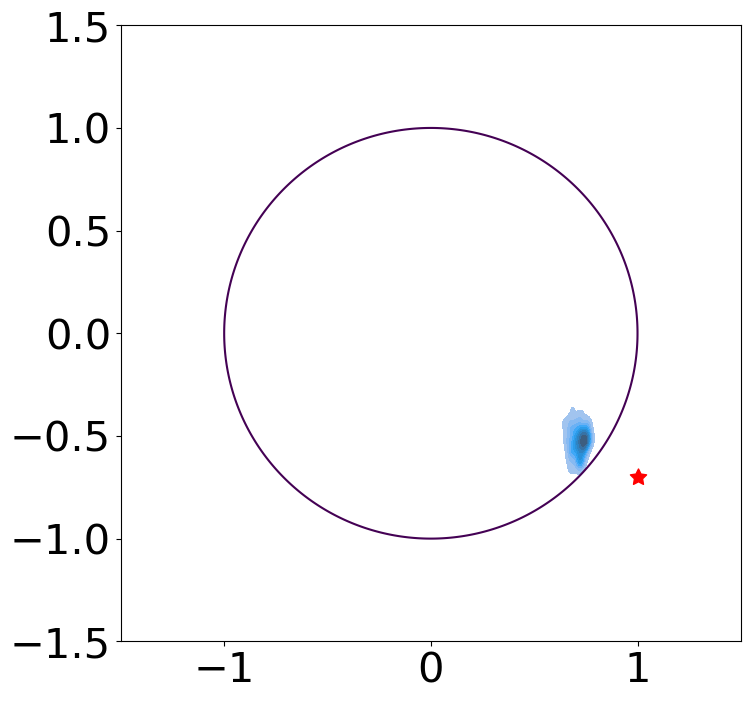}
        \caption{SRNSGLD ($J_{a = 1}$)}
    \end{subfigure}
    
    \vspace{0.4cm}
    
    \begin{subfigure}[c]{0.3\textwidth}
        \centering
        \includegraphics[width=\linewidth]{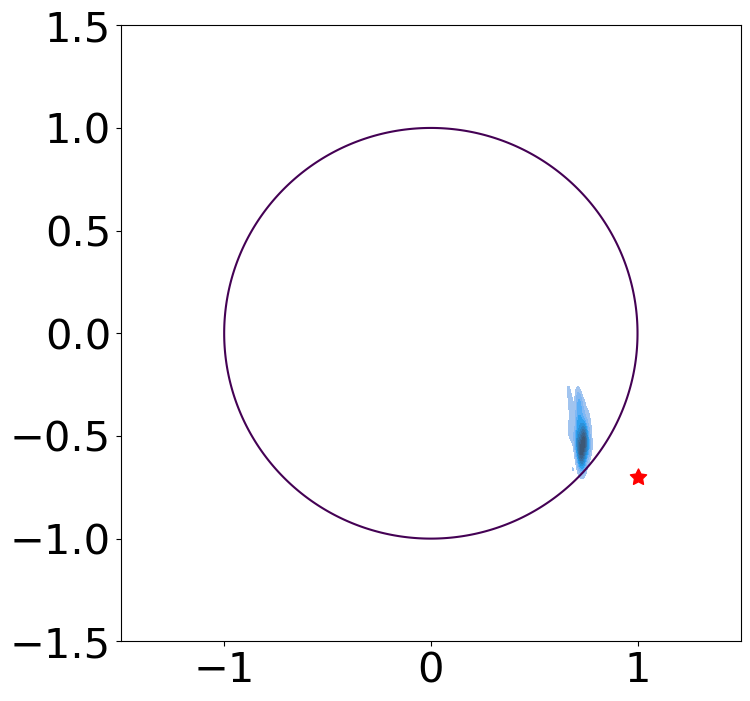}
        \caption{SRNSGLD ($J_{s = 5}$)}
    \end{subfigure}%
    \hspace{0.1\textwidth}    
    \begin{subfigure}[c]{0.3\textwidth}
        \centering
        \includegraphics[width=\linewidth]{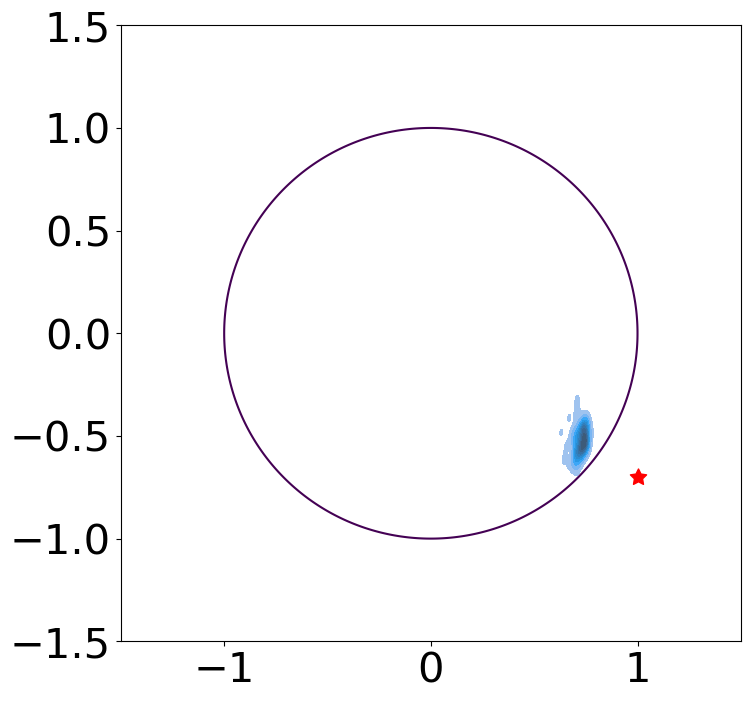}
        \caption{PSGLD ($J = 0$)}
    \end{subfigure}
    
   \caption{Prior and posterior distributions plots with a centered ball constraint.}
    \label{linear}
\end{figure}

\begin{figure}[htbp]
        \centering
        \includegraphics[width=0.5\textwidth]{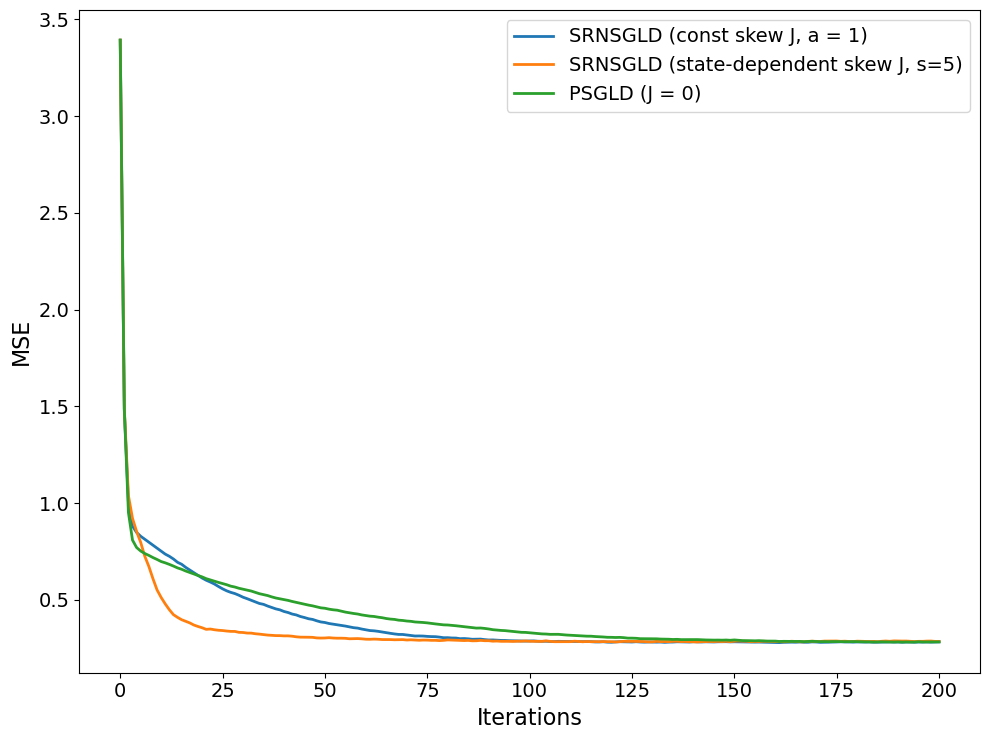}
        \caption{MSE results of SRNSGLD ($J_{a = 1}$), SRNSGLD ($J_{s = 5}$) and PSGLD ($J = 0$) for the constrained Bayesian linear regression.}
        \label{mse}
\end{figure}

Under the same setting, we consider this constrained Bayesian linear regression problem by using a smoothed $\ell_p$ ball constraint with a sublevel-set defined by $K = \{x \in \mathbb{R}^3: g(x) \leq \lambda\}$ with $g(x):= \sum_{i=1}^d(x_i^2 + \varepsilon^2)^{p/2}$
from the definition of ~\eqref{eq:smoothlp_Rd:f}, where we take $p = 4$, $\varepsilon = 0.2$, $\lambda = 1.0$ and the scaling parameter $s = 8$ for $J_g(x)$ defined in~\eqref{eq:smooth:matrix:scale}. We summarize the results in Figure~\ref{linear2} and Figure~\ref{mse2}.

\begin{figure}[htbp]
    \centering
    \begin{subfigure}[c]{0.3\textwidth}
        \centering
        \includegraphics[width=\linewidth]{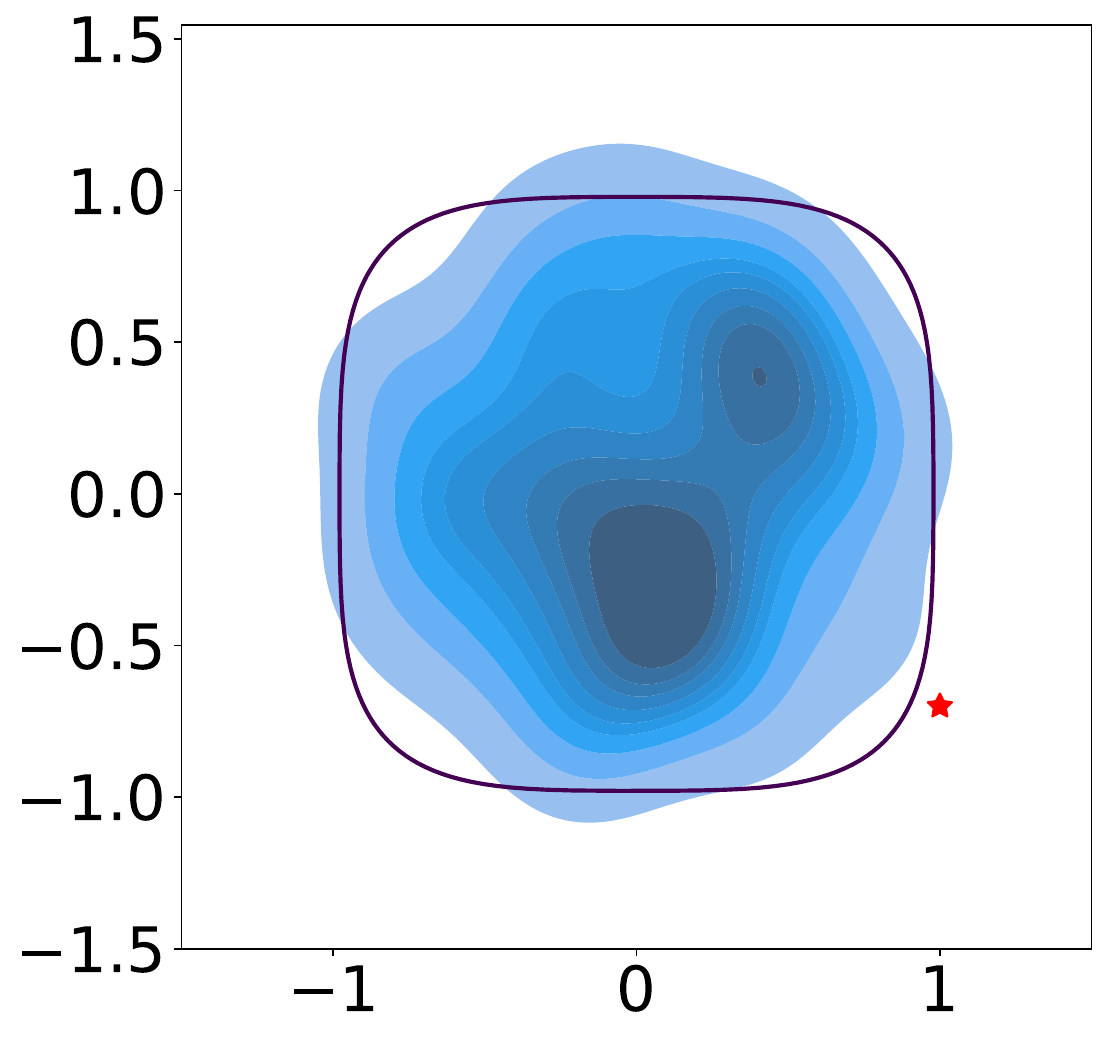}
        \caption{Prior}
    \end{subfigure}%
    \hspace{0.1\textwidth}
    \begin{subfigure}[c]{0.3\textwidth}
        \centering
        \includegraphics[width=\linewidth]{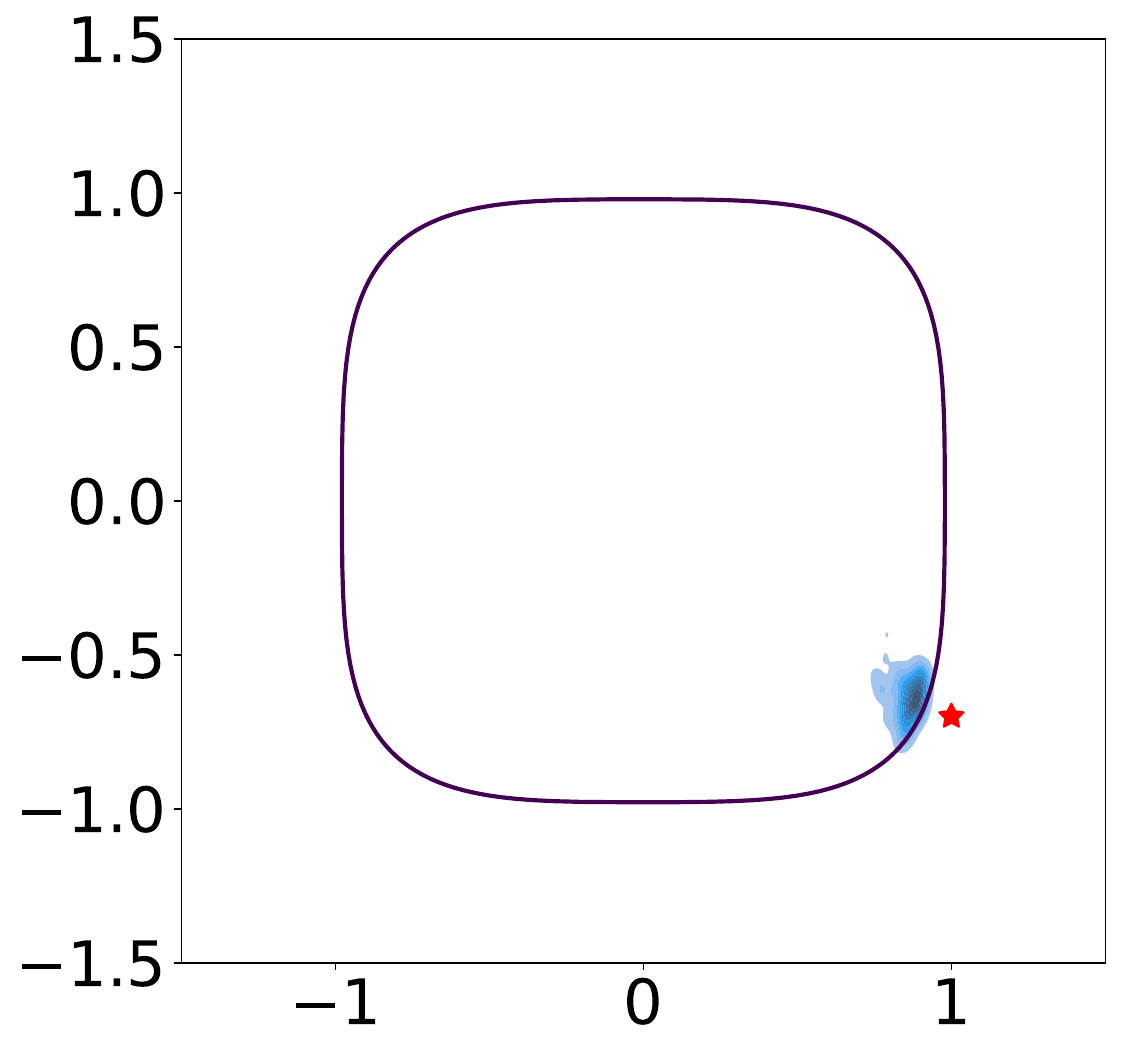}
        \caption{SRNSGLD ($J_{a = 1}$)}
    \end{subfigure}
    
    \vspace{0.4cm}
    
    \begin{subfigure}[c]{0.3\textwidth}
        \centering
        \includegraphics[width=\linewidth]{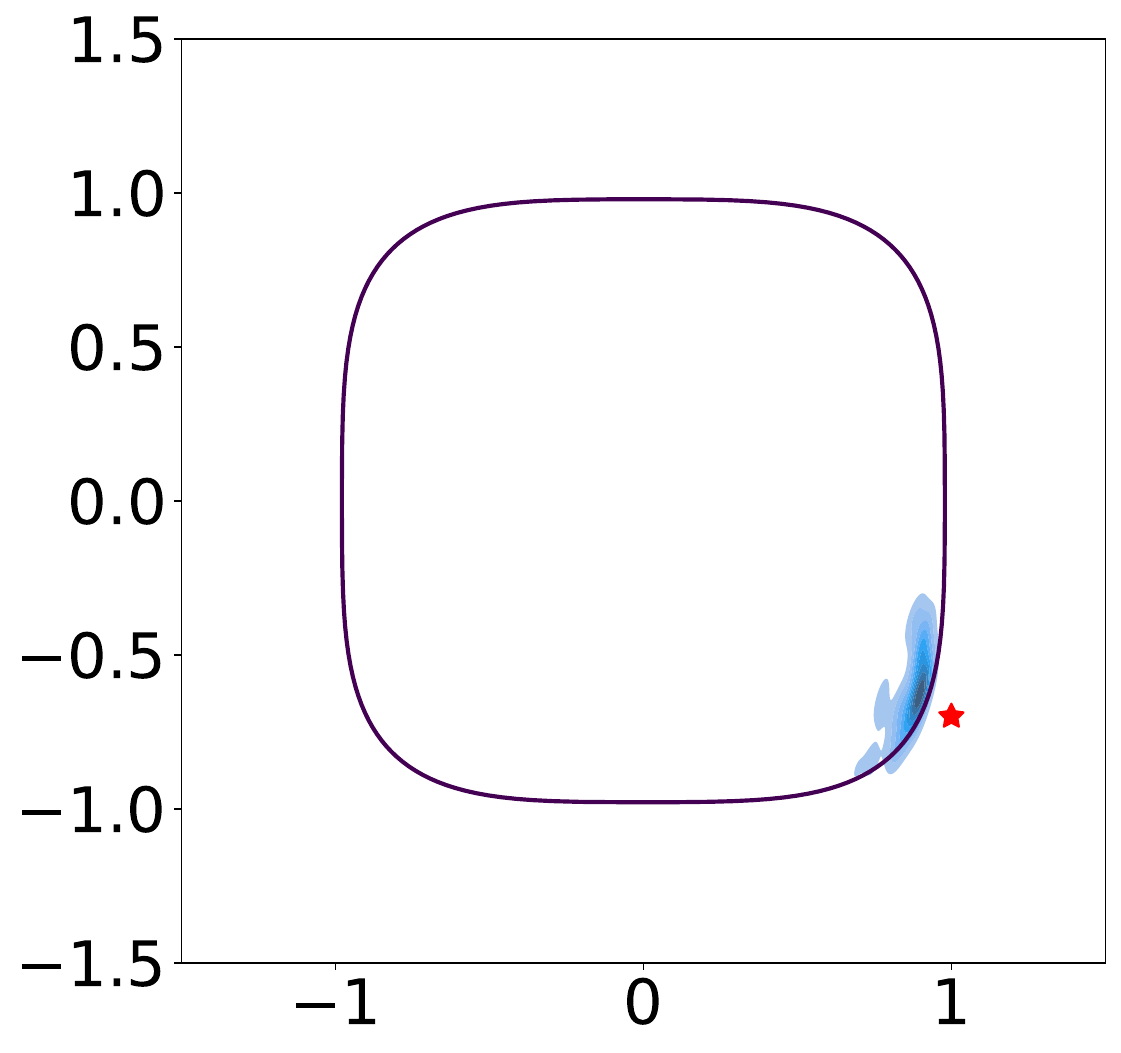}
        \caption{SRNSGLD ($J_{g}$)}
    \end{subfigure}%
    \hspace{0.1\textwidth}    
    \begin{subfigure}[c]{0.3\textwidth}
        \centering
        \includegraphics[width=\linewidth]{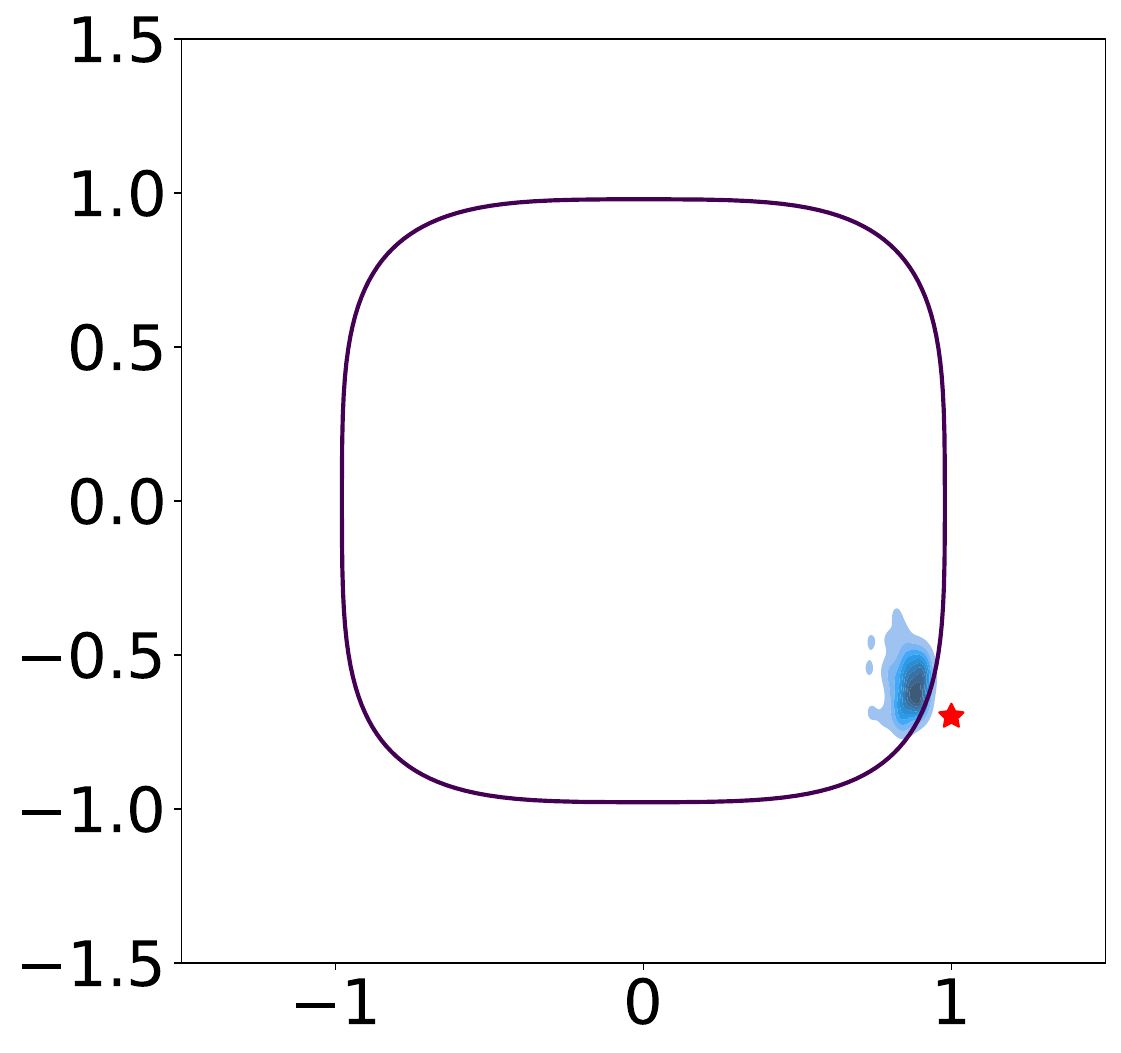}
        \caption{PSGLD ($J = 0$)}
    \end{subfigure}
    
   \caption{Prior and posterior distributions plots with a smoothed $\ell_p$ ball constraint with a sublevel-set.}
    \label{linear2}
\end{figure}

\begin{figure}[htbp]
        \centering
        \includegraphics[width=0.5\textwidth]{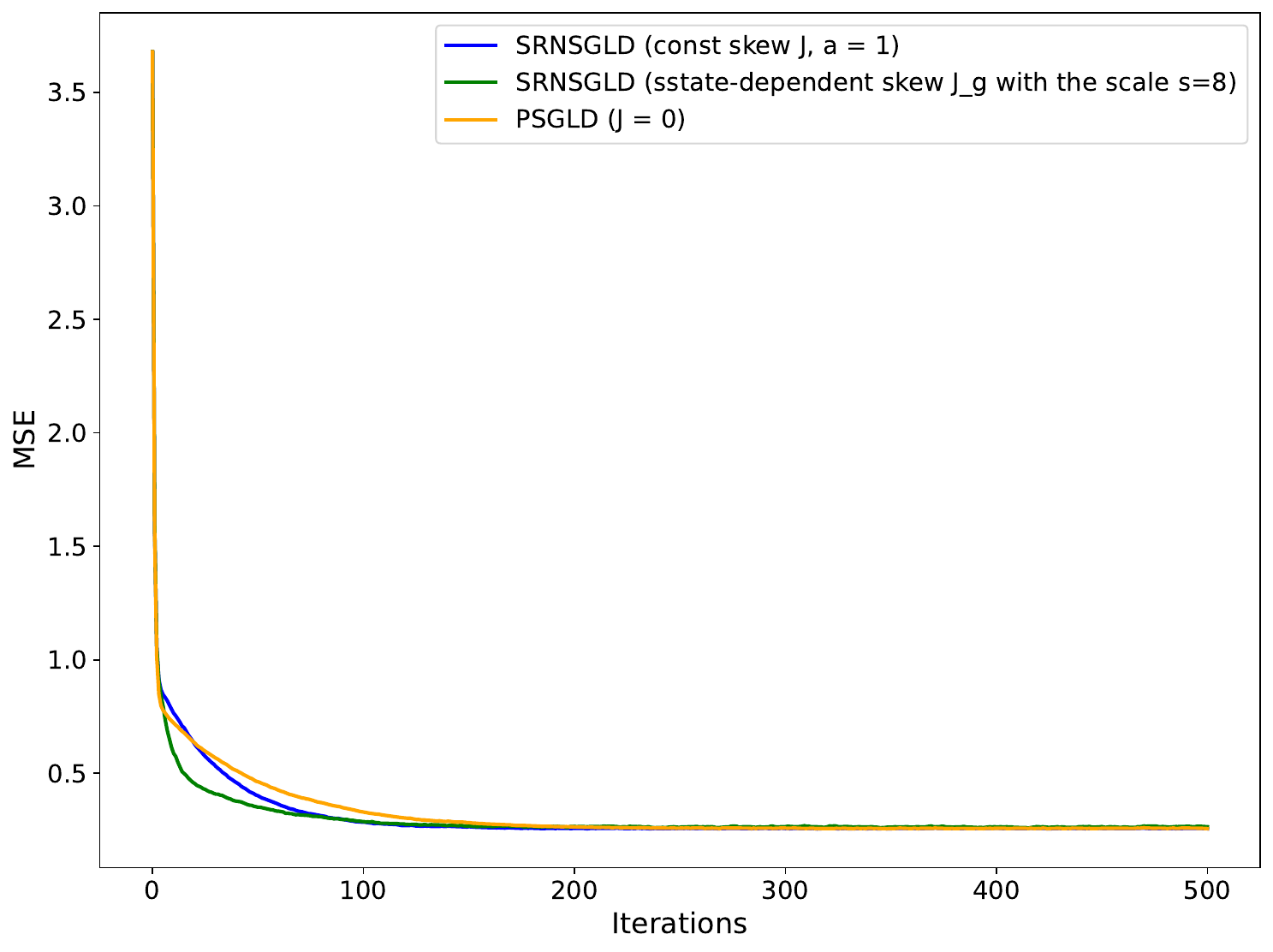}
        \caption{MSE results of SRNSGLD ($J_{a = 1}$), SRNSGLD ($J_{g}$) and PSGLD ($J = 0$) for the constrained Bayesian linear regression with a smoothed $\ell_p$ ball constraint with a sublevel-set.}
        \label{mse2}
\end{figure}

 In Figure~\ref{linear} and Figure~\ref{linear2}, the prior uniform distribution is shown in panel (a). In panels (b), (c), (d), our simulation shows that SRNSGLD ($J_a$),  SRNSGLD ($J_s(x)$ in Figure~\ref{linear}), SRNSGLD ($J_g(x)$ in Figure~\ref{linear2}) and PSGLD can converge to the samples whose distribution concentrates around the closest position to the target value, which are the red stars shown outside the constraints. 
By computing the MSE at the $k$-th run using the formula
$$\operatorname{MSE}_k := \frac{1}{n} \sum_{j=1}^n \left( y_j - \left( x_k \right)^\top a_j \right)^2,$$
we get  Figure~\ref{mse} and Figure~\ref{mse2}. The blue line represents the mean squared error (MSE) of SRNSGLD ($J_a$), the orange line corresponds to  SRNSGLD ($J_s(x)$) in Figure~\ref{mse} and SRNSGLD ($J_g(x)$) in Figure~\ref{mse2}, and the green line represents PSGLD. Although we do not provide theoretical guarantees for acceleration with respect to MSE, our numerical results show that  either SRNSGLD ($J_s(x)$) or SRNSGLD ($J_g(x)$) achieves significantly better acceleration performance than either SRNSGLD ($J_a$) or PSGLD.
%%%%%%%%%%%%%%%

\subsection{Constrained Bayesian Logistic Regression}
\label{subsec:Bayesian:log:example}

To test the performance of our algorithm in binary classification problems, we implement the constrained Bayesian linear regression models on both synthetic and real data, where the constraint set is the ball of radius $r$ centered at $0\in\mathbb{R}^{d}$ given by
\begin{equation}
\label{eqn:radius}
K_r=\left\{x \in \mathbb{R}^d :\|x\|_2^2 \leq r\right\}.
\end{equation}

Suppose we can access a dataset $Z=\left\{z_j\right\}_{j=1}^{n}$ where $z_j=\left(X_j, y_j\right), X_j \in \mathbb{R}^d$ are the features and $y_j \in\{0,1\}$ are the labels with the assumption that $X_j$ are independent and the probability distribution of $y_j$ given $X_j$ and the regression coefficients $\beta \in \mathbb{R}^d$ are given by
$$\mathbb{P}\left(y_j=1 \mid X_j, \beta\right)=\frac{1}{1+e^{-\beta^\top X_j}}.$$
In our experiments, we choose the prior distribution of $\beta$ to be the uniform distribution in the ball constraint. Then the goal of the constrained Bayesian logistic regression is to sample from $\mu(\beta) \propto e^{-f(\beta)}\mathbf{1}_K$ where:
\begin{equation}
f(\beta):=-\sum_{j=1}^{n} \log p\left(y_j \mid X_j, \beta\right)-\log p(\beta)=\sum_{j=1}^{n} \log \left(1+e^{-\beta^\top X_j}\right)+\log V_d,
\end{equation}
and $V_d:=\frac{\pi^{d/2}}{\Gamma(\frac{d}{2}+1)}$ is the volume of the unit ball in $\mathbb{R}^d$.

 We will conduct two sets of experiments by using either synthetic data or real data, where one considers the constrained set to be a centered ball defined in~\eqref{eqn:radius}, and another considers the constrained set to be the smoothed $\ell_p$ ball constraint with a sublevel-set defined by
\begin{equation}
 K = \{x \in \mathbb{R}^3: g(x) \leq \lambda\}, \qquad g(x):= \sum_{i=1}^3(x_i^2 + \varepsilon^2)^{p/2}.
 \label{eq:sublevel}
\end{equation}
from the definition of \eqref{eq:smoothlp_Rd:f}. We will specify the parameters for this contraint set later.

%%%%%%%%%%%%%%%%%%%%%%%%%%%%%%%%%%%
\paragraph{Synthetic Data}

Consider the following example of $d = 3$. We first generate $n = 2000$ synthetic data by the following model
\begin{equation}
X_j \sim \mathcal{N}\left(0,2 I_3\right), \quad p_j \sim \mathcal{U}(0,1), \quad y_j= \begin{cases}1 & \text { if } p_j \leq \frac{1}{1+e^{-\beta^\top X_j}} \\ 0 & \text { otherwise }\end{cases},
\end{equation}
where $\mathcal{U}(0,1)$ stands for the uniform distribution on $[0,1]$ and the prior distribution of $ \beta=\left[\beta_1, \beta_2, \beta_3\right]^\top \in \mathbb{R}^3$ is a uniform distribution on the centered ball $K_{r=1}$ in $\mathbb{R}^3$ defined in~\eqref{eqn:radius}. Then we implement them with $1000$ iterates by choosing the step size $\eta=10^{-4}$ with the batch size $m=50$. For the constrained Bayesian logistic regression, it is not practical for us to compute the $\mathcal{W}_1$ distance between the approximated Gibbs distribution and the empirical distribution. Hence, for such a binary classification problem, we can use the accuracy over the training set and test set to measure the goodness of the convergence performance. 

 In the experiment with the ball constraint set in~\eqref{eqn:radius} with $r = 1$, we choose the skew-symmetric matrix $J_a$ in~\eqref{eqn:skew:matrix} with $a = 1$ for SRNSGLD ($J_a$) shown by the blue line and the skew-symmetric matrix $J_s(x)$ in~\eqref{eqn:sym:matrix} with $s = 10$ for SRNSGLD ($J_s(x)$) shown by the green dash line. We use $20\%$ of the whole dataset as the test set. The mean and standard deviation of the accuracy distribution are shown in Figure~\ref{logsyth}.
\begin{figure}[htbp]
    \centering
    \begin{subfigure}{0.45\textwidth}
        \centering
        \includegraphics[width=\linewidth]{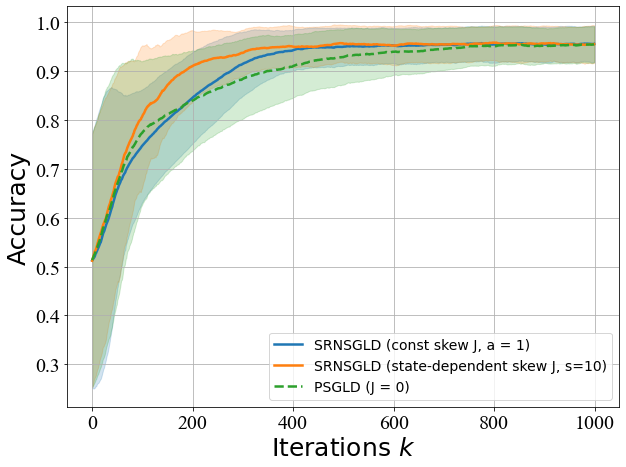}
        \caption{Accuracy over the training set with a centered ball constraint.}
    \end{subfigure}%
    \hspace{0.8cm}
    \begin{subfigure}{0.45\textwidth}
        \centering
        \includegraphics[width=\linewidth]{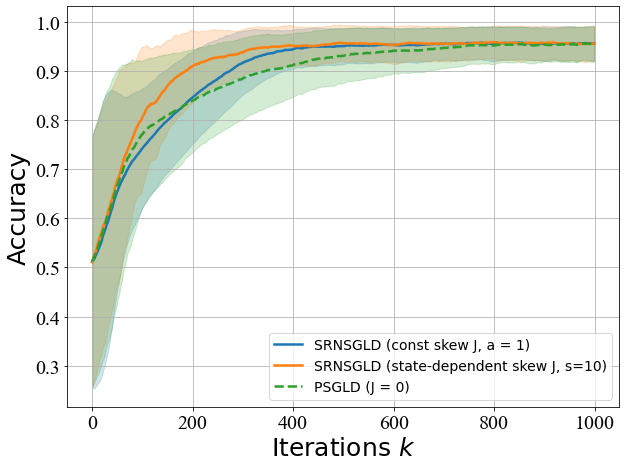}
        \caption{Accuracy over the test set with a centered ball constraint.}
    \end{subfigure}%
    \caption{Accuracy over the training set and the test set for the synthetic data. The blue part (solid line) denotes the mean and standard deviation of SRNSGLD ($J_a$), the green part (dash-dot line) denotes the mean and standard deviation of SRNSGLD ($J_s(x)$) and the orange part (dashed line) denotes the mean and standard deviation of PSGLD ($J\equiv 0$).}
    \label{logsyth}    
\end{figure}

Under the same setting, we consider the smoothed $\ell_p$ ball constraint with a sublevel-set in~\eqref{eq:sublevel} with $p = 4$ in this synthetic data example, $\varepsilon = 0.2$, $\lambda = 1.0$ and the scaling parameter $s = 10$ for $J_g(x)$ defined in~\eqref{eq:smooth:matrix:scale}.
We obtain Figure~\ref{logsyth2}.

\begin{figure}[htbp]
    \centering
    \begin{subfigure}{0.45\textwidth}
        \centering
        \includegraphics[width=\linewidth]{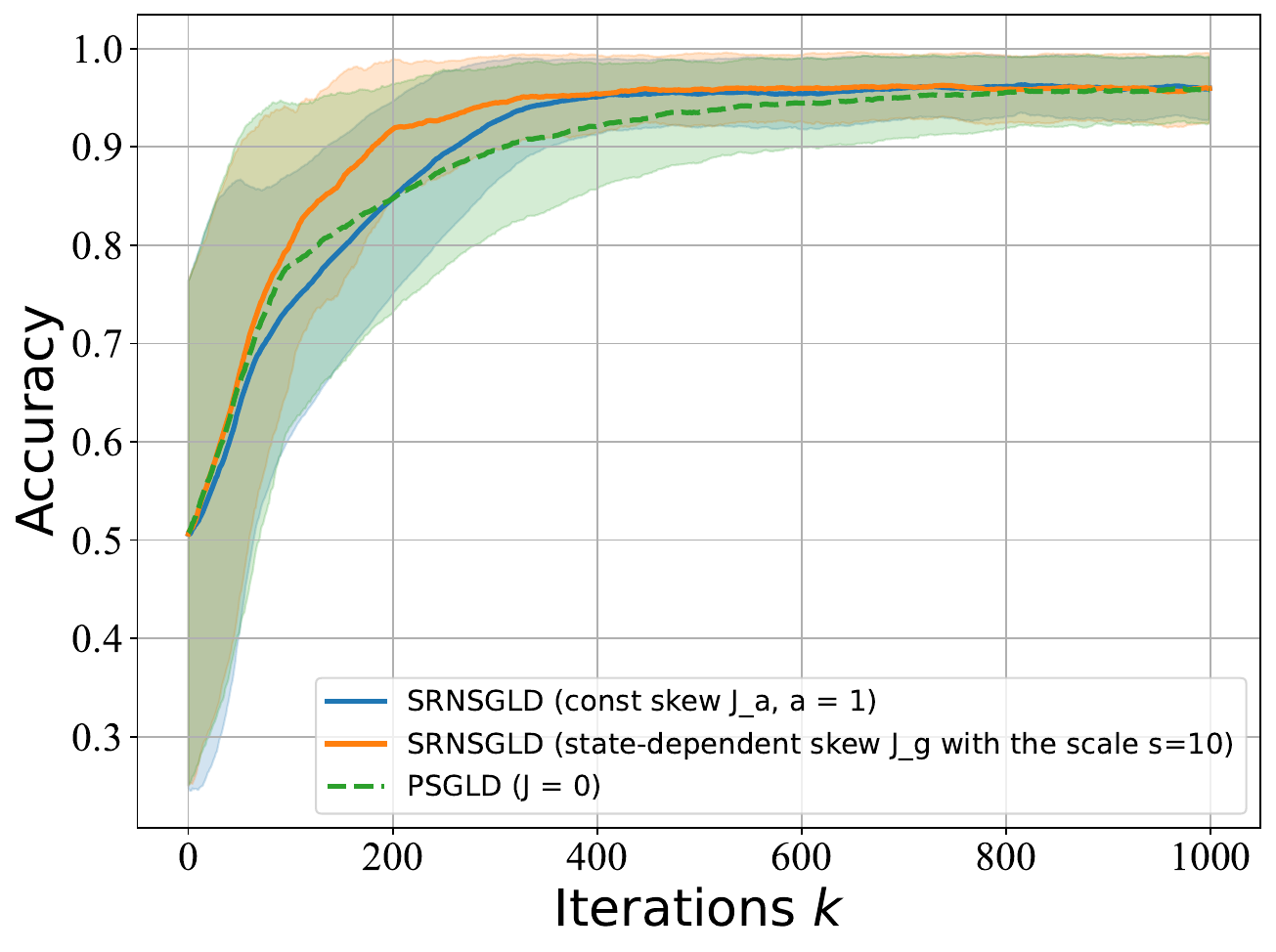}
        \caption{Accuracy over the training set with a smoothed $\ell_p$ ball constraint with a sublevel-set.}
    \end{subfigure}%
    \hspace{0.8cm}
    \begin{subfigure}{0.45\textwidth}
        \centering
        \includegraphics[width=\linewidth]{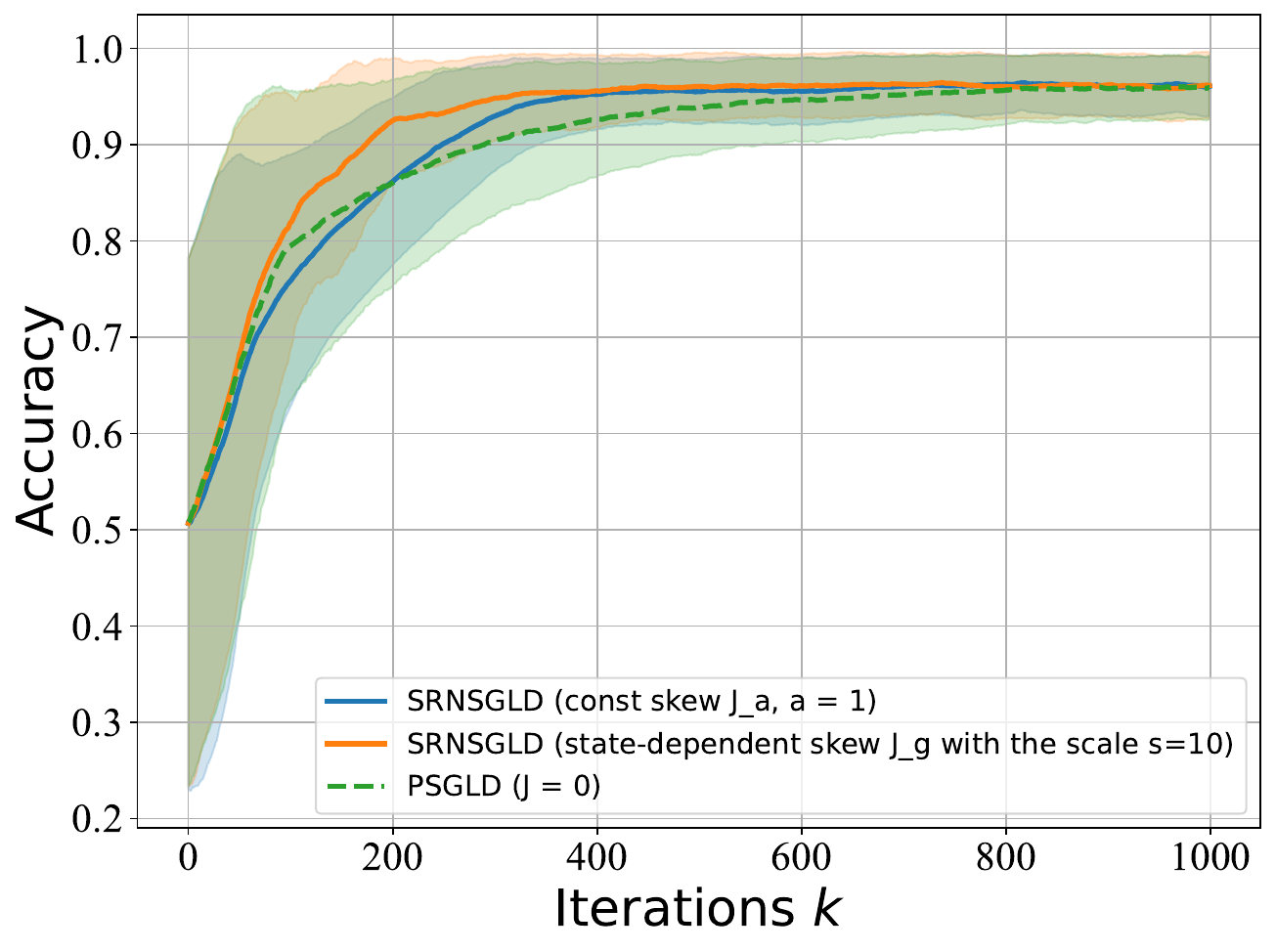}
        \caption{Accuracy over the test set with a smoothed $\ell_p$ ball constraint with a sublevel-set.}
    \end{subfigure}%
    \caption{Accuracy over the training set and the test set for the synthetic data. The blue part (solid line) denotes the mean and standard deviation of SRNSGLD ($J_a$), the green part (dash-dot line) denotes the mean and standard deviation of SRNSGLD ($J_g(x)$) and the orange part (dashed line) denotes the mean and standard deviation of PSGLD ($J\equiv 0$).}
    \label{logsyth2}    
\end{figure}
 We observe that, both SRNSGLD ($J_s(x)$) in green line in Figure~\ref{logsyth} and SRNSGLD ($J_g(x)$) in orange line in Figure~\ref{logsyth2} requires fewer iterations to achieve the same level of accuracy on the test set and exhibits more stable sampling behavior comparing to PSGLD in green dash line. Moreover, both SRNSGLD ($J_s(x)$) in green line in Figure~\ref{logsyth} and SRNSGLD ($J_g(x)$) in orange line in Figure~\ref{logsyth2} demonstrate suprior performance to SRNSGLD ($J_a$) with constant skew-symmetric matrix and PSGLD ($J = 0$).

%%%%%%%%%%%%%%%%
%\vspace{-0.1in}
\paragraph{Real Data}

In this section, we consider the constrained Bayesian logistic regression problem on the MAGIC Gamma Telescope dataset\footnote{The Telescope dataset is publicly available from:\\ \href{https://archive.ics.uci.edu/ml/datasets/magic+gamma+telescope}{\texttt{https://archive.ics.uci.edu/ml/datasets/magic+gamma+telescope.}} } and the Titanic dataset\footnote{The Titanic dataset is publicly available from:\\ \href{https://www.kaggle.com/c/titanic}{\texttt{https://www.kaggle.com/c/titanic}.}}. The Telescope dataset contains $n=19020$ samples with dimension $d=9$, describing the registration of high energy gamma particles in a ground-based atmospheric Cherenkov gamma telescope using the imaging technique. The Titanic dataset contains $891$ samples with $10$ features representing information about the passengers. And the goal is to predict whether a passenger survived or not based on these features. However, some of the features are irrelevant to our goal, such as the ID number. Therefore, after a pre-proceeding on the raw dataset, we obtained a dataset containing $n=891$ labeled samples with $d=9$ features. 

For both real data sets, we initialize SRNSGLD and PSGLD with the uniform distribution on the centered unit ball in the respective dimensions. To fit the dimension, we need to adjust the constant skew-symmetric matrix $J_a$ to fit the $9$-dimensional setting with the whole superdiagonal set to be $a$ and the subdiagonal set to be $-a$. We let the state-dependent skew-symmetric matrix $J_s(x)$ be a $9 \times 9$ skew-symmetric matrix defined for $x \in K_r \subset \mathbb{R}^9$. We construct $J_s(x)$ as a block-diagonal matrix:
\begin{align*}
J_s(x)=\left(\begin{array}{ccc}
J^{(1)}\left(x_1, x_2, x_3\right) & \mathbf{0}_{3 \times 3} & \mathbf{0}_{3 \times 3} \\
\mathbf{0}_{3 \times 3} & J^{(2)}\left(x_4, x_5, x_6\right) & \mathbf{0}_{3 \times 3} \\
\mathbf{0}_{3 \times 3} & \mathbf{0}_{3 \times 3} & J^{(3)}\left(x_7, x_8, x_9\right)
\end{array}\right),
\end{align*}
where $\mathbf{0}_{m \times n}$ denotes an $m \times n$ zero matrix and we define the blocks by
\begin{align*}
&J^{(1)}\left(x_1, x_2, x_3\right)=\left(\begin{array}{ccc}
0 & -s_1 x_3 & s_1 x_2 \\
s_1 x_3 & 0 & -s_1 x_1 \\
-s_1 x_2 & s_1 x_1 & 0
\end{array}\right),
\\
& J^{(2)}\left(x_4, x_5, x_6\right)=\left(\begin{array}{ccc}
0 & -s_2 x_6 & s_2 x_5 \\
s_2 x_6 & 0 & -s_2 x_4 \\
-s_2 x_5 & s_2 x_4 & 0
\end{array}\right),
\\
&J^{(3)}\left(x_7, x_8, x_9\right)=\left(\begin{array}{ccc}
0 & -s_3 x_9 & s_3 x_8 \\
s_3 x_9 & 0 & -s_3 x_7 \\
-s_3 x_8 & s_3 x_7 & 0
\end{array}\right).
\end{align*}
The $3$-dimensional constant $s = [s_1, s_2,s_3]$ can be chosen independently (e.g., $s_1, s_2, s_3 \neq 0$).

 Based on the formula of the smoothed $\ell_p$ ball constraint with a sublevel-set in~\eqref{eq:smooth:matrix}, we can also construct the skew-symmetric matrix $J_g(x)$ as a block-diagonal matrix. Moreover, we incorporate a scaling factor $s = [s_1,s_2,s_3] \in \mathbb{R}^3$ into $J_{g}(x)$ that allows fine-tuning such that
\begin{align}
\label{eq:block:matrix:scale}
J_g(x)=\left(\begin{array}{ccc}
s_1 J_g^{(1)}\left(x_1, x_2, x_3\right) & \mathbf{0}_{3 \times 3} & \mathbf{0}_{3 \times 3} \\
\mathbf{0}_{3 \times 3} & s_2 J_g^{(2)}\left(x_4, x_5, x_6\right) & \mathbf{0}_{3 \times 3} \\
\mathbf{0}_{3 \times 3} & \mathbf{0}_{3 \times 3} & s_3 J_g^{(3)}\left(x_7, x_8, x_9\right)
\end{array}\right),
\end{align}
where $\mathbf{0}_{m \times n}$ denotes an $m \times n$ zero matrix and we denote
\begin{align*}
J^{(k_1,k_2,k_3)}(x)=\left(\begin{array}{ccc}0 & -k_3^{\left(a_{\ell}, b_{\ell}, c_{\ell}\right)}(x) & k_2^{\left(a_{\ell}, b_{\ell}, c_{\ell}\right)}(x) \\ k_3^{\left(a_{\ell}, b_{\ell}, c_{\ell}\right)}(x) & 0 & -k_1^{\left(a_{\ell}, b_{\ell}, c_{\ell}\right)}(x) \\ -k_2^{\left(a_{\ell}, b_{\ell}, c_{\ell}\right)}(x) & k_1^{\left(a_{\ell}, b_{\ell}, c_{\ell}\right)}(x) & 0\end{array}\right),
\end{align*}
and define
\begin{align*}
& J_g^{(1)}\left(x_1, x_2, x_3\right)=J^{(k_1,k_2,k_3)}(x), \quad x = [x_1~x_2~x_3]^{\top}, 
\\
&J_g^{(2)}\left(x_4, x_5, x_6\right)=J^{(k_1,k_2,k_3)}(x), \quad x = [x_4~x_5~x_6]^{\top},
\\
&J_g^{(3)}\left(x_7, x_8, x_9\right)=J^{(k_1,k_2,k_3)}(x), \quad x = [x_7~x_8~x_9]^{\top}.
\end{align*}

In the experiment of using Telescope dataset  with a centered ball constraint set, we take the constraint set as the centered ball $K_{r=2} \subset\mathbb{R}^9$ defined in~\eqref{eqn:radius}, and we set the step size $\eta = 10^{-4}$ and batch size $m=30$ and implement  SRNSGLD and PSGLD $1000$ iterates with $100$ samples over the training set, where we choose the skew-symmetric matrix as $J_a$ defined in~\eqref{eqn:skew:matrix} with $a=2$ for SRNSGLD ($J_a$) shown by the blue line and the skew-symmetric matrix as $J_s(x)$ defined in~\eqref{eqn:sym:matrix} with $s=[5,5,5]$ for SRNSGLD ($J_s(x)$) shown by the orange dash line. Figure~\ref{tele} reports the accuracy of two algorithms over the training set and the test set where the test set accounts for $20\%$ of the whole dataset.
\begin{figure}[htbp]
    \centering
    \begin{subfigure}{0.45\textwidth}
        \centering
        \includegraphics[width=\linewidth]{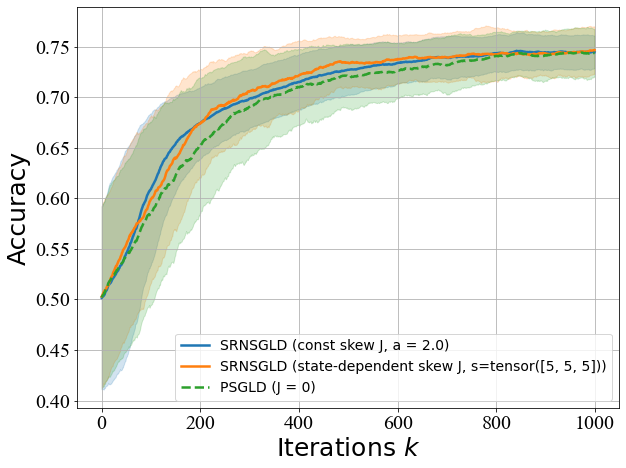}
        \caption{Accuracy over the training set with a centered ball constraint.}
    \end{subfigure}%
    \hspace{0.8cm}
    \begin{subfigure}{0.45\textwidth}
        \centering
        \includegraphics[width=\linewidth]{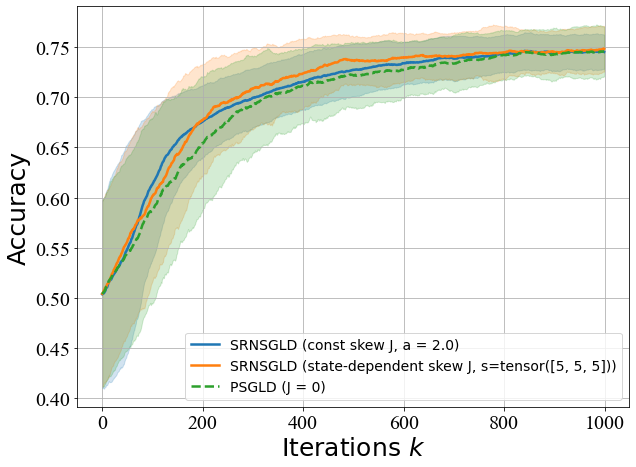}
        \caption{Accuracy over the test set with a centered ball constraint.}
    \end{subfigure}%
    \caption{Accuracy over the training set and the test set for the Telescope dataset. The blue part (solid line) denotes the mean and standard deviation of SRNSGLD ($J_a$), the orange part (solid line) denotes the mean and standard deviation of SRNSGLD ($J_s(x)$) and the green part (dashed line) denotes the mean and standard deviation of PSGLD ($J\equiv 0$).}
    \label{tele}    
\end{figure}

\begin{figure}[htbp]
    \centering
    \begin{subfigure}{0.45\textwidth}
        \centering
        \includegraphics[width=\linewidth]{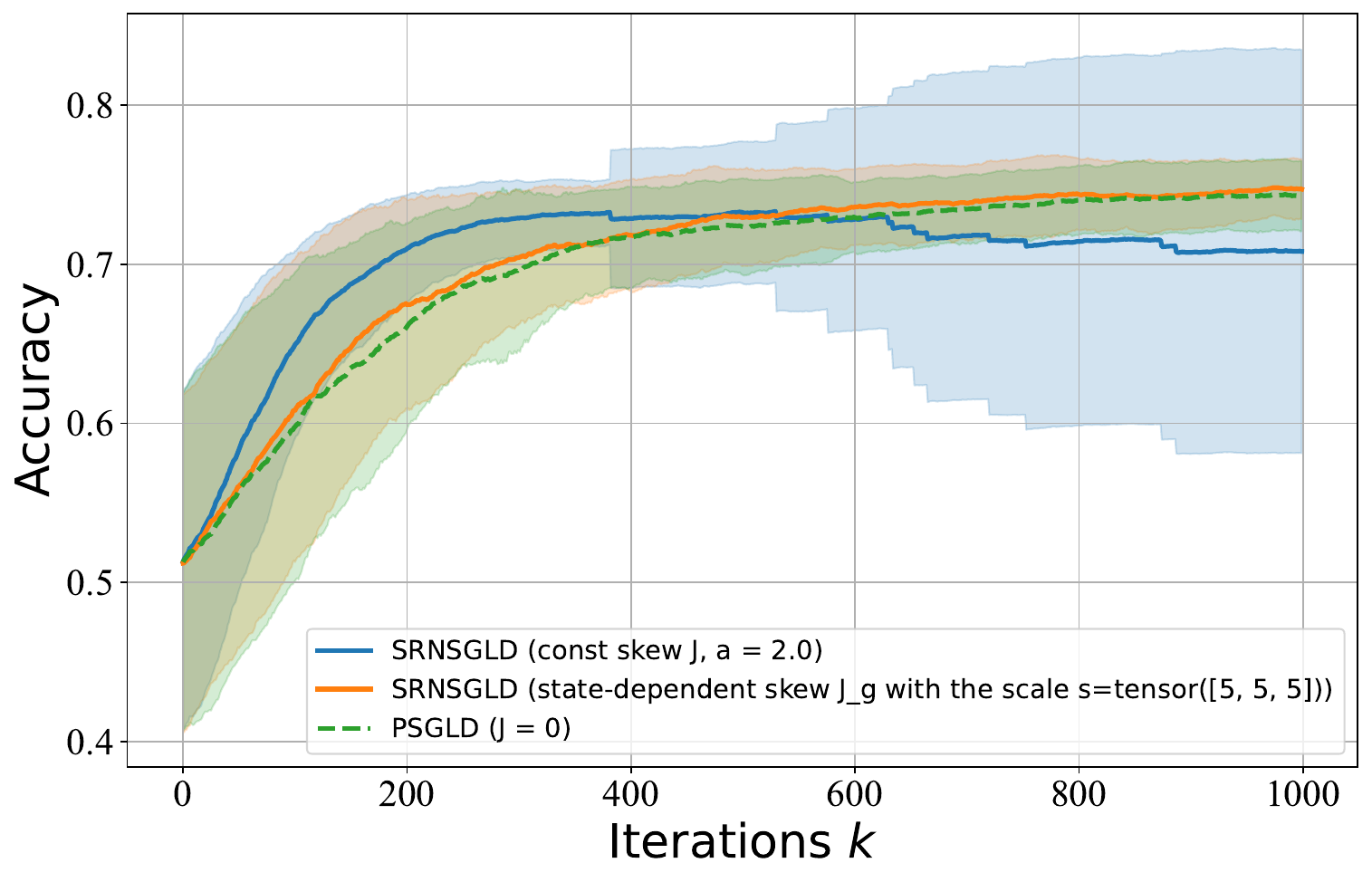}
        \caption{Accuracy over the training set with a smoothed $\ell_p$ ball constraint with a sublevel-set.}
    \end{subfigure}%
    \hspace{0.8cm}
    \begin{subfigure}{0.45\textwidth}
        \centering
        \includegraphics[width=\linewidth]{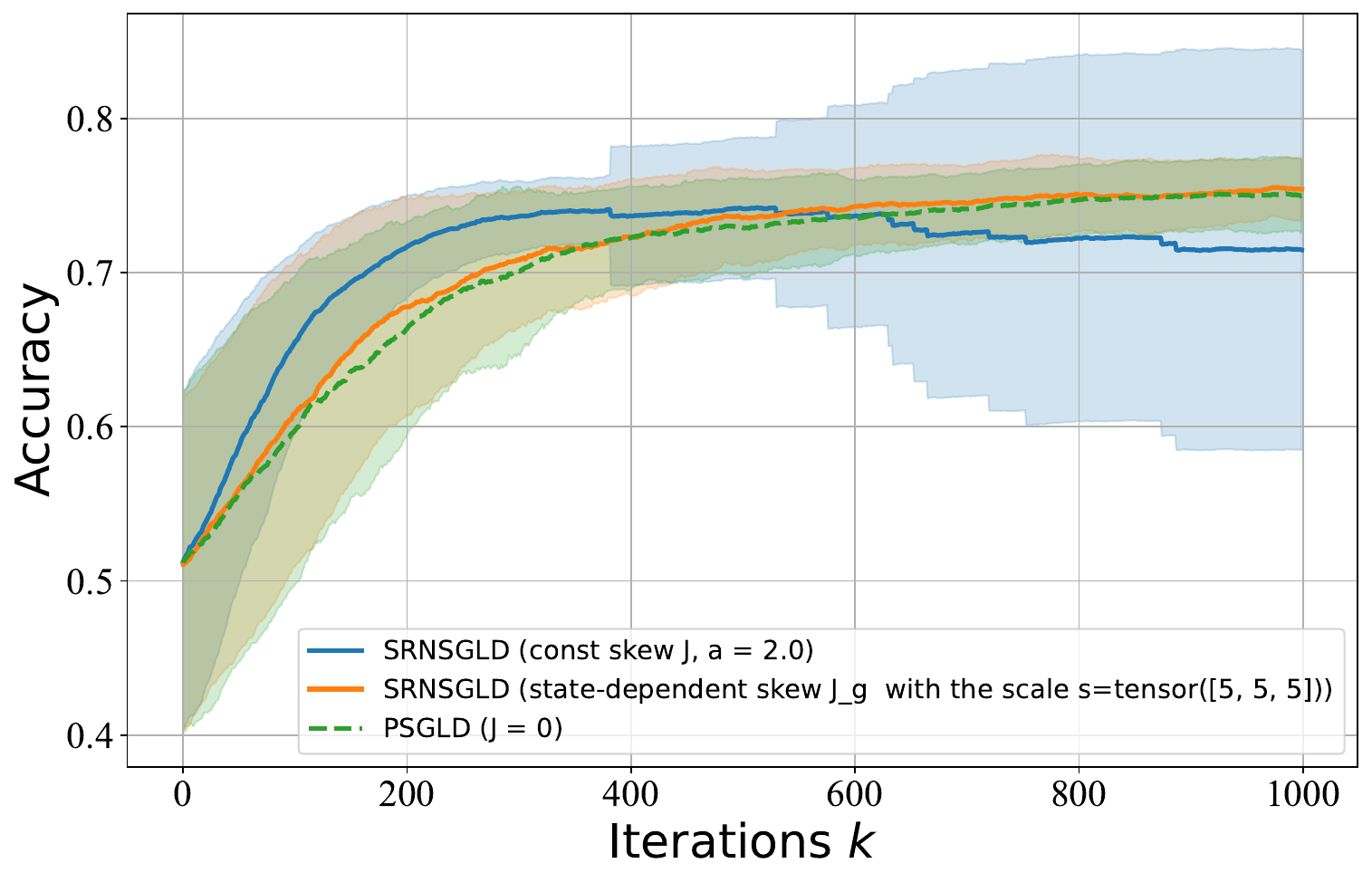}
        \caption{Accuracy over the test set with a smoothed $\ell_p$ ball constraint with a sublevel-set.}
    \end{subfigure}%
    \caption{Accuracy over the training set and the test set for the Telescope dataset. The blue part (solid line) denotes the mean and standard deviation of SRNSGLD ($J_a$), the orange part (solid line) denotes the mean and standard deviation of SRNSGLD ($J_g(x)$) and the green part (dashed line) denotes the mean and standard deviation of PSGLD ($J\equiv 0$).}
    \label{tele2}    
\end{figure}

In the experiment of using Telescope dataset with a smoothed $\ell_p$ ball constraint with a sublevel-set defined in~\eqref{eq:sublevel} with $p = 2.4$ in this Telescope dataset example, $\varepsilon = 0.2$, $\lambda = 4.0$ 
and the scaling factor of $J_g(x)$ is $s = [5,5,5]$ in the definition~\eqref{eq:block:matrix:scale}.
then we implement algorithms with the same implementation parameters as the one used in the example with the centered ball constraint set to obtain Figure~\ref{tele2}.

For the Titanic data, we first consider the constraint set as the centered ball $K_{r=2} \subset\mathbb{R}^9$, and we set the step size $\eta = 10^{-4}$ and batch size $m=30$ and implement SRNSGLD and PSGLD $1500$ iterates over the training set with $100$ samples, we choose the skew-symmetric matrix as $J_a$ defined in~\eqref{eqn:skew:matrix} with $a=3$ for SRNSGLD ($J_a$) shown by the green line and the skew-symmetric matrix as $J_s(x)$ defined in~\eqref{eqn:sym:matrix} with $s=[2,7,2]$ for SRNSGLD ($J_s(x)$) shown by the orange dash line. Figure~\ref{Titanic} reports the accuracy level of algorithms over the training and test sets, where the test set accounts for $20\%$ of the whole dataset. 

\begin{figure}[htbp]
    \centering
    \begin{subfigure}{0.45\textwidth}
        \centering
        \includegraphics[width=\linewidth]{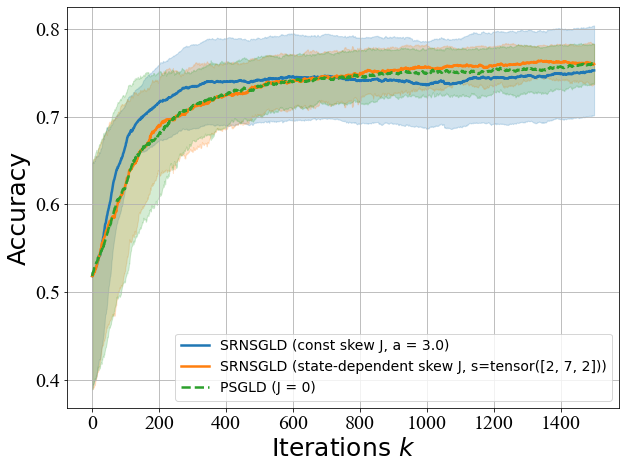}
        \caption{Accuracy over the training set with a centered ball constraint.}
    \end{subfigure}%
    \hspace{0.8cm}
    \begin{subfigure}{0.45\textwidth}
        \centering
        \includegraphics[width=\linewidth]{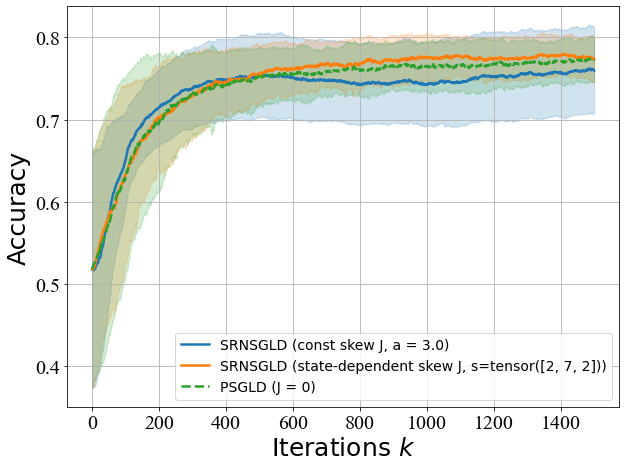}
        \caption{Accuracy over the test set with a centered ball constraint.}
    \end{subfigure}%
    \caption{Accuracy over the training set and the test set for the Titanic dataset. The green part (dash line) denotes the mean and standard deviation of SRNSGLD ($J_a$), the orange part (dash-dot line) denotes the mean and standard deviation of SRNSGLD ($J_s(x)$) and the blue part (solid line) denotes the mean and standard deviation of PSGLD ($J\equiv 0$).}
    \label{Titanic}    
\end{figure}

 Taking the constraint set as the smoothed $\ell_p$ ball constraint with a sublevel-set defined in~\eqref{eq:sublevel} with $p = 2.4$ in this Telescope dataset example, $\varepsilon = 0.18$, $\lambda = 4.0$ 
and the scaling factor of $J_g(x)$ is $s = [2,7,2]$ in the definition~\eqref{eq:block:matrix:scale}, 
then we use the same implementation parameters with $2000$ steps as for the centered ball constraint set and obtain Figure~\ref{Titanic2}.

\begin{figure}[htbp]
    \centering
    \begin{subfigure}{0.45\textwidth}
        \centering
        \includegraphics[width=\linewidth]{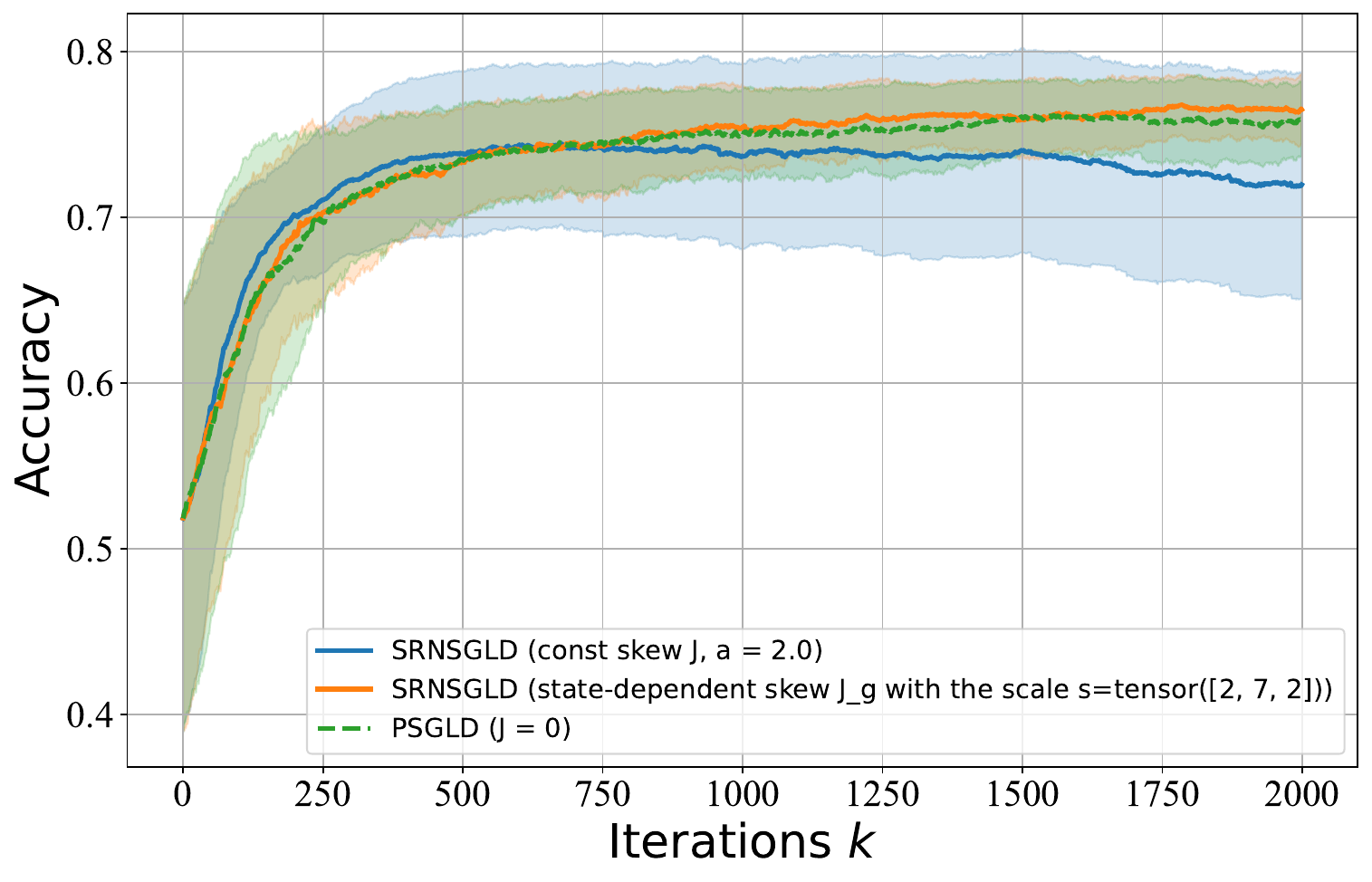}
        \caption{Accuracy over the training set with a smoothed $\ell_p$ ball constraint with a sublevel-set.}
    \end{subfigure}%
    \hspace{0.8cm}
    \begin{subfigure}{0.45\textwidth}
        \centering
        \includegraphics[width=\linewidth]{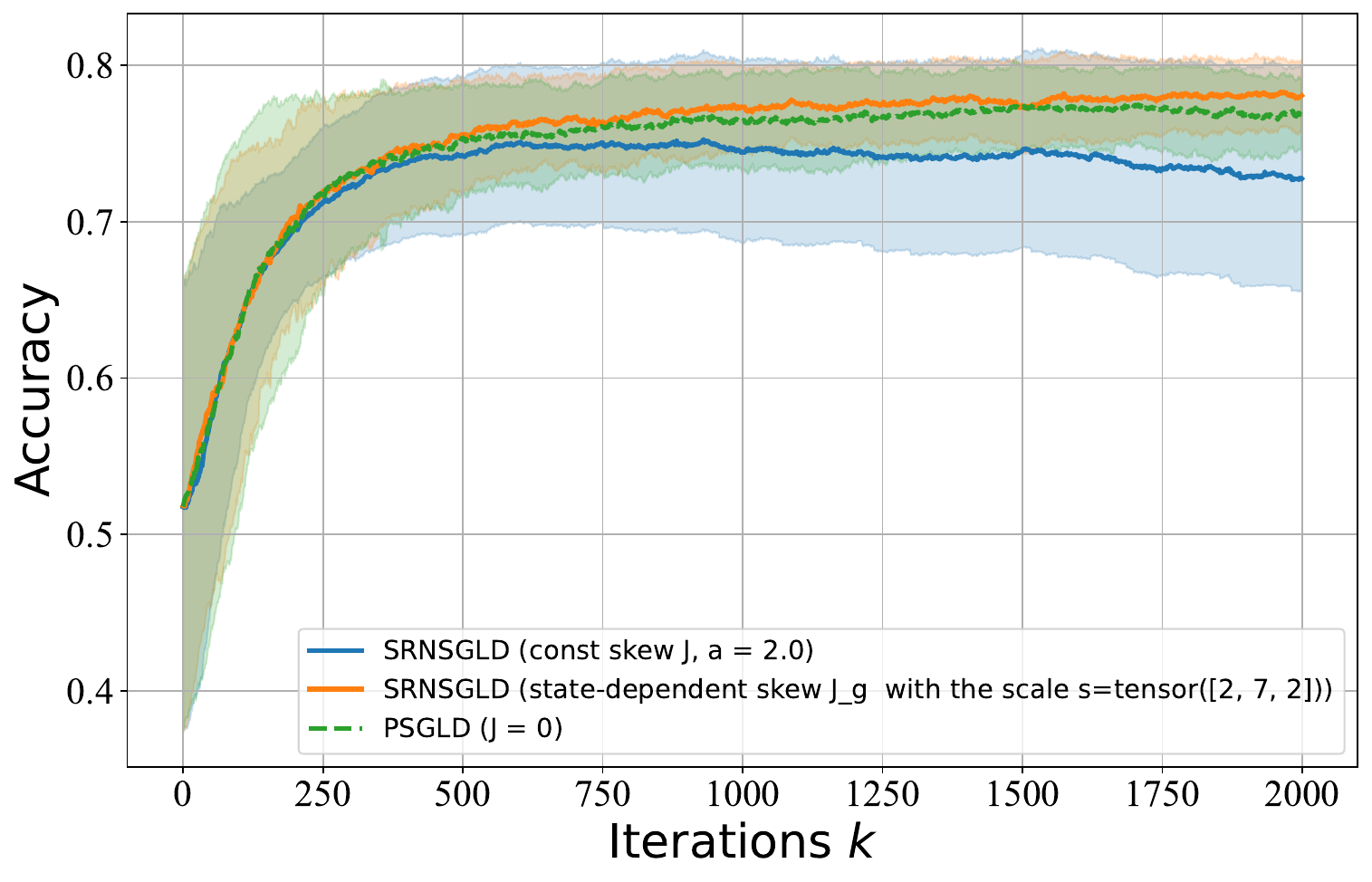}
        \caption{Accuracy over the test set with a smoothed $\ell_p$ ball constraint with a sublevel-set.}
    \end{subfigure}%
    \caption{Accuracy over the training set and the test set for the Titanic dataset. The green part (dash line) denotes the mean and standard deviation of SRNSGLD ($J_a$), the orange part (dash-dot line) denotes the mean and standard deviation of SRNSGLD ($J_g(x)$) and the blue part (solid line) denotes the mean and standard deviation of PSGLD ($J\equiv 0$).}
    \label{Titanic2}    
\end{figure}

 From Figure~\ref{tele} and Figure~\ref{tele2} for Telescope dataset, Figure~\ref{Titanic} and Figure~\ref{Titanic2} for Titanic dataset, we can conclude that both SRNSGLD ($J_s(x)$) and SRNSGLD ($J_g(x)$) consistently outperforms SRNSGLD ($J_a$) in terms of classification accuracy in the Telescope and Titanic datasets,   and the performance of both SRNSGLD ($J_s(x)$) and SRNSGLD ($J_g(x)$) is slightly better than the performance of PSGLD. In particular, SRNSGLD ($J_s(x)$)  and SRNSGLD ($J_g(x)$) attain higher accuracy with fewer iterates, which indicates a faster convergence rate. This improvement can be attributed to the adaptive nature and better acceleration rate of the state-dependent skew-symmetric matrix $J_s(x)$ and $J_g(x)$, as it enables more effective exploration of the parameter space compared to the fixed skew-symmetric matrix $J_a$.  In particular, we can observe from Figure~\ref{tele2} and Figure~\ref{Titanic2} the performance of SRNSGLD ($J_a$) degrades substantially with the smoothed $\ell_p$ ball constraint with the sublevel-set.

In summary, these results demonstrate that incorporating a well-chosen state-dependent matrix into the SRNSGLD framework not only leads to theoretical benefits, but also translates into practical efficiency and better empirical performance across different datasets.
%\vspace{-0.1in}
%%%%%%%%%%%%%%%%%%%%%%%%%%%%%%%%%%%%%%%%%
\section{Concluding Remarks}\label{sec:conclude}

We established a large deviation principle for empirical measures of skew-reflected non-reversible Langevin dynamics (SRNLD), where a skew-symmetric matrix $J$ with $\nabla\!\cdot J=0$ is added to the (reversible) reflected Langevin dynamics (RLD), that samples a target distribution on a constrained domain $K$. We choose $J$ such that its product with the outward unit normal vector field on the boundary is zero ($J(x)\mathbf{n}(x)=0$ on $\partial K$). This simplifies the oblique boundary condition to the standard Neumann condition, which helps overcome the technical challenges that arise from the large deviations theory.

As a by-product, this simplification also provides us a natural choice
for the skew-symmetric matrix, that is constructed explicitly for examples including the ball constraint, 
and more generally the constrained domains characterized by sublevel sets that include special cases such as the smoothed $\ell_p$ ball constraint. 
The compactness of $K$ also simplifies several technical aspects, particularly related to the existence of $\lambda(g)$, properties of the Feynman-Kac semigroup, and exponential tightness. The core results on the variational representation and decomposition of the rate function $I(\nu)$ hold, providing valuable insight into the contributions of the reversible ($\calL_S$) and non-reversible ($\calL_A$) parts of the dynamics to the likelihood of fluctuations. The explicit forms for $I_S(\nu)$ and $I_A(\nu)$ depend on the precise definitions of the Sobolev (semi)norms and the properties of $\calL_A$ under the invariant measure $\mu$.

By explicitly characterizing the rate functions, we show that SRNLD accelerates the convergence to the target distribution when compared with RLD. Furthermore, we provide an operator-theoretic proof that these dynamics also reduce the asymptotic variance of time-average estimators, solidifying the advantages of the non-reversible approach. 

Numerical experiments based on our choice of the skew-symmetric matrix, that is constructed to satisfy all assumptions in our theory for the ball constraint and the smoothed $\ell_{p}$ ball constraint applying to a toy example of truncated multivariate normal distribution, constrained Bayesian linear regression and constrained Bayesian logistic regression using synthetic and real data, show superior performance compared to the existing literature, that validate the theoretical findings from large deviations theory.

There are several promising directions for future research.
First, while Theorem \ref{thm:variance_reduction} provides a spectral characterization of the variance reduction, deriving an explicit \emph{quantitative lower bound} depending solely on the geometry of $K$ and the potential $f$ remains a challenging open problem.
Second, building on the connection to the Cram\'{e}r-Rao lower bound established in Section \ref{subsec:quantitative_variance}, future work could focus on the theoretical optimization of the matrix field $J(x)$ to maximize the rate function curvature, thereby minimizing the asymptotic variance.
These will be left for future studies.

%\vspace{-0.1in}
%%%%%%%%%%%%%%%%%%%%%%%%%%%%%%%%%%%%
\section*{Acknowledgments}
We would like to thank Hengrong Du and Qi Feng for helpful discussions.
Xiaoyu Wang and Yingli Wang are supported by the Guangzhou-HKUST(GZ) Joint Funding Program (No.2025A03J3556).
Lingjiong Zhu is partially supported by the grants NSF DMS-2053454 and NSF DMS-2208303.

\clearpage

%%%%%%%%%%%%%%%%%%%%%%%%%%%%%%%%%%%%%%%%%%%%%%%%
\appendix
\section{Summary of Notations}\label{sec:notations}
%For clarity, we summarize the main notations used throughout this paper:

\subsection{Definitions and Terminologies}

\begin{tabular}{p{0.2\textwidth}p{0.75\textwidth}}
\toprule
Term & Description \\
\midrule
$C^k(K)$ & The space of functions defined on $K$ that are $k$ times continuously differentiable. \\
\addlinespace
$C^{k,\alpha}(K)$ & The H\"{o}lder space consisting of functions $u \in C^k(K)$ whose $k$-th order partial derivatives are H\"{o}lder continuous with exponent $\alpha \in (0,1]$. The associated norm is
\begin{footnotesize}
\[ \|u\|_{C^{k,\alpha}(K)} := \sum_{|\beta| \le k} \sup_{x \in K} |D^\beta u(x)| + \sum_{|\beta| = k} \sup_{x,y \in K, x \neq y} \frac{|D^\beta u(x) - D^\beta u(y)|}{|x-y|^\alpha}. \]
\end{footnotesize} \\
\addlinespace
$\partial K \in C^{k,\alpha}$ & The boundary $\partial K$ is of class $C^{k,\alpha}$ if locally it is the graph of a $C^{k,\alpha}$ function. Specifically, for every $x_0 \in \partial K$, there exists a neighborhood $U$ and a bijective map $\psi: U \to B_1(0)$ such that $\psi, \psi^{-1} \in C^{k,\alpha}$, with $\psi(U \cap K) = \{y \in B_1(0) : y_d \ge 0\}$. \\
\addlinespace
Gateaux-differentiable & For a function $\lambda$, means that for any $g, h \in \CK$, the limit 
\begin{footnotesize}
\[
D\lambda(g)[h] := \lim_{\varepsilon \to 0} \frac{\lambda(g+\varepsilon h) - \lambda(g)}{\varepsilon}
\]
\end{footnotesize}
exists. \\
\midrule
$H^1(K,d\nu)$ & Sobolev space consisting of functions $u$ defined on domain $K$ such that both $u$ and its first-order weak derivatives are square-integrable with respect to measure $\nu$. More formally, $u \in H^1(K,d\nu)$ if $u \in L^2(K,d\nu)$ and its gradient $\nabla u$ (where derivatives are understood in the weak sense) satisfies $\nabla u \in (L^2(K,d\nu))^d$. The squared norm is
\begin{footnotesize}
\[
    \|u\|_{H^1(K,d\nu)}^2 = \int_K |u(x)|^2 \nu(dx) + \int_K \|\nabla u(x)\|_{2}^2 \nu(dx).
\]
\end{footnotesize}
Weak derivatives generalize derivatives to functions that may not be continuously differentiable but still exhibit regularity. \\
$\mathcal{H}^1_S(\nu)$ & A Sobolev $H^1$-type semi-norm with respect to measure $\nu$, defined for a sufficiently regular function $\varphi$ as
\begin{footnotesize}
\[
    \|\varphi\|_{\mathcal{H}^1_S(\nu)}^2 := \int_K \|\nabla \varphi\|_{2}^2 d\nu.
\]
\end{footnotesize}
This semi-norm is associated with the symmetric part of the infinitesimal generator $\mathcal{L}_S$. \\
$\mathcal{H}^{-1}_S(\nu)$ & The dual Sobolev semi-norm corresponding to $\|\cdot\|_{\mathcal{H}^1_S(\nu)}$. For an element $\varphi \in L^2(K, d\nu)$ that has zero mean with respect to $\nu$ (i.e., $\int_K \varphi \, d\nu = 0$), it is defined via a variational problem:
\begin{footnotesize}
\[
\|\varphi\|_{\mathcal{H}_S^{-1}(\nu)}^2 = \sup_{\psi\in D^+(\mathcal L_J)}\left\{ 2\int_K \psi\varphi d\nu - \|\psi\|_{\mathcal{H}^1_S(\nu)}^2 \right\}.
\]
\end{footnotesize} \\
\bottomrule
\end{tabular}

%%%%%%%%%%%%%%%%%%%%%%%%%%%
\subsection{Sets and Spaces}

\begin{tabular}{ll}
\toprule
Notation & Description \\
\midrule
$\R^d$ & $d$-dimensional Euclidean space \\
$K$ & Compact, connected domain with smooth boundary $\partial K$ \\
$\CK$ & Space of continuous functions on $K$ \\
$\CbK$ & Space of bounded continuous functions on $K$ \\
$L^\infty(K)$ & Space of measurable functions on $K$ that are essentially bounded. \\
$\calP(K)$ & Space of probability measures on $K$ \\
$\DomLJ$ & Domain of $\calL_J$, which is $\{ u \in C^2(K) : \inner{\nabla u(x)}{(I + J(x))\mathbf n(x)} = 0 \text{ on } \partial K \}$ \\
$\DomLplus$ & Space of positive functions $u$ in $\DomLJ \cap \CK$ with $\frac{\calL_J u}{u} \in \CbK$ \\
\bottomrule
\end{tabular}

%%%%%%%%%%%%%%%%%%%%%%%%
\subsection{Operators and Functions}

\begin{tabular}{ll}
\toprule
Notation & Description \\
\midrule
$\calL_J$ & Infinitesimal generator of SRNLD \\
$\calL_S$ & Symmetric part of $\calL_J$ \\
$\calL_A$ & Skew-symmetric part of $\calL_J$ \\
$\nabla$ & Gradient operator \\
$\nabla \cdot$ & Divergence operator\\
$\Delta$ & Laplacian operator \\
$f$ & Potential function \\
$J$ & Skew-symmetric matrix field \\
$\mathbf n$ & outward unit normal vector on $\partial K$ \\
\bottomrule
\end{tabular}

\subsection{Measures and Norms}

\begin{tabular}{ll} 
\toprule
Notation & Description \\
\midrule
$\mu$ & Invariant measure of SRNLD \\
$\nu$ & A realization of the empirical measure of $X_t$\\
$\|\cdot\|_2$ & Euclidean 2-norm \\
$\|\cdot\|_\infty$ & Supremum norm \\
$\|\cdot\|_{\mathcal{H}^1_S(\nu)}$ & Sobolev semi-norm with respect to $\nu$ \\
$\|\cdot\|_{\mathcal{H}^{-1}_S(\nu)}$ & Dual Sobolev semi-norm \\
$\mathcal{W}_1$ & 1-Wasserstein distance \\
\bottomrule
\end{tabular}

\subsection{Other Notations}

\begin{tabular}{ll}
\toprule
Notation & Description \\
\midrule
$\inner{\cdot}{\cdot}$ & Inner product in the Euclidean space $\mathbb R^d$\\
$\langle\cdot,\cdot\rangle_\mu$ & Inner product in $L^2(K,d\mu)$\\
$\calL_J^*$ & Adjoint of operator $\calL_J$ \\
$I(\nu)$ & Rate function of the empirical measure $\nu$\\
$\lambda(g)$ & Scaled cumulant generating function \\
\bottomrule
\end{tabular}

%\subsection{Abbreviations}
%
%\begin{tabular}{ll}
%\toprule
%Abbreviation & Full Name \\
%\midrule
%SRNLD & Skew-Reflected Non-reversible Langevin Dynamics \\
%SRNLMC & Skew-Reflected Non-reversible Langevin Monte Carlo \\
%PLMC & Projected Langevin Monte Carlo \\
%LDP & Large Deviation Principle \\
%\bottomrule
%\end{tabular}

\section{Proofs of Technical Lemmas}
\label{app:proofs_of_technical_lemmas}

\subsection{Proof of Lemma~\ref{existence}}
\label{app:existence}
        \begin{enumerate}
            \item \textbf{Properties of the operator $\calL_J+g$ and the semigroup $P_t^g$:}
            Since $K$ is a compact domain in $\R^d$ and $g \in C(K)$, $g$ is bounded on $K$. Let $g_{\max} = \max_{x \in K} |g(x)|$.
            The operator $\calL_J$ of $(X_{t})_{t\geq 0}$ in \eqref{eq:sde_local} is given by:
            $\calL_J u(x) = \Delta u(x) - \inner{(I + J(x))\nabla f(x)}{\nabla u(x)}$.
            The term $\Delta u(x)$ is uniformly elliptic. 
            Assume $\partial K\in C^{2,\alpha}$, $f\in C^{2,\alpha}(K)$ and $J\in C^{1,\alpha}(K)$. 
            Then the drift coefficients $(I+J(x))\nabla f(x)$ belong to $C^{1,\alpha}(K)$, and hence are H\"older continuous. 
            Therefore $\mathcal L_J$ is a second-order linear uniformly elliptic operator on $K$ with $C^{1,\alpha}$ coefficients. 
            Moreover, multiplication by $g(x)$ is a bounded (zeroth-order) perturbation (e.g. $g\in L^\infty(K)$).
            The domain $D(\calL_J+g)$ is specified by functions in $C^2(K)$ satisfying Neumann boundary conditions 
            $\inner{\nabla u(x)}{\mathbf n(x)} = 0$
            on $\partial K$.
        
            \begin{enumerate}[label=\alph*)]
                \item \textbf{Compactness of $P_t^g$:} The semigroup $P_t^g$ is generated by $\calL_J+g$. For an elliptic operator on a compact manifold with boundary conditions such as Neumann, the resolvent 
                $(\zeta I - (\calL_J+g))^{-1}$
                is compact for $\zeta$ in the resolvent set. This implies that the semigroup $P_t^g = e^{t(\calL_J+g)}$ is compact for all $t>0$ (see, e.g., \cite[Chapter 2, Corollary 3.5]{pazy1983semigroups}).
                Alternatively, $P_t^g$ can be represented via the fundamental solution (Green's function) $p^g(t,x,y)$ of the parabolic PDE 
                $\partial_t u = (\calL_J+g)u$.
                For elliptic operators, this fundamental solution is smooth for $t>0$ (see, e.g., \cite[Chapter 3.3, Theorem 3.3.11]{stroock2008partial}). Since $K$ is compact, 
                $(P_t^g \phi)(x) = \int_K p^g(t,x,y) \phi(y)\, dy$
                is an integral operator with a smooth kernel, which implies that $P_t^g$ is compact on $C(K)$ (also on $L^p(K)$).
        
                \item \textbf{Strong positivity of $P_t^g$.}
                
                Recall that an operator $T$ on $C(K)$ is \emph{strongly positive} if for any $\phi\in C(K)$ with
                $\phi\ge 0$ and $\phi\not\equiv 0$, one has $T\phi(x)>0$ for all $x\in K$.
                Let $(P_t)_{t\ge 0}$ be the Markov semigroup associated with the reflected diffusion generated by
                $\calL_J$ on the bounded connected domain $K$.
                Since the diffusion matrix is the identity, $\calL_J$ is uniformly elliptic; in particular the
                Neumann heat kernel $p(t,x,y)$ exists, is jointly continuous on $(0,\infty)\times K\times K$, and is
                \emph{strictly positive} for every $t>0$ and $x,y\in K$ (see, e.g., \cite[p.~98]{DaviesHeatKernels}).
                Consequently, for any $\phi\ge 0$, $\phi\not\equiv 0$,
                \[
                (P_t\phi)(x)=\int_K p(t,x,y)\phi(y)\,dy>0,\qquad x\in K,
                \]
                so $P_t$ is positivity improving; equivalently, the semigroup is irreducible in the sense of
                \cite[Theorem~3.3.5]{DaviesHeatKernels}.

                Now let $g\in L^\infty(K)$ and define the Feynman--Kac semigroup
                \[
                (P_t^g\phi)(x)
                :=\E^x\!\left[\phi(X_t)\exp\!\left(\int_0^t g(X_s)\,ds\right)\right].
                \]
                The exponential weight is strictly positive and satisfies
                $$e^{-t\|g\|_\infty}\le \exp\left(\int_0^t g(X_s)\,ds\right)\le e^{t\|g\|_\infty}.$$
                Hence, for $\phi\ge 0$,
                \[
                (P_t^g\phi)(x)\;\ge\; e^{-t\|g\|_\infty}\,\E^x[\phi(X_t)]
                \;=\; e^{-t\|g\|_\infty}(P_t\phi)(x)\;>\;0,\qquad x\in K,
                \]
                which proves that $P_t^g$ is strongly positive for every $t>0$.
                
            \end{enumerate}
        
            \item \textbf{Existence of a principal eigenvalue/eigenfunction and $t$-consistency:}
            
            Fix some $t_0>0$. Since $P_{t_0}^g$ is compact and strongly positive on the Banach lattice $C(K)$,
            the Krein--Rutman theorem (see, e.g., \cite[Chapter~6, Theorem~19.2]{deimling2013nonlinear})
            yields that the spectral radius $r(P_{t_0}^g)>0$ is a simple eigenvalue of $P_{t_0}^g$ and there
            exists $h_g\in C(K)$ with $h_g>0$ on $K$ such that
            \[
            P_{t_0}^g h_g = r(P_{t_0}^g)\,h_g .
            \]
            Write $r(P_{t_0}^g)=e^{t_0\lambda(g)}$ and define $\lambda(g):=\frac1{t_0}\log r(P_{t_0}^g)$.

            \smallskip
            \noindent\emph{Step 1: the same $h_g$ works for all $t\ge0$.}
            For any $t\ge0$, by the semigroup property,
            \[
            P_{t_0}^g\left(P_t^g h_g\right)=P_t^g\left(P_{t_0}^g h_g\right)
            =r(P_{t_0}^g)\,P_t^g h_g .
            \]
            Thus $P_t^g h_g$ is a (strictly positive) eigenfunction of $P_{t_0}^g$ associated with the simple
            eigenvalue $r(P_{t_0}^g)$. By uniqueness up to scaling, there exists a scalar $c(t)>0$ such that
            \[
            P_t^g h_g = c(t)\,h_g \qquad (t\ge0).
            \]
            Again by the semigroup property, $c(t+s)=c(t)c(s)$ and $c(0)=1$.
            Since $(P_t^g)_{t\ge0}$ is strongly continuous on $C(K)$ (as $X_t$ is Feller and $g$ is bounded),
            $t\mapsto P_t^g h_g$ is continuous in $C(K)$. Hence $c(t)$ is continuous, and $c(t)=e^{t\lambda(g)}$ for all $t\ge0$.
            Therefore,
            \[
            P_t^g h_g = e^{t\lambda(g)}h_g,\qquad \text{for any $t\ge0$}.
            \]
            In particular, for every fixed $t>0$, the principal eigenfunction of $P_t^g$ is unique up to normalization;
            thus it coincides with $h_g$ after rescaling and the corresponding exponent satisfies $\lambda_t(g)=\lambda(g)$.

            \smallskip
            \noindent\emph{Step 2: eigenvalue problem for the generator.}
            Since $t_0>0$ and $\mathcal L_J+g$ is uniformly elliptic with smooth coefficients and Neumann boundary condition,
            the kernel $p^g(t_0,x,y)$ is smooth in $(x,y)$.
            Hence $P_{t_0}^g$ is smoothing, and in particular $P_{t_0}^g h_g\in D(\mathcal L_J+g)$
            (see, e.g., \cite[Chapter~3.3, Theorem~3.3.11]{stroock2008partial}).
            From $P_{t_0}^g h_g=e^{t_0\lambda(g)}h_g$ we deduce $h_g\in D(\mathcal L_J+g)$. From $P_t^g h_g=e^{t\lambda(g)}h_g$ and the definition of the generator,
            \[
            (\mathcal L_J+g)h_g
            =\lim_{t\downarrow 0}\frac{P_t^g h_g-h_g}{t}
            =\lim_{t\downarrow 0}\frac{e^{t\lambda(g)}-1}{t}\,h_g
            =\lambda(g)\,h_g,
            \]
            so that $h_g$ is a (strictly positive) eigenfunction of $\mathcal L_J+g$ with eigenvalue $\lambda(g)$.
            %\textit{p. 41, statement (a): Why does the eigenfunction $h_g$ not depend on $t$? This is not adequately addressed here. This is importantly used in part 3 below. Also, in the bold heading, $\lambda(g)$ is independent of $t$ should be $\lambda_t(g)$ is independent of $t$.}

            \item \textbf{Existence of the limit and independence of $x$:}
            From the definition of the principal eigenvalue, for the strictly positive eigenfunction $h_g$, we have
            $(P_t^g h_g)(x) = e^{\lambda(g)t} h_g(x)$.
            This can be written as 
            $\E^x \left[ h_g(X_t) \exp\left( \int_0^t g(X_s) ds \right) \right] = e^{\lambda(g)t} h_g(x)$.
            Thus, 
            $\frac{1}{t} \log \E^x \left[ \frac{h_g(X_t)}{h_g(x)} \exp\left( \int_0^t g(X_s) ds \right) \right] = \lambda(g)$.
            Since $K$ is compact and $h_g \in C(K)$ with $h_g(x) > 0$, there exist $0 < c_1 \le c_2 < \infty$ such that $c_1 \le h_g(x) \le c_2$ for all $x \in K$.
            Therefore, 
            $\frac{c_1}{c_2} \le \frac{h_g(X_t)}{h_g(x)} \le \frac{c_2}{c_1}$.
            Then,
            %\begin{footnotesize}
        \begin{equation}\label{t:to:infty:2}
            \lambda(g) + \frac{1}{t}\log\left(\frac{c_1}{c_2}\right) \le \frac{1}{t} \log \E^x \left[ \exp\left( \int_0^t g(X_s) ds \right) \right] \le \lambda(g) + \frac{1}{t}\log\left(\frac{c_2}{c_1}\right). 
            \end{equation}
            %\end{footnotesize}
            Taking the limit as $t \to \infty$ in \eqref{t:to:infty:2}, we get
            $\lim_{t\to\infty} \frac{1}{t} \log \E^x \left[ \exp\left( \int_0^t g(X_s) ds \right) \right] = \lambda(g)$.
            This limit exists, is equal to the principal eigenvalue of $\calL_J+g$, and is independent of the initial state $x \in K$.
            The finiteness of $\lambda(g)$ follows because $g$ is bounded on the compact set $K$. If $g_{\min} \le g(x) \le g_{\max}$, then
            %\begin{footnotesize}
            \begin{align*}
                g_{\min} &= \lim_{t\to\infty} \frac{1}{t} \log \E^x \left[ \exp\left( \int_0^t g_{\min} ds \right) \right] 
                \\
                &\le \lambda(g) \le \lim_{t\to\infty} \frac{1}{t} \log \E^x \left[ \exp\left( \int_0^t g_{\max} ds \right) \right] = g_{\max}.
            \end{align*}
            %\end{footnotesize}
            Thus, $\lambda(g)$ is finite.
        \end{enumerate}
        The results cited (e.g., Krein-Rutman theorem \cite[Chapter 6, Theorem 19.2]{deimling2013nonlinear},  \cite[Chapter 2, Corollary 3.5]{pazy1983semigroups}) are standard tools for analyzing such semigroups on compact domains. Donsker and Varadhan's original work on LDPs for Markov processes on compact spaces also establishes the existence of $\lambda(g)$ as the principal eigenvalue \citep{DV1}.
        This completes the proof.

\subsection{Proof of Lemma~\ref{convexity}}
\label{app:proof_convexity}
    This is a standard consequence of H\"{o}lder's inequality. For any $\alpha \in [0,1]$, $g_1, g_2 \in C(K)$,
    %\begin{footnotesize}
    \begin{align}
      \E\left[e^{\int_0^t (\alpha g_1(X_s) + (1-\alpha)g_2(X_s)) ds }\right]
      &= \E\left[\left(e^{\int_0^t g_1(X_s) ds }\right)^\alpha \left(e^{\int_0^t g_2(X_s) ds }\right)^{1-\alpha}\right]\nonumber \\
      &\le \left(\E\left[e^{\int_0^t g_1(X_s) ds }\right]\right)^\alpha \left(\E\left[e^{\int_0^t g_2(X_s) ds }\right]\right)^{1-\alpha}.\label{t:to:infty:3}
    \end{align}
    %\end{footnotesize}
    Taking $\frac{1}{t}\log$ and letting $t\to\infty$ in \eqref{t:to:infty:3} yields: 
    $\lambda(\alpha g_1 + (1-\alpha)g_2) \le \alpha \lambda(g_1) + (1-\alpha)\lambda(g_2)$. 
    This completes the proof.

%\vspace{-0.1in}
\subsection{Proof of Lemma~\ref{gateaux}}
\label{app:proof_gateaux}
The proof relies on perturbation theory for linear operators, specifically for isolated simple eigenvalues.
Let $T(g) = \calL_J + M_g$, where $M_g$ is the multiplication operator by $g$.
From the previous proposition (existence of $\lambda(g)$), we know that for each $g \in \CK$, $\lambda(g)$ is a simple, isolated eigenvalue of $T(g)$ with the largest real part.
Let $h_g \in \CK$ be a corresponding strictly positive eigenfunction, and let $h_g^*$ be a corresponding strictly positive eigenfunction of the adjoint operator
$T(g)^* = \calL_J^* + M_g$ (adjoint with respect to the $L^2(K, dx)$ inner product, assuming appropriate domains).
Since $h_g$ and $h_g^*$ are unique up to positive multiplicative constants, we can normalize them such that their $L^2(K, dx)$ inner product is
$\inner{h_g}{h_g^*}_{L^2(K,dx)} = \int_K h_g(x) h_g^*(x) dx = 1$.
Note that since $K$ is compact and $h_g \in \CK$ with $h_g(x)>0$, $h_g$ is bounded from above and below by positive constants on $K$.

\textbf{1. Continuity of $\lambda(g)$:}
We have already established in Proposition~\ref{convexity} that $\lambda(g)$ is convex.
A convex function on a Banach space (such as $\CK$) is continuous if it is bounded on some open set (see, e.g., \cite[Chapter 1, Lemma 2.1]{ekeland1999convex}).
Since $g_{\min} \le \lambda(g) \le g_{\max}$ (where $g_{\min} = \inf_K g(x)$ and $g_{\max}=\sup_K g(x)$), if $g$ is in an open ball of radius $R$ around $0$ in $\CK$, then $\norm{g}_\infty \le R$, which implies $-R \le \lambda(g) \le R$. Thus $\lambda(g)$ is bounded on open sets.
Being convex and bounded on open sets, $\lambda(g)$ is continuous.

\textbf{2. Gateaux-Differentiability of $\lambda(g)$:}
We want to show that for any $g, h \in \CK$, the limit
$D\lambda(g)[h] = \lim_{\varepsilon \to 0} \frac{\lambda(g+\varepsilon h) - \lambda(g)}{\varepsilon}$
exists.
Since $\lambda(g)$ is a simple eigenvalue of $T(g) = \calL_J + M_g$, its first-order variation with respect to a linear perturbation $T(g) + \varepsilon M_h$ is given by standard perturbation theory (see, e.g., \cite[Chapter II, Equation~(2.33)]{kato2013perturbation} for the coefficient of the first-order term in the eigenvalue series expansion).
Let $T(\varepsilon) = T(g+\varepsilon h)$. The eigenvalue is $\lambda(\varepsilon) = \lambda(g+\varepsilon h)$ and the corresponding eigenfunction is $h(\varepsilon) = h_{g+\varepsilon h}$.
The differentiability of $\lambda(\varepsilon)$ and $h(\varepsilon)$ with respect to $\varepsilon$ at $\varepsilon=0$ is assured for such analytic perturbations (see \cite[Chapter VII, Theorem 2.6, Theorem 3.6]{kato2013perturbation}).

The first-order correction to $\lambda(g)$, denoted $\lambda'(0) = \left.\frac{d\lambda(\varepsilon)}{d\varepsilon}\right|_{\varepsilon=0}$, can be found by differentiating the eigenvalue equation $T(\varepsilon)h(\varepsilon) = \lambda(\varepsilon)h(\varepsilon)$ with respect to $\varepsilon$ at $\varepsilon=0$:
%\begin{footnotesize}
\[ M_h h_g + T(g)h'(0) = \lambda'(0)h_g + \lambda(g)h'(0), \]
%\end{footnotesize}
where $h_g = h(0)$ is the eigenfunction for $T(g)$.
Taking the inner product with the eigenfunction $h_g^*$ of the adjoint $T(g)^*$ (normalized so that $\inner{h_g}{h_g^*}_{L^2(K,dx)} = 1$), and using $T(g)^*h_g^* = \lambda(g)h_g^*$, we obtain:
$\inner{M_h h_g}{h_g^*}_{L^2(K,dx)} = \lambda'(0)$.
Thus,
$\lambda'(0) = \int_K h(x) h_g(x) h_g^*(x) dx$.
This shows that the Gateaux derivative $D\lambda(g)[h]$ exists and is given by
$D\lambda(g)[h] = \int_K h(y) \nu_g(dy)$, 
where the measure $\nu_g$ is defined by
\[
\nu_g(dy)= h_g(y)\,h_g^*(y)\,dy,
\qquad\text{with}\qquad\int_K h_g(y)h_g^*(y)\,dy=1.
\]
In particular, $\nu_g$ is a probability measure, consistent with the eigenfunction representation
used in the main text.

The differentiability of $\lambda(g+\varepsilon h)$ with respect to $\varepsilon$ (and in fact, its analyticity for real $\varepsilon$) is a deep result from the perturbation theory of operators. The family of operators $T(g) = \calL_J + M_g$ is an affine (and thus analytic) family of operators of type (A) in the sense of \cite[Chapter VII, Section 2.1]{kato2013perturbation}, when considered as operators from, say, $H^2(K) \cap \{\text{BCs}\}$ to $L^2(K)$. Since $\lambda(g)$ is a simple isolated eigenvalue, it is an analytic function of $g$ (if $\CK$ is considered as a real Banach space, $\lambda(g)$ is real-analytic). Real analyticity implies Fr{\'e}chet differentiability, which in turn implies Gateaux differentiability.
The proof is complete.

%%%%%%%%%%%%%%%%%%%%%%%%%%%%
%\vspace{-0.1in}
\subsection{Proof of Lemma~\ref{Exponential-Tightness}}
\label{app:proof_exponential_tightness}
    By definition (see, e.g., \cite[Definition 1.2.18]{DZ1998}), a family of probability measures $(\mathbb{Q}_t)_{t \ge 0}$ on a Polish space $\mathcal{X}$ is exponentially tight if for every $A < \infty$, there exists a compact set $\Gamma_A \subset \mathcal{X}$ such that
    $\limsup_{t\to\infty} \frac{1}{t} \log \mathbb{Q}_t(\Gamma_A^c) \le -A$.
    In our case, the Polish space is $\mathcal{X} = \calP(K)$, the space of probability measures on $K$, equipped with the weak topology. The measures are $\mathbb{Q}_t(\cdot) = \Prob(L_t \in \cdot)$.
    
    Since $K$ is a compact metric space, by Prokhorov's theorem (see, e.g., \cite[Theorem 5.1 and Theorem 5.2]{billingsley2013convergence}), the space $\calP(K)$ itself is compact in the weak topology.
    Now, to demonstrate exponential tightness, for any given $A < \infty$, we can choose the compact set $\Gamma_A = \calP(K)$.
    Since $L_t$ is, by its definition as an empirical measure, always an element of $\calP(K)$ (it is a probability measure on $K$), the event $L_t \notin \Gamma_A = \calP(K)$ is an impossible event.
    Therefore, $L_t \notin \calP(K)$ implies $L_t \in \emptyset$.
    Thus,
    $\Prob(L_t \notin \Gamma_A) = \Prob(L_t \notin \calP(K)) = \Prob(L_t \in \emptyset) = 0$.
    Then, for any $t > 0$,
    $\frac{1}{t} \log \Prob(L_t \notin \Gamma_A) = \frac{1}{t} \log 0$. 
    The logarithm of zero is formally $-\infty$. So,
    $\limsup_{t\to\infty} \frac{1}{t} \log \Prob(L_t \notin \Gamma_A) = \limsup_{t\to\infty} (-\infty) = -\infty$.
    Since for any $A < \infty$, we have $-\infty \le -A$, the condition for exponential tightness is satisfied.
    
    This argument is significantly simpler than the one required for non-compact state spaces (such as \cite[Lemma 6.9]{LDP-GG}), where one typically needs to construct a family of precompact sets using a Lyapunov function $\Psi$ and show that the probability of the empirical measure lying outside these precompact sets decays sufficiently fast. The compactness of the entire space $\calP(K)$ makes this step trivial. The proof is complete.

%\vspace{-0.1in}
\subsection{Proof of Lemma~\ref{lemma:asymptotic_variance}}
\label{app:proof_asymptotic_variance}
The two formulas are classical and equivalent characterizations of the asymptotic variance for an ergodic Markov process.

(a) The first formula,
$\sigma_{g,J}^2 = 2 \int_0^\infty \E_\mu [g(X_0) g(X_t)] dt$,
is known as the Green-Kubo formula. It is obtained by expanding the definition of the variance of the time-average estimator, $\lim_{T\to\infty} T \cdot \text{Var}_\mu \left[\frac{1}{T}\int_0^T g(X_t) dt\right]$,
and applying the stationarity of the process $(X_t)_{t\ge 0}$. The integrability of the auto-correlation function is guaranteed by the exponential decay of correlations, a consequence of the spectral gap of the process. For a detailed derivation, see, e.g., \cite[Chapter 4]{pavliotis2014stochastic}.

(b) The equivalence to the second formula:
$\sigma_{g,J}^2 = 2 \inner{g}{(-\calL_J)^{-1}g}$,
is established via the operator semigroup theory. The auto-correlation function can be expressed as $\E_\mu[g(X_0) g(X_t)] = \inner{g}{P_t^J g}$, where $P_t^J=e^{t\calL_J}$ is the Markov semigroup generated by $\calL_J$. The Green-Kubo formula then becomes:
$\sigma_{g,J}^2 = 2 \int_0^\infty \inner{g}{P_t^J g} dt = 2 \inner{g}{\left( \int_0^\infty P_t^J dt \right) g}$.
The exchange of the inner product and the integral is justified by Fubini's theorem. The integral of the semigroup is the resolvent of the generator at zero, i.e., $\int_0^\infty P_t^J dt = (-\calL_J)^{-1}$. This identity is fundamental in semigroup theory and connects the integrated time-evolution of the process to the inverse of its generator. The existence and boundedness of $(-\calL_J)^{-1}$ on $L^2_0(\mu)$ is ensured by the spectral gap of $\calL_J$ on the compact domain $K$. This establishes the equivalence.

%%%%%%%%%%%%%%%%%%%%%%%%%%%%%%%%%%%%%%%%%%
%\vspace{-0.1in}
\section{Derivation of Adjoint and Operator Decompositions under Assumption~\ref{assump:nablaJ}}
\label{app:operator_decomposition_simplified}

In this appendix, we derive the explicit expressions for the adjoint operator $\mathcal{L}_J^*$ and the symmetric and skew-symmetric parts, $\mathcal{L}_S $ and $\mathcal{L}_A $, of the infinitesimal generator $\mathcal{L}_J$ with respect to the invariant measure $\mu(dx) = Z^{-1}e^{-f(x)}dx$. These derivations are performed under Assumptions~\ref{assump:regularity_local}, \ref{assump:boundary}, and crucially, Assumption~\ref{assump:nablaJ}. The infinitesimal generator of $(X_{t})_{t\geq 0}$ in \eqref{eq:sde_local} is given by
$\mathcal{L}_J u(x) = \Delta u(x) - \langle \nabla f(x), \nabla u(x) \rangle - \langle J(x)\nabla f(x), \nabla u(x) \rangle$,
and acts on functions $u\in D(\mathcal{L}_J)$, which satisfy Neumann boundary conditions $\langle \nabla u(x), \mathbf n(x) \rangle = 0$ on $\partial K$ due to Assumption~\ref{assump:boundary} (where $\mathbf n(x)$ is the outward unit normal vector).

Let
$\mathcal{O}_1 u := \Delta u - \langle \nabla f, \nabla u \rangle$
and
$\mathcal{O}_2 u := - \langle J\nabla f, \nabla u \rangle$,
so that $\mathcal{L}_J = \mathcal{O}_1 + \mathcal{O}_2$.

%\vspace{-0.1in}
\subsection{Adjoint of \texorpdfstring{$\mathcal{O}_1$}{O1}}
The operator $\mathcal{O}_1 u = \Delta u - \langle \nabla f, \nabla u \rangle$ can be written as $\mathcal{O}_1 u = e^{f(x)} \nabla \cdot (e^{-f(x)} \nabla u(x))$. This is the infinitesimal generator of a reversible Langevin dynamics with invariant measure $\mu$. For $u, v \in D(\mathcal{L}_J)$ (hence satisfying Neumann boundary conditions), $\mathcal{O}_1$ is self-adjoint with respect to $\mu$:
%\begin{footnotesize}
\begin{align*}
    \langle \mathcal{O}_1 u, v \rangle_\mu &= \frac{1}{Z} \int_K \left(e^f \nabla \cdot \left(e^{-f} \nabla u\right)\right) v e^{-f} dx
    = \frac{1}{Z} \int_K \left(\nabla \cdot \left(e^{-f} \nabla u\right)\right) v dx \\
    &= \frac{1}{Z} \left[ -\int_K \left(e^{-f} \nabla u\right) \cdot \nabla v dx + \int_{\partial K} v \left(e^{-f} \nabla u\right) \cdot \mathbf n_{out} dS \right] \\
    &= -\frac{1}{Z} \int_K e^{-f} \left(\nabla u \cdot \nabla v\right) dx \quad (\text{since } \langle \nabla u, \mathbf n_{out} \rangle=0).
\end{align*}
%\end{footnotesize}
This expression is symmetric in $u$ and $v$, so that $\langle \mathcal{O}_1 u, v \rangle_\mu = \langle u, \mathcal{O}_1 v \rangle_\mu$. Thus, $\mathcal{O}_1^* = \mathcal{O}_1$.

%\vspace{-0.1in}
\subsection{Adjoint of \texorpdfstring{$\mathcal{O}_2$}{O2} under Assumption~\ref{assump:nablaJ}}
The operator $\mathcal{O}_2 u = - \langle J\nabla f, \nabla u \rangle$. We seek $\mathcal{O}_2^*$ such that $\langle \mathcal{O}_2 u, v \rangle_\mu = \langle u, \mathcal{O}_2^* v \rangle_\mu$.
%\begin{footnotesize}
\begin{align*}
    \langle \mathcal{O}_2 u, v \rangle_\mu &= -\frac{1}{Z} \int_K \langle J(x)\nabla f(x), \nabla u(x) \rangle v(x) e^{-f(x)} dx \\
    &= -\frac{1}{Z} \int_K \left(v(x) J(x)\nabla f(x) e^{-f(x)}\right) \cdot \nabla u(x) dx.
\end{align*}
%\end{footnotesize}
Using integration by parts (divergence theorem, $\int_K \mathbf{F} \cdot \nabla u \,dx = -\int_K u (\nabla \cdot \mathbf{F}) \,dx + \int_{\partial K} u (\mathbf{F} \cdot \mathbf n_{out}) \,dS$, with $\mathbf{F}(x) = v(x) J(x)\nabla f(x) e^{-f(x)}$):
%\begin{footnotesize}
\begin{align*}
    \langle \mathcal{O}_2 u, v \rangle_\mu
    & = -\frac{1}{Z} \Bigg[ -\int_K u(x) \nabla \cdot \left(v(x) J(x)\nabla f(x) e^{-f(x)}\right) dx
    \\
    &\qquad\qquad+ \int_{\partial K} u(x) v(x) \left(J(x)\nabla f(x) e^{-f(x)}\right) \cdot \mathbf n_{out}(x) dS \Bigg].
\end{align*}
%\end{footnotesize}
The boundary term is $\int_{\partial K} u v e^{-f} (J\nabla f \cdot \mathbf n_{out}) dS$.
Since $\mathbf n_{out} = -\mathbf n_{in}$ and by Assumption~\ref{assump:boundary} ($J\mathbf n_{in}=0$), we have
%\begin{footnotesize}
\begin{align*}
J\nabla f \cdot \mathbf n_{out} &= - \left(J\nabla f \cdot \mathbf n_{in}\right) = - \left(\nabla f \cdot J^\top \mathbf n_{in}\right) = - \left(\nabla f \cdot (-J \mathbf n_{in})\right) 
\\
&= (\nabla f \cdot J \mathbf n_{in}) = (\nabla f \cdot \mathbf{0}) = 0.
\end{align*}
%\end{footnotesize}
Thus, the boundary term vanishes.
Therefore,
\begin{align*}
\langle \mathcal{O}_2 u, v \rangle_\mu = \frac{1}{Z} \int_K u(x) \nabla \cdot \left(v(x) J(x)\nabla f(x) e^{-f(x)}\right) dx.
\end{align*}
The adjoint operator $\mathcal{O}_2^*$ acting on a function $v$ is thus:
%\begin{footnotesize}
\begin{align*}
    \mathcal{O}_2^* v(x) &= e^{f(x)} \nabla \cdot \left(v(x) J(x)\nabla f(x) e^{-f(x)}\right) \\
    &= e^{f(x)} \left[ (\nabla v(x)) \cdot \left(J(x)\nabla f(x)e^{-f(x)}\right) + v(x) \nabla \cdot \left(J(x)\nabla f(x)e^{-f(x)}\right) \right] \\
    &= \langle J(x)\nabla f(x), \nabla v(x) \rangle + v(x) e^{f(x)} \left( \nabla \cdot \left(J(x)\nabla f(x)e^{-f(x)}\right) \right).
\end{align*}
%\end{footnotesize}
By chain rule and product rule, we can compute that
%\begin{footnotesize}
\begin{align*}
&\nabla \cdot \left(J(x)\nabla f(x)e^{-f(x)}\right) 
\\
&=-\nabla \cdot \left(J(x)\nabla\left(e^{-f(x)}\right)\right) 
=-\sum\nolimits_i \frac{\partial}{\partial_{x_{i}}} \left( \sum\nolimits_k J_{ik}(x) \left(\frac{\partial}{\partial_{x_{k}}} e^{-f(x)}\right) \right)
\\
&=-\sum\nolimits_i  \left(\frac{\partial}{\partial_{x_{i}}}\sum\nolimits_k J_{ik}(x)\right) \left(\frac{\partial}{\partial_{x_{k}}} e^{-f(x)}\right)
-\sum\nolimits_{i}\sum\nolimits_{k}J_{ik}(x)\frac{\partial^{2}e^{-f(x)}}{\partial x_{i}\partial x_{k}}.
\end{align*}
%\end{footnotesize}

Under Assumption~\ref{assump:nablaJ}, we have $\nabla \cdot J(x) = 0$. Consequently, 
\[
\sum_i  \left(\frac{\partial}{\partial_{x_{i}}}\sum_k J_{ik}(x)\right) \left(\frac{\partial}{\partial_{x_{k}}} e^{-f(x)}\right)=0.
\]
Moreover, $J(x)$ is skew-symmetric, and $\left(\frac{\partial^{2}e^{-f(x)}}{\partial x_{i}\partial x_{k}}\right)_{i,k}$ is a symmetric matrix, 
which implies that
$\sum_{i}\sum_{k}J_{ik}(x)\frac{\partial^{2}e^{-f(x)}}{\partial x_{i}\partial x_{k}}=0$.    
Hence, we conclude that 
$\nabla \cdot \left(J(x)\nabla f(x)e^{-f(x)}\right)=0$,
and the expression for $\mathcal{O}_2^* v(x)$ simplifies to:
\begin{align*}
\mathcal{O}_2^* v(x) = \langle J(x)\nabla f(x), \nabla v(x) \rangle.
\end{align*}

%%%%%%%%%%%%%%%%%%%%%%%%%%%%%%%%%%%%%%%%%%%%%%%%%%%%%%%%%%%%%%%%%%%%%
%\vspace{-0.1in}
\subsection{Decomposition of \texorpdfstring{$\mathcal{L}_J$}{LJ} under Assumption~\ref{assump:nablaJ}}
Now we can compute $\mathcal{L}_S $ and $\mathcal{L}_A $ using the simplified $\mathcal{O}_2^*$.

\textbf{The Symmetric Part $\mathcal{L}_S $:}
%\begin{footnotesize}
\begin{align*}
    \mathcal{L}_S u &= \frac{1}{2}(\mathcal{L}_J u + \mathcal{L}_J^* u) = \frac{1}{2}( (\mathcal{O}_1 u + \mathcal{O}_2 u) + (\mathcal{O}_1^* u + \mathcal{O}_2^* u) ) \\
    &= \frac{1}{2}( \mathcal{O}_1 u + \mathcal{O}_2 u + \mathcal{O}_1 u + \mathcal{O}_2^* u )= \mathcal{O}_1 u + \frac{1}{2}(\mathcal{O}_2 u + \mathcal{O}_2^* u). \quad (\text{since } \mathcal{O}_1^* = \mathcal{O}_1)
\end{align*}
%\end{footnotesize}
Substituting the expressions for $\mathcal{O}_1 u$, $\mathcal{O}_2 u = -\langle J\nabla f, \nabla u \rangle$, and the simplified $\mathcal{O}_2^* u = \langle J\nabla f, \nabla u \rangle$:
\begin{align*}
\mathcal{O}_2 u + \mathcal{O}_2^* u = -\langle J(x)\nabla f(x), \nabla u(x) \rangle + \langle J(x)\nabla f(x), \nabla u(x) \rangle = 0.
\end{align*}
Thus, the symmetric part is:
\begin{align*}
\mathcal{L}_S u(x) = \mathcal{O}_1 u(x) = \Delta u(x) - \langle \nabla f(x), \nabla u(x) \rangle.
\end{align*}

\textbf{The Skew-Symmetric Part $\mathcal{L}_A $:}
%\begin{footnotesize}
\begin{align*}
    \mathcal{L}_A u = \frac{1}{2}(\mathcal{L}_J u - \mathcal{L}_J^* u)= \frac{1}{2}( (\mathcal{O}_1 u + \mathcal{O}_2 u) - (\mathcal{O}_1 u + \mathcal{O}_2^* u) ) 
    = \frac{1}{2}(\mathcal{O}_2 u - \mathcal{O}_2^* u).
\end{align*}
%\end{footnotesize}
Substituting $\mathcal{O}_2 u = -\langle J\nabla f, \nabla u \rangle$ and the simplified $\mathcal{O}_2^* u = \langle J\nabla f, \nabla u \rangle$:
%\begin{footnotesize}
\begin{align*}
    \mathcal{O}_2 u - \mathcal{O}_2^* u = -\langle J(x)\nabla f(x), \nabla u(x) \rangle - \langle J(x)\nabla f(x), \nabla u(x) \rangle
    = -2\langle J(x)\nabla f(x), \nabla u(x) \rangle.
\end{align*}
%\end{footnotesize}
Thus, the skew-symmetric part is:
$\mathcal{L}_A u(x) = -\langle J(x)\nabla f(x), \nabla u(x) \rangle$.
Under Assumption~\ref{assump:nablaJ}, $\mathcal{L}_A$ is anti-self-adjoint with respect to $\mu$.


\vskip 0.2in
\bibliographystyle{plainnat}
\bibliography{langevin}

\end{document}